\documentclass[oneside]{article}

\usepackage[utf8]{inputenc}
\usepackage[T1]{fontenc}
\usepackage{hyperref}
\usepackage{url}
\usepackage{nicefrac}
\usepackage{microtype}
\usepackage{xcolor}
\usepackage{amsfonts, amsmath, mathtools, amssymb, amsthm}
\usepackage{algorithm}
\usepackage[noend]{algorithmic}
\usepackage{graphicx}
\usepackage{booktabs}
\usepackage{multirow}
\usepackage[font=small]{caption}
\usepackage{subcaption}
\usepackage[numbers]{natbib}
\usepackage[
    letterpaper,
    left=0.95in,
    right=0.95in,
    top=1in,
    bottom=1in
]{geometry}
\usepackage{authblk}

\theoremstyle{plain}
\newtheorem{theorem}{Theorem}
\newtheorem{proposition}{Proposition}[section]
\newtheorem{lemma}[proposition]{Lemma}
\newtheorem{corollary}[proposition]{Corollary}
\theoremstyle{definition}
\newtheorem{definition}[proposition]{Definition}
\newtheorem{assumption}[proposition]{Assumption}
\theoremstyle{remark}
\newtheorem{remark}[proposition]{Remark}

\input{macros.sty}

\date{}
\title{\DPZerons: Private Fine-Tuning of Language Models \\ without Backpropagation}

\author[1]{Liang Zhang}
\author[1]{Bingcong Li}
\author[2]{Kiran Koshy Thekumparampil}
\author[3]{Sewoong Oh}
\author[1]{Niao He}
\affil[1]{Department of Computer Science, ETH Zurich}
\affil[2]{Amazon Search}
\affil[3]{Paul G. Allen School of Computer Science and Engineering, University of Washington}
\affil[ ]{\texttt{\{liang.zhang, bingcong.li, niao.he\}@inf.ethz.ch, \protect\\ kkt@amazon.com, sewoong@cs.washington.edu}}

\begin{document}

\maketitle

\begin{abstract}
    The widespread practice of fine-tuning large language models (LLMs) on domain-specific data faces two major challenges in memory and privacy.
    First, as the size of LLMs continues to grow, the memory demands of gradient-based training methods via backpropagation become prohibitively high. Second, given the tendency of LLMs to memorize training data, it is important to protect potentially sensitive information in the fine-tuning data from being regurgitated.
    Zeroth-order methods, which rely solely on forward passes, substantially reduce memory consumption during training. However, directly combining them with standard differentially private gradient descent suffers more as model size grows.
    To bridge this gap, we introduce \DPZerons, a novel private zeroth-order algorithm with nearly dimension-independent rates.
    The memory efficiency of \DPZero is demonstrated in privately fine-tuning RoBERTa and OPT on several downstream tasks. Our code is available at \url{https://github.com/Liang137/DPZero}.
\end{abstract}

\section{Introduction}
\label{sec:intro}

Fine-tuning pretrained large language models (LLMs), such as BERT \citep{devlin2019bert, liu2019roberta, sanh2019distilbert}, OPT \citep{zhang2022opt}, LLaMA \cite{touvron2023llama,touvron2023llama2}, and GPT \citep{radford2018improving, brown2020language, ouyang2022training, openai2023gpt4}, achieves state-of-the-art performance in a wide array of downstream applications.
However, two significant challenges persist in practical adoption: memory demands for gradient-based optimizers and the need to safeguard the privacy of domain-specific fine-tuning data.

As the memory requirement of fine-tuning LLMs is increasingly becoming a bottleneck, various approaches have been proposed, spanning from parameter-efficient fine-tuning (PEFT) \citep{li2021prefix, hu2022lora} to  novel optimization algorithms \citep{shazeer2018adafactor, anil2019memory}.
Since these methods rely on  backpropagation to compute the gradients, which can be memory-intensive, a recent trend has emerged in developing algorithms that do not require backpropagation \citep{baydin2022gradients, silver2022learning, hinton2022forward, hou2023promptboosting, phang2023hypertuning, chen2023deepzero}. 
Specifically for LLMs, \citet{malladi2023fine} introduced zeroth-order methods for fine-tuning, thereby eliminating the backward pass and freeing up the memory for gradients and activations.
Utilizing a single A100 GPU (80 GiB memory), zeroth-order methods are capable of fine-tuning a 30-billion-parameter model, whereas first-order methods, even equipped with PEFT, fail to fit into the memory for a model with more than  6.7 billion parameters. This greatly expands the potential for deploying and fine-tuning LLMs even on personal devices.

On the other hand, empirical studies have highlighted the risk of LLMs inadvertently revealing sensitive information from their fine-tuning datasets \cite{mireshghallah2022memorization, zeng2023exploring, mattern2023membership, lukas2023analyzing}.
Such privacy concerns are pronounced especially when users opt to fine-tune LLMs on datasets of their own. Notably, the expectation that machine learning models should not compromise the confidentiality of their contributing entities is codified into legal frameworks \citep{voigt2017eu}.
Differential privacy (DP) \citep{dwork2006calibrating} is a widely accepted mathematical framework for ensuring privacy by preventing attackers from identifying participating entities  \citep{shokri2017membership}. Consequently, the development of methods that fine-tune LLMs under differential privacy is of pressing necessity \citep{li2022large, yu2022differentially, he2023exploring, bu2023differentially, du2023dp}; however, most efforts so far have focused on first-order algorithms.

Motivated by the memory-hungry nature and privacy concerns in fine-tuning LLMs, we investigate  zeroth-order methods that guarantee differential privacy for solving the following stochastic optimization problem:
\begin{equation}
    \min_{x\in\bR^d} F_S(x)\; := \;\frac{1}{n}\sum_{i=1}^n f(x;\xi_i)\;, \label{eq:opt}
\end{equation}
where $S=\{\xi_i\}_{i=1}^n$ is the training data, $x\in\bR^d$ is the model weight, the loss $f(x;\xi_i)$ is Lipschitz for each sample $\xi_i$, and the averaged loss $F_S(x)$ is smooth and possibly nonconvex.
In theory, previous work on both differentially private optimization \citep{bassily2014private} and zeroth-order optimization \citep{duchi2015optimal} indicated that their convergence guarantees depend explicitly on the dimension $d$. Such dimension dependence becomes problematic in the context of LLMs with $d$ scaling to billions. 
In practice, and somewhat surprisingly, empirical studies on the fine-tuning of LLMs using zeroth-order methods \citep{malladi2023fine} and DP first-order methods \citep{yu2022differentially, li2022large, li2022does} have shown that the performance degradation due to the large model size is marginal. For example, \citet{yu2022differentially} showed that the performance drop due to privacy is smaller for larger architectures. A 345 million-sized GPT-2-Medium, fine-tuned with ($\varepsilon=6.8,\delta=10^{-5}$)-DP, showcases a modest drop of 5.1 in BLEU score \citep{papineni2002bleu} (compared to a non-private model of the same size and architecture), whereas a larger GPT-2-XL with 1.5 billion parameters exhibits smaller cost in test performance, i.e., 4.3 BLEU score under the same privacy budget.

This gap between theory and practice has been linked  to the presence of low-rank structures in the fine-tuning of pretrained LLMs \citep{malladi2023fine, li2022does}. 
Empirical evidence suggests that fine-tuning occurs within a low-dimensional subspace \citep{sagun2017empirical, gur2018gradient, ghorbani2019investigation, li2018measuring}: 200 dimensions for RoBERTa with 355 million parameters \citep{aghajanyan2021intrinsic} and 100 dimensions for PEFT on  DistilRoBERTa with 7 million parameters \citep{li2022does}.
In such cases where the intrinsic dimension is small, zeroth-order methods are known to achieve dimension-independent convergence rate \cite{malladi2023fine} and private first-order methods are also known to achieve dimension-independent guarantees  \citep{ma2022dimension,li2022does}.

\vskip 0.1in
Given the significance of fine-tuning LLMs on domain-specific datasets, we ask the following fundamental question: {\em Can we achieve a dimension-independent rate both under differential privacy and with access only to the zeroth-order oracle?} Our contributions are summarized below.
\vskip 0.1in

$\bullet$ We first show that the straightforward approach --- that combines DP first-order methods with  zeroth-order gradient estimators (Algorithm \ref{algo:d-dependent}) --- exhibits an undesirable dimension dependence in the convergence guarantees, even when the effective rank of the problem does not scale with the dimension (Theorems \ref{thm:d-dependent} and \ref{thm:d-dependent-rank} in Section \ref{sec:d-dependent}). There are two root causes. First, the standard practice of choosing the clipping threshold to be the maximum norm of the estimated sample gradient leads to an unnecessarily large threshold.
Next, this choice of the clipping threshold forces the addition of a large noise to ensure privacy, and Algorithm \ref{algo:d-dependent} adds that noise in all $d$ directions.

$\bullet$ We present \DPZero (Algorithm \ref{algo:d-free}), the first nearly dimension-independent DP zeroth-order method for stochastic optimization. Its convergence guarantee depends on the effective rank of the problem (specified in Assumption \ref{asp:rank}) and exhibits logarithmic dependence on the dimension $d$ (Theorem \ref{thm:d-free} in Section \ref{sec:d-free}). This builds upon two insights. 
First, the direction of the estimated gradient is a public information and does not need to be private; it is sufficient to make only the magnitude of the estimated gradient private, which is a scalar value. Next, we introduce a tighter analysis that allows us to choose a significantly smaller clipping threshold, leveraging the fact that the typical norm of the estimated gradient is much smaller than its maximum.

$\bullet$ We verify the effectiveness of \DPZero in both synthetic examples and private fine-tuning tasks on RoBERTa \citep{liu2019roberta} and OPT \citep{zhang2022opt}. In contrast to first-order algorithms that demand extensive effort for the efficient implementation of per-sample gradient clipping \citep{li2022large, he2023exploring, bu2023differentially}, \DPZero offers the advantage of near-zero additional costs compared to non-private zeroth-order methods \citep{malladi2023fine}. Our empirical results validate theoretical findings, revealing only a slight performance decrement for \DPZero even with large model sizes.

\subsection{Related Works}

We build upon exciting advances in zeroth-order optimization and differentially private optimization, which we survey here. Notably, \DPZero is inspired by new empirical and theoretical findings showing that fine-tuning LLMs does not suffer in high-dimensions when using zeroth-order methods in \citet{malladi2023fine} or using private first-order optimization in \citet{li2022does}.
A more comprehensive overview is deferred to Appendix \ref{app:review}. 

\paragraph{Zeroth-order optimization.}
\citet{nesterov2017random} pioneered the formal analysis of the convergence rate of zeroth-order methods, i.e., zeroth-order (stochastic) gradient descent (ZO-SGD) that replaces gradients in SGD by their zeroth-order estimators.
Their findings are later refined by several works \citep{ghadimi2013stochastic, shamir2017optimal, lin2022gradient}.
These well-established results indicate a runtime complexity $\mathcal{O}(d)$ worse than first-order methods.
Such dimension dependence of zeroth-order methods is proven inevitable without additional structures \citep{wibisono2012finite, duchi2015optimal}.

There are several recent works that relax the dimension dependence in zeroth-order methods leveraging problem structures. \citet{balasubramanian2018zeroth} demonstrated that ZO-SGD can directly identify the sparsity of the problem and proved a dimension-independent rate when the support of gradients remains unchanged. \citet{yue2023zeroth} and \citet{malladi2023fine} relaxed the dependence on dimension $d$ to a quantity related to the trace of the loss's Hessian.

\paragraph{Differentially private optimization.}
Previous works on DP optimization mostly center around first-order methods.
When the problem is nonconvex, i.e., the setting of our interest, differentially private (stochastic) gradient descent (DP-GD) achieves a rate of $\cO(\sqrt{d\log(1/\delta)}/(n\eps))$ on the squared norm of the gradient \citep{wang2017differentially, zhou2020private}.
We show that \DPZero matches this rate with access only to the zeroth-order oracle in Theorem~\ref{thm:d-free}.
Given access to the first-order oracle, it has been recently shown that such rate can be improved to $\cO((\sqrt{d\log(1/\delta)}/(n\eps))^{4/3})$ leveraging momentum \citep{tran2022momentum} or variance reduction techniques \citep{arora2023faster}.

Early works established dimension-independent rates when the gradients lie in some fixed low-rank subspace \citep{jain2014near, song2021evading}.
Closest to our result is \citet{song2021evading}, which demonstrated that the rate of DP-GD for smooth nonconvex optimization can be improved  to $\cO(\sqrt{r\log(1/\delta)}/(n\eps))$ for generalized linear models (GLMs) with a rank-$r$ feature matrix. \DPZero matches this result with access only to the zeroth-order oracle in Theorem~\ref{thm:d-free} for more general problems beyond low-rank GLMs.
Our result is inspired by \citet{li2022does} that introduced a relaxed Lipschitz condition for the gradients and provided dimension-free bounds when the loss is convex and the relaxed Lipschitz parameters decay rapidly. Similarly, \citet{ma2022dimension} suggested that the  dependence  on $d$ in the utility upper bound for DP stochastic convex optimization can be improved.

Literature on DP optimization  beyond first-order methods remains less explored.
Recently, \citet{zhang2024private} studied the problem of private zeroth-order nonsmooth nonconvex optimization and achieved a rate that depends on the dimension $d$.
As far as we are aware, no prior studies have addressed the challenge of deriving a dimension-independent rate in DP zeroth-order optimization.

After the workshop version of our paper \citep{zhang2023dpzero} was released, \citet{tang2024private} concurrently discovered the same algorithm as \DPZero (up to a minor difference in how $u_t$ is drawn) and showed empirical benefits when applied to fine-tuning OPT models but without theoretical analysis.
Also building upon the workshop version of our paper, \citet{liu2024differentially} introduced DP-ZOSO, a stage-wise zeroth-order method with an additional quadratic regularizer.
With extra hyper-parameters to be tuned, DP-ZOSO demonstrates further empirical gain over \DPZerons. However, \citet{liu2024differentially} only provided {\em dimension-dependent} guarantees.

\section{Preliminaries}
\label{sec:prelim} 

\paragraph{Notation.}
We use $\norm{\cdot}$ for the Euclidean norm and define $\|v\|_W^2=v^\top Wv$ for a square matrix $W$.
$\bS^{d-1}=\{x\in\bR^d \,|\, \norm{x}=1\}$ denotes the unit sphere in $\bR^d$, and $\eta\,\bS^{d-1}$ is the sphere of radius $\eta>0$.
A function $p: \bR^d\rightarrow\bR$ is $L$-Lipschitz if $\abs{p(x_1)-p(x_2)}\leq L\norm{x_1-x_2}, \forall x_1, x_2$.
A function $q: \bR^d\rightarrow\bR$ is $\ell$-smooth if it is differentiable and $\norm{\nabla q(x_1)-\nabla q(x_2)}\leq \ell\norm{x_1-x_2}$.
The trace of a square matrix $J$ is denoted by $\tr(J)$.
A symmetric real matrix $M\succeq 0$ if it is positive semi-definite.
The clipping operation is defined to be $\clip_C(x)=x\,\min\{1, C/\norm{x}\}$ given $C>0$.
The notation $\tilde\cO(\cdot)$ hides additional logarithmic terms.

\subsection{Differential Privacy}
\label{sec:dp} 

\begin{definition}[Differential Privacy \cite{dwork2006calibrating, dwork2014algorithmic}]    
    Two datasets $S =\{\xi_i\}_{i=1}^n$ and $S'=\{\xi_i'\} _{i=1}^{n}$ are \emph{neighboring} if $\max\{| S \setminus S'|,| S' \setminus S|\} = 1$, and we denote it by $S \sim S'$. 
    For prescribed $\eps>0$ and $\delta\in(0,1)$, an algorithm $\cA$ is said to satisfy
    $(\eps,\delta)$-\emph{differential privacy} (DP) if
    $\bP(\cA(S)\in \cB) \leq e^\eps \bP(\cA(S') \in \cB) + \delta$
    for all $S\sim S'$ and all measurable set $\cB$ in the range of $\cA$.
    \label{def:dp}
\end{definition}
To ensure DP while solving the optimization problem in Eq.~\eqref{eq:opt}, first-order approaches, such as DP-GD, update via $x_{t+1}\gets x_t - \alpha ((1/n)\sum_{i=1}^n {\rm clip}_C(\nabla f(x_t;\xi_i))+z_t)$; see e.g., \citep{song2013stochastic, abadi2016deep}.
Through the following composition lemma \citep[Theorem 4.3]{kairouz2015composition}, the privacy for entire $T$ updates is secured by the per-sample clipping operation that ensures finite sensitivity of $\Delta=2C/n$ together with the Gaussian noise $z_t$. 

\begin{lemma}[Advanced Composition]
    Let $\cA$ be some randomized algorithm operating on a dataset $S$ and outputting a vector in $\bR^d$. If $\cA$ has sensitivity $\Delta:=\sup_{S \sim S'} \norm{\cA(S) - \cA(S')}$, the mechanism that adds Gaussian noise $\cN(0,\sigma^2\rI_d)$ with variance $\sigma^2=(2\Delta\sqrt{2T\log(e+(\eps/\delta))}/\eps )^2$ satisfies $(\eps, \delta)$-DP under $T$-fold adaptive composition for any $\eps>0$ and $\delta\in(0,1)$. 
    \label{lm:composition}
\end{lemma}

\subsection{Zeroth-Order Optimization}

When the gradient is expensive to compute, zeroth-order methods are useful for optimizing Eq.~\eqref{eq:opt}.
For example, the two-point gradient estimator below requires only two evaluation of function values \citep{shamir2017optimal}
\begin{equation}\label{eq.zo}
    g_\lambda(x;\xi_i) \;\;:=\;\; \frac{f(x+\lambda u;\xi_i) - f(x-\lambda u;\xi_i)}{2\lambda} u\;,
\end{equation}
where $u$ is sampled uniformly from the Euclidean sphere $\sqrt{d}\,\bS^{d-1}$ and $\lambda>0$ is the smoothing parameter \citep{yousefian2012stochastic, duchi2012randomized}. A common approach to generate $u$ is to set $u=\sqrt{d}\, z/\norm{z}$, with $z$ sampled from the standard multivariate Gaussian $\cN(0, \rI_d)$ \citep{muller1959note, marsaglia1972choosing}. We refer to $g_\lambda(x;\xi)$ as the {\em zeroth-order gradient} (estimator) in the sequel.
The results in this paper can be directly extended to other zeroth-order gradient estimators, e.g., any $u$ satisfying $\bE[uu^\top]=\rI_d$ \citep{duchi2015optimal}, the one-point estimator \citep{flaxman2005online}, and the directional derivative \citep{nesterov2017random}.

\section{DP-GD with Zeroth-Order Gradients Suffers in High Dimensions}
\label{sec:d-dependent}
 
In this section, we show that the direct integration of zeroth-order gradient estimators in Eq. \eqref{eq.zo} into DP-GD, which we term DPGD-0th, leads to undesirable dimension dependence in the error rate. Such dependence persists even under a low effective rank assumption.

\subsection{Direct Integration Leads to an ${\cal O}(d^{3/2})$ Rate}

We present in Algorithm \ref{algo:d-dependent} the straightforward private zeroth-order approach that substitutes the gradients in DP-GD with zeroth-order estimators $g_\lambda(x_t;\xi_i)$ in Eq.~\eqref{eq.zo}.

The privacy guarantee follows from standard DP-GD analysis, and the utility guarantee on the squared gradient norm is derived from classical techniques for analyzing zeroth-order methods \citep{nesterov2017random}. Before presenting the convergence result, we make the following standard assumption, which is common in nonconvex DP optimization  \citep{wang2017differentially, wang2019differentially, tran2022momentum}. 

\begin{assumption}
    The loss $f(x;\xi)$ is $L$-Lipschitz for every $\xi$. The average loss $F_S(x)$ is $\ell$-smooth for every given dataset $S$, and its minimum $F_S^*:=\min_{x\in\bR^d} F_S(x)$ is finite.
    \label{asp:smooth}
\end{assumption}

\begin{theorem} 
    For any $\eps>0$ and $\delta\in(0,1)$, Algorithm \ref{algo:d-dependent} is $(\eps, \delta)$-DP.  
    Under Assumption \ref{asp:smooth}, its output $x_\tau$ satisfies that
    \begin{equation}
        \bE[\norm{\nabla F_S(x_\tau)}^2 ] \;\; \leq \;\; 16\Big((F_S(x_0) - F_S^*) \,\ell  +2L^2\Big)\frac{d\sqrt{d\log(e+(\eps/\delta))}}{n\eps},
        \label{eq:d-dependent}
    \end{equation}
    with the choice of parameters
    \begin{equation*}
        \alpha=\frac{1}{4\ell d}, \quad
        T=\frac{n\eps}{\sqrt{d\log(e + (\eps/\delta))}}, \quad
        \lambda\leq\frac{4L}{\ell d}\Big(\frac{\sqrt{d\log(e + (\eps/\delta))}}{n\eps}\Big)^{1/2}, \quad
        C = Ld.
    \end{equation*}
    The total number of zeroth-order gradient computations is $nT=\cO(n^2/\sqrt{d})$.
    \label{thm:d-dependent}
\end{theorem}

\begin{remark}
    Theorem \ref{thm:d-dependent} demonstrates that directly combining DP-GD with zeroth-order gradients leads to an ${\cal O}(d^{3/2})$ error complexity, which is $\cO(d)$ worse than first-order DP approaches \citep{wang2017differentially}. 
\end{remark}

\begin{remark}
    \label{remark.naive}
    Three sources contribute to the dependence in $d$: the squared norm of the zeroth-order gradient estimator $\bE[\|(1/n) \sum_{i=1}^n g_\lambda(x,\xi_i)\|^2]=\cO(d\,\norm{\nabla F_S(x)}^2)$ when taking $\lambda\rightarrow0$ for simplicity, the clipping threshold $C=\cO(d)$, and the norm of the privacy noise $\bE[\|z_t\|^2]=\cO(d\,C^2) = \cO(d^3)$.
    The standard analysis of one-step update gives
    \begin{equation}
        \bE[F_S(x_{t+1})]
        \;\leq\;
        \bE[F_S(x_t)] - \frac{\alpha}{2}\big(1 - 2d\,\ell\alpha\big)\,\bE[\norm{\nabla F_S(x_t)}^2] + c\,\alpha^2\, d^3,
        \label{eq:onestep1simple} 
    \end{equation}
    where $c$ is a constant that depends on problem parameters other than $\alpha$ and $d$; see Eq.~\eqref{eq:onestep1} for details. A small enough step size, $\alpha < 1/(2\ell d)$, is required to make the second term negative, where the dependence in $d$ comes from $\bE[\|(1/n) \sum_{i=1}^n g_\lambda(x,\xi_i)\|^2]$. The dependence on $d^3$ in the last term arises from $\bE[\|z_t\|^2]$, which leads to the $\cO(d^{3/2})$ rate in Eq.~\eqref{eq:d-dependent} after balancing error terms.  Detailed proofs can be found in Appendix~\ref{app:d-dependent}.
\end{remark}

\begin{remark}
    The choice of the clipping threshold $C=Ld$ ensures that clipping does not happen with probability one, which is a common choice in the theoretical analysis of private optimization algorithms \citep{bassily2014private, bassily2019private, wang2017differentially}. This follows from the fact that, for $L$-Lipschitz $f(x;\xi)$, the zeroth-order gradient is upper bounded by $\norm{g_\lambda(x;\xi)}\leq Ld$ almost surely.
    Selecting the clipping threshold without knowledge of this upper bound remains an active research topic \citep{chen2020understanding, yang2022normalized, fang2023improved, koloskova2023revisiting, zhang2023differentially}.
\end{remark}

\begin{algorithm}[t]
    \caption{DP-GD with 0th-order gradients (DPGD-0th)}
    \begin{algorithmic}[1]
        \REQUIRE Dataset $S=\{\xi_1, \cdots, \xi_n\}$, initialization $x_0\in\bR^d$, number of iterations $T$, stepsize $\alpha>0$, smoothing parameter $\lambda>0$, clipping threshold $C>0$, privacy parameters $\eps>0, \delta\in(0,1)$.
        
        \FOR{$t=0, 1, \cdots, T-1$}
            \STATE Sample $u_t$ uniformly at random from the Euclidean sphere $\sqrt{d}\,\bS^{d-1}$ and for all $i=1,\cdots,n$ compute
            \begin{equation*}
                g_\lambda(x_t;\xi_i) \gets \frac{f(x_t+\lambda u_t;\xi_i) - f(x_t-\lambda u_t;\xi_i)}{2\lambda} u_t.
            \end{equation*}
            
            \STATE Sample $z_t\in\bR^d$ randomly from the multivariate Gaussian distribution $\cN(0, \sigma^2\rI_d)$ with variance $\sigma=4C\sqrt{2T\log(e + (\eps/\delta))}/(n\eps)$ and update
            \begin{equation*}
                x_{t+1} \;\gets\; x_t - \alpha\Big(
                    \frac{1}{n}\sum_{i=1}^n \clip_C(g_\lambda(x_t;\xi_i)) + z_t
                \Big).
            \end{equation*}
        \ENDFOR
        
        \ENSURE $x_\tau$ for $\tau$ sampled uniformly at random from $\{0,1,\cdots,T-1\}$.
    \end{algorithmic}
    \label{algo:d-dependent}
\end{algorithm}

\subsection{Rate Improves to $\cO(d)$ under Low Effective Rank} 
\label{sec:lowrank} 

Here, under the low-dimensional structures in fine-tuning LLMs (cf. Section~\ref{sec:intro}), we demonstrate improved performance for Algorithm~\ref{algo:d-dependent}. Unfortunately, a linear dependence in $d$ still persists even under the low effective rank structure.

\begin{assumption}
    The function $f(x;\xi)$ is $L$-Lipschitz and $\ell$-smooth for every $\xi$. The average function $F_S(x)$ is twice differentiable with $-H \preceq \nabla^2 F_S(x)\preceq H$ for any $x\in\bR^d$, and its minimum $F_S^*:=\min_{x\in\bR^d} F_S(x)$ is finite. Here, the real-valued $d\times d$ matrix $H\succeq 0$ satisfies that $\norm{H}_2\leq \ell$ and $\tr(H)\leq r\norm{H}_2$. We refer to $r$ as the effective rank or the intrinsic dimension of the problem.
    \label{asp:rank}
\end{assumption}
 
Assumption~\ref{asp:rank} boils down to Assumption~\ref{asp:smooth} if $r=d$. This is because $-H'\preceq \nabla^2 F_S(x)\preceq H', \forall\,x\in\bR^d$ and $H'=\ell\,\rI_d$ imply that $\norm{H'}_2\leq \ell$ and $\tr(H')\leq d\norm{H'}_2$. With $r<d$, this assumption reflects the additional structures encoded in the Hessian matrix. 
While Assumption \ref{asp:rank} naturally holds for low-rank Hessians, it covers more general cases.
For example, the assumption is satisfied with $r=\cO(\log d)\ll d$ in the case of a full-rank matrix $H$, with its $i$-th largest eigenvalue being $\ell/i$ for $1\leq i\leq d$. 

Similar assumptions have been made to relax the dimension dependence in zeroth-order optimization in the limit $\lambda\to 0$ \citep{malladi2023fine} and also for DP first-order optimization when the objective is smooth and convex \citep{ma2022dimension}.
However, even under Assumption \ref{asp:rank}, DPGD-0th (Algorithm \ref{algo:d-dependent}) still suffers from a linear dependence in $d$ in its error rate, as presented below. A proof is provided in Appendix \ref{app:d-dependent}.

\begin{theorem}
    For any $\eps>0$ and $\delta\in(0,1)$, Algorithm \ref{algo:d-dependent} is $(\eps, \delta)$-DP. Under Assumption \ref{asp:rank}, its output $x_\tau$ satisfies that
    \begin{equation}
        \bE[\norm{\nabla F_S(x_\tau)}^2 ] \;\; \leq \;\;
        16\Big((F_S(x_0) - F_S^*) \,\ell  +2L^2\Big)\frac{d\sqrt{r\log(e+(\eps/\delta))}}{n\eps}, 
        \label{eq:d-dependent-rank}
    \end{equation}
    with the choice of parameters
    \begin{equation*}
        \alpha=\frac{1}{4\ell (r+2)}, \quad
        T=\frac{n(r+2)\eps}{d\sqrt{r\log(e + (\eps/\delta))}}, \quad
        \lambda\leq\frac{4L}{\ell d}\Big(\frac{\sqrt{r\log(e + (\eps/\delta))}}{n\eps}\Big)^{1/2}, \quad
        C = Ld.
    \end{equation*}
    The total number of zeroth-order gradient computations is $nT=\cO(n^2\sqrt{r}/d)$.
    \label{thm:d-dependent-rank}
\end{theorem}

\begin{remark}
    \label{rem:onestep2}
    Comparing to Remark \ref{remark.naive}, both the zeroth-order gradient, $\bE[\norm{(1/n)\sum_{i=1}^n g_\lambda(x_t;\xi_i)}_{H}^2]$, and the DP noise, $\bE[\norm{z_t}_H^2]$, decrease by a factor of ${\cal O}(r/d)$ under low effective rank. This is made precise in  Lemma \ref{lm:sphere}. 
    As a result, the one-step update analysis can be tightened as 
    \begin{equation}
        \bE[F_S(x_{t+1})]
         \;\leq\;
        \bE[F_S(x_t)] - \frac{\alpha}{2}\big(1 - 2(r+2)\ell\alpha \big)\,\bE[\norm{\nabla F_S(x_t)}^2] + c\,\alpha^2\, r\,d^2.
        \label{eq:onestep2simple}
    \end{equation}
    Comparing to the RHS of Eq. \eqref{eq:onestep1simple}, it achieves an improved dependence in $d$. 
    However, the third term in Eq. \eqref{eq:onestep2simple} is still at $\cO(d^2)$ due to the clipping threshold $C=\cO(d)$. 
    Consequently, even when the effective rank $r$ is small, Eq.~\eqref{eq:d-dependent-rank} still grows linearly in $d$. 
\end{remark}

\section{DPZero: Nearly Dimension-Independent Private Zeroth-Order Optimization}
\label{sec:d-free}

A straightforward combination of DP-GD and zeroth-order methods has a large dimension dependence.
Our novel \DPZero overcomes this issue with two key insights elaborated below.

\paragraph{Scalar privacy noise.}
By decoupling zeroth-order gradients in Eq.~\eqref{eq.zo} into direction and magnitude, our key observation is that the direction, $u_t$, is public knowledge, and we only need to make the magnitude private. Privacy can be guaranteed by clipping the finite-difference, $(f(x_t+\lambda u_t;\xi_i) - f(x_t-\lambda u_t;\xi_i))/(2\lambda)$, and then adding a \textit{scalar} noise $z_t$; see line 3 of Algorithm~\ref{algo:d-free}.
This change, when applied to Algorithm~\ref{algo:d-dependent}, can significantly improve the rate in Eq.~\eqref{eq:d-dependent-rank} by a factor of $d^{1/2}$.

\paragraph{Tighter clipping threshold.}
Another factor of $d^{1/2}$ improvement originates from a tighter analysis on the upper bound of the finite-difference term.
Although its worst-case upper bound scales with the dimension $d$, this only happens with an exponentially small probability over the randomness of $u_t$.
As proved in Eq.~\eqref{eq:d-free-C-upper-bound} in Appendix \ref{app:d-free}, the size of the finite-difference is
\begin{equation*}
    \frac{\abs{f(x_t+\lambda u_t;\xi_i) - f(x_t-\lambda u_t;\xi_i)}}{2\lambda} \; \leq\;
    \abs{u_t^\top \nabla f(x_t;\xi_i)} + \frac{\ell}{2}\lambda d,
\end{equation*}
where we use the assumption that each $f(x;\xi)$ is $\ell$-smooth.
When $u_t$ is sampled from the sphere $\sqrt{d}\,\bS^{d-1}$, a tail bound (part $(ii)$ of Lemma \ref{lm:sphere} in the appendix) implies that
\begin{equation*}
    \bP\left(\abs*{u_t^\top\nabla f(x_t;\xi_i)}\geq C\right) \leq 2\sqrt{2\pi}\,\exp\Big(-\frac{C^2}{8L^2}\Big).
\end{equation*}
By selecting the smoothing parameter $\lambda$ to be sufficiently small, a careful choice of $C=\tilde\cO(L)$, which is nearly independent of $d$, can ensure that clipping does not occur with a high probability. 
This choice is significantly smaller than the worst-case clipping threshold of $Ld^{1/2}$. 
The main technical challenge is that we need to analyze the algorithm given the event that clipping does not happen. The choice of drawing $u_t$ from the uniform distribution over the sphere, together with corresponding tail bounds in Appendix~\ref{app:technical}, allows us to prove the following nearly dimension-independent bound under the low effective rank structure in Assumption \ref{asp:rank}.
A proof is provided in Appendix \ref{app:d-free}.

\begin{algorithm}[t]
    \caption{\DPZero}
    \begin{algorithmic}[1]
        \REQUIRE Dataset $S=\{\xi_1, \cdots, \xi_n\}$, initialization $x_0\in\bR^d$, number of iterations $T$, stepsize $\alpha>0$, smoothing parameter $\lambda>0$, clipping threshold $C>0$, privacy parameters $\eps>0, \delta\in(0,1)$.
        
        \FOR{$t=0, 1, \cdots, T-1$}
            \STATE Sample $u_t$ uniformly at random from the Euclidean sphere $\sqrt{d}\,\bS^{d-1}$.
            
            \STATE Sample a scalar $z_t\in\bR$ randomly from the univariate Gaussian distribution $\cN(0, \sigma^2)$ with variance $\sigma=4C\sqrt{2T\log(e + (\eps/\delta))}/(n\eps)$ and update the parameter
            \begin{equation*}
                x_{t+1} \;\gets\; x_t - \alpha\Big(
                    \frac{1}{n}\sum_{i=1}^n \clip_C\Big(\frac{f(x_t+\lambda u_t;\xi_i) - f(x_t-\lambda u_t;\xi_i)}{2\lambda}\Big) + z_t
                \Big)u_t.
            \end{equation*}
        \ENDFOR
        
        \ENSURE $x_\tau$ for $\tau$ sampled uniformly at random from $\{0,1,\cdots,T-1\}$.
    \end{algorithmic}
    \label{algo:d-free}
\end{algorithm}

\begin{theorem}
    For any $\eps>0$ and $\delta\in(0,1)$, Algorithm \ref{algo:d-free} is $(\eps, \delta)$-DP. Under Assumption \ref{asp:rank}, suppose $\max_{0\leq t\leq T}\abs{F_S(x_t)}\leq B$, the output $x_\tau$ satisfies that
    \begin{equation}
        \bE[\norm{\nabla F_S(x_\tau)}^2] \;\; \leq \;\;
        \Big(
            64\Big( (F_S(x_0) - F_S^*) \, \ell + \tilde L^2\Big) + 2L^2
        \Big)
        \frac{\sqrt{r\log(e+(\eps/\delta))}}{n\eps},
        \label{eq:d-free}
    \end{equation}
    where we define
    \begin{equation*}
        \tilde L^2 = L^2\,\log\Big(\frac{2\sqrt{2\pi}\, n^3\eps^2 (r+2)(d + 8\ell B(r+2)/L^2)}{r\log(e+(\eps/\delta))}\Big),
    \end{equation*}
    and choose the parameters to be
    \begin{align*}
        & \alpha =\frac{1}{4\ell(r+2)}, \quad
        T=\frac{n(r+2)\eps}{4\sqrt{r\log(e+(\eps/\delta))}}, \quad
        C=4\tilde L, \\
        & \lambda \leq \frac{1}{\ell d}\min\Big\{
            4(2-\sqrt{2})\tilde L, \;
            \frac{L}{\sqrt{d}}\Big(\frac{\sqrt{r\log(e+(\eps/\delta))}}{n\eps}\Big)^{\frac12}
        \Big\}.
    \end{align*}
    The total number of zeroth-order gradient computations is $nT=\cO(n^2\sqrt{r})$.
    \label{thm:d-free}
\end{theorem}

\begin{remark}
    \label{rem:highrank}
    Algorithm \ref{algo:d-free} is nearly dimension-independent, given its logarithmic dependence on $d$. To the best of our knowledge, this is the first zeroth-order DP method that is nearly dimension-independent.
    This feature is significantly beneficial for fine-tuning pretrained LLMs where the effective rank has been observed to be quite small \citep{aghajanyan2021intrinsic, li2022does}.
    When $r=d$, our rate in Eq.~\eqref{eq:d-free} nearly matches that of the best known achievable bound of the first-order method DP-GD for smooth nonconvex losses \cite{wang2017differentially}. When the effective rank $r$ is smaller, this algorithm achieves $\tilde \cO(\sqrt{r\log (1/\delta)}/(n\eps))$ squared gradient norm.
    Similar dimension-free error rate is established for DP-GD on unconstrained generalized linear losses \citep{song2021evading}, with a dependence on the rank of the feature matrix.
    Table \ref{tab:compare} provides a summary on how \DPZero depends on dimension $d$ and effective rank $r$.
\end{remark}

\begin{table}[ht]
    \centering
    \caption{The dependence of the error rate on dimension $d$ and effective rank $r$ shows that the proposed \DPZero (Algorithm \ref{algo:d-free}) significantly outperforms DPGD-0th (Algorithm \ref{algo:d-dependent}) and achieves performance close to the popular first-order method, DP-GD, on both scenarios with and without a low-effective rank assumption.
    Note that the error rates of zeroth and first-order DP methods are achieved with different number of iterations.}
    \begin{tabular}{ccc}
        \toprule
        & without Assumption \ref{asp:rank}
        & with Assumption \ref{asp:rank} \\
        \midrule
        DPGD-0th 
        & $\cO(d\sqrt{d})$
        & $\cO(d\sqrt{r})$ \\
        \DPZerons 
        & $\cO((\log d)\sqrt{d})$
        & $\cO((\log d)\sqrt{r})$ \\
        \midrule
        DP-GD
        & $\cO(\sqrt{d})$
        & $\cO(\sqrt{r})$ \\
        \bottomrule
    \end{tabular}
    \label{tab:compare}
\end{table}

\begin{remark}
    The RHS of Eq.~\eqref{eq:d-free} improves upon Eq.~\eqref{eq:d-dependent-rank} of Algorithm~\ref{algo:d-dependent} by a factor of $d$.  Simplifying our analysis in  Eq.~\eqref{eq:onestep3} and conditioned on the event that the clipping does not happen, we get a similar one-step update analysis as Eq.~\eqref{eq:onestep2simple} (see Eq.~\eqref{eq:onestep3} and \eqref{eq:onestep3b} for a precise inequality).
    However, since the privacy noise $z_t$ is a scalar and the clipping threshold has been reduced, we have that $\bE[\|z_tu_t\|_H^2] = \tilde\cO(r)$ is nearly independent of the dimension $d$, and thus the final error scales as $\tilde\cO(r^{1/2})$.
\end{remark}

\begin{remark}
    The strategy of appropriately selecting the clipping threshold to ensure that clipping occurs with low probability is commonly applied in the analysis of private algorithms \citep{fang2023improved, shen2023share}. Adaptive choices of clipping thresholds can provably improve error rates for certain problems including PCA \cite{liu2022dp} and linear regression \cite{liu2023near}. 
    One technical challenge in the choice of the clipping threshold in \DPZero  is that we need the {\em expected} one-step progress to be sufficient in Eq.~\eqref{eq:onestep3}. This requires controlling the progress in the low-probability event that finite difference is clipped. The fact that $\|u_t\|$ is finite with probability one simplifies the analysis, which is the reason we choose to sample $u_t$ uniformly at random over the sphere.
    We believe that the analysis extends to the commonly used spherical Gaussian random vectors, which we leave as a future research direction. Table \ref{tab:noise} in the appendix supports our hypothesis that the resulting performances are similar whether Gaussian or spherical random vectors are used. We choose Gaussian vectors for our experiments in Section \ref{sec:exp} for simplicity.
\end{remark}

\begin{remark}
    Our theoretical results, including Theorems \ref{thm:d-dependent}, \ref{thm:d-dependent-rank}, and \ref{thm:d-free}, can be extended to the setting where the average loss $F_S(x)$ additionally satisfies the PL inequality \citep{karimi2016linear, polyak1963gradient, lojasiewicz1963topological}. Under Assumption \ref{asp:rank}, \DPZero converges to an optimal solution in a nearly dimension-independent error rate. See more details in Appendix \ref{app:pl}.
\end{remark}

\begin{remark}
    \label{rmk:clipping}
    Per-sample clipping is essential in DP algorithms to ensure bounded sensitivity that determines the magnitude of the DP noise.
    Besides the dimension-free error rates and memory saving of no backpropagation, another practical merit of \DPZero stems from the significantly simplified clipping compared with DP-GD.
    In addition to the advantage of clipping a \textit{scalar} function value difference rather than a gradient vector as required by first-order methods, the efficiency of \DPZero is mainly attributed to the low-cost \textit{per-sample} operations.
    In DP first-order methods, clipping is applied to gradients for every sample in a batch. The straightforward method of performing backward steps for each sample to compute its gradient loses the benefit of parallelization, leading to significant memory and runtime overhead. Despite extensive effort in improving the efficiency of per-sample gradient clipping \citep{li2022large, he2023exploring, bu2023differentially}, these methods still incur extra costs compared to non-DP algorithms.
    However, the clipping in \DPZero only involves computing the per-sample loss from forward steps and incurs no overhead in memory and runtime.
    This is straightforward for implementation as it is directly supported by, e.g., PyTorch, and no additional techniques are required.
    \DPZero is thus the first private method for fine-tuning LLMs that achieves near-zero additional costs compared to non-DP baselines, which is highly preferable especially in resource-constrained scenarios.
\end{remark}

\section{Experiments} \label{sec:exp}

We provide empirical results on synthetic problems and private fine-tuning of language models for sentence classification and generation tasks. A thorough description of the experimental settings is available in Appendix \ref{app:exp}. All experiments are tested on a single NVIDIA GeForce RTX 3090 GPU with 24 GiB memory. Code is available at \url{https://github.com/Liang137/DPZero}.

\subsection{Synthetic Example}

\begin{figure}[t]
    \centering
    \begin{subfigure}{0.32\textwidth}
        \includegraphics[width=\textwidth]{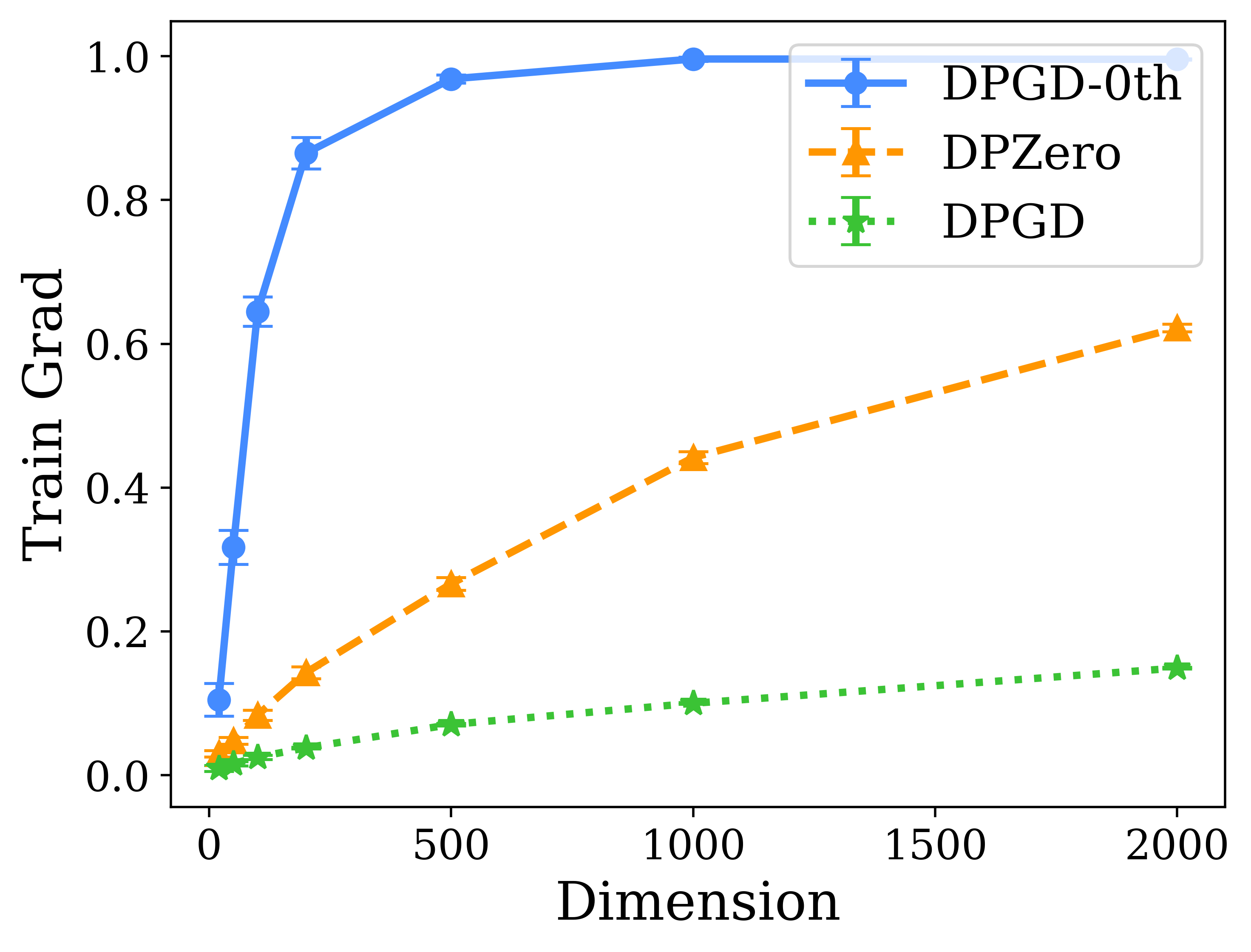}
        \caption{$\tr(A)=\cO(d)$.}
    \end{subfigure}
    \hfill
    \begin{subfigure}{0.32\textwidth}
        \includegraphics[width=\textwidth]{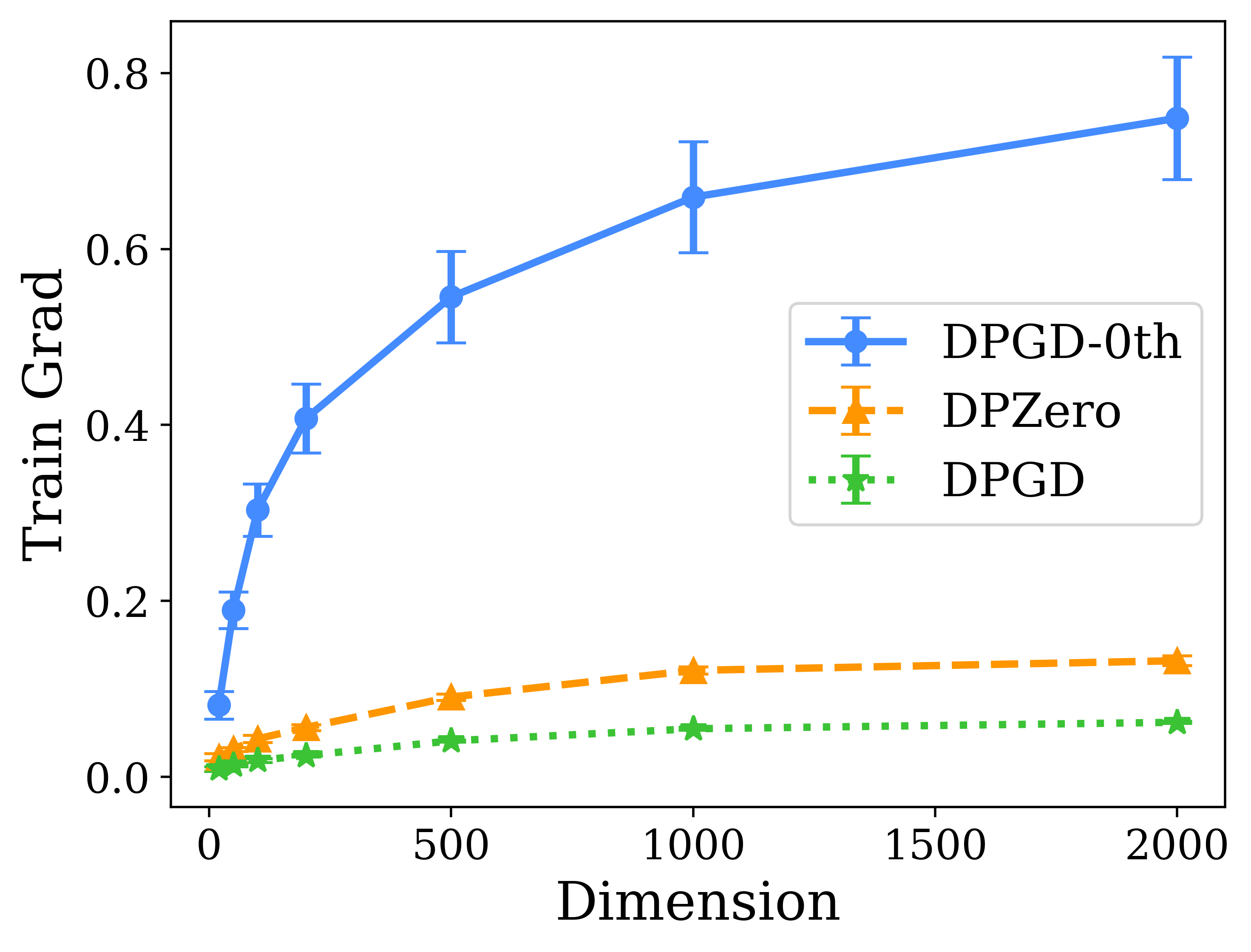}
        \caption{$\tr(A)=\cO(\sqrt{d})$.}
    \end{subfigure}
    \hfill
    \begin{subfigure}{0.32\textwidth}
        \includegraphics[width=\textwidth]{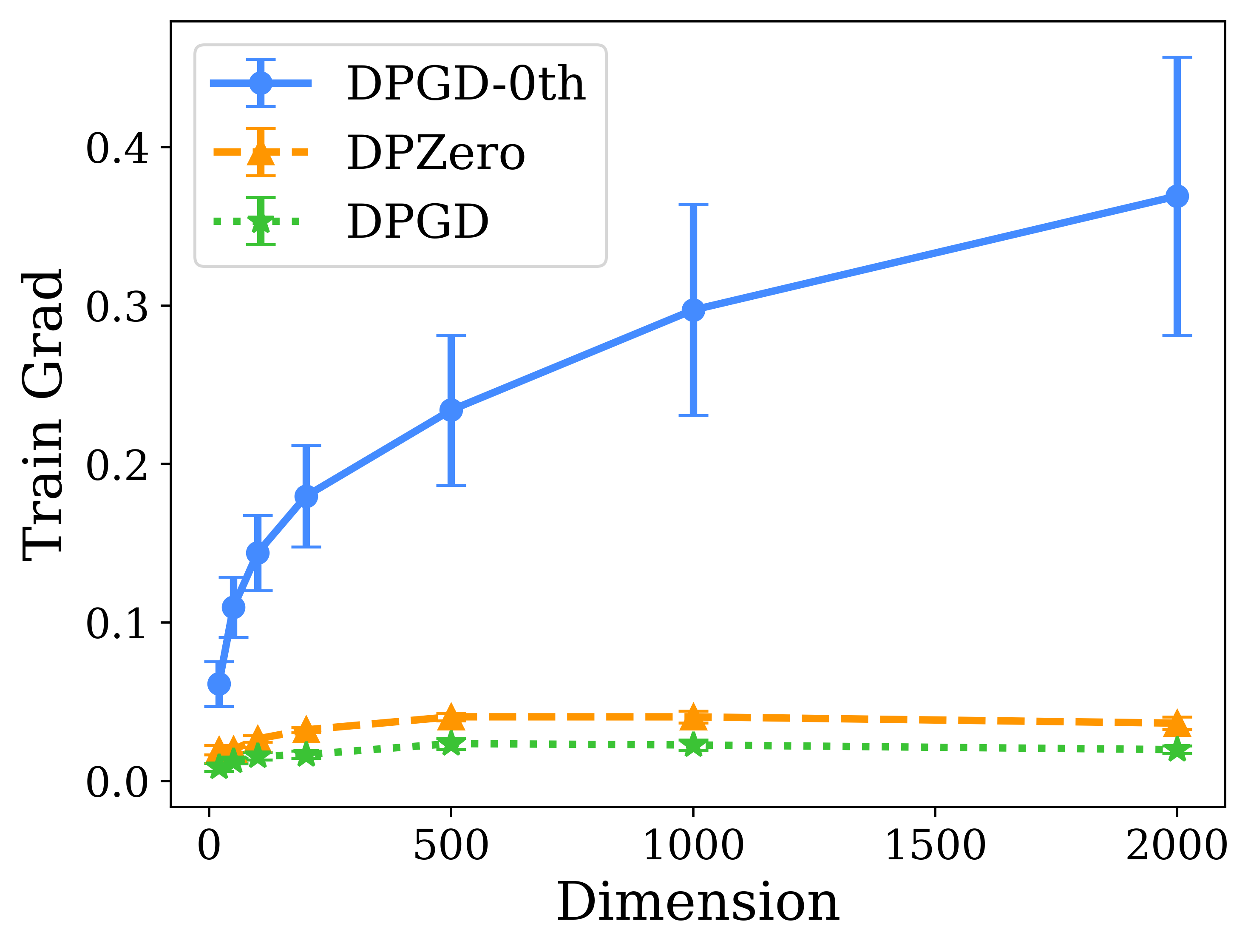}
        \caption{$\tr(A)=\cO(\log d)$.}
    \end{subfigure}
    
    \begin{subfigure}{0.32\textwidth}
        \includegraphics[width=\textwidth]{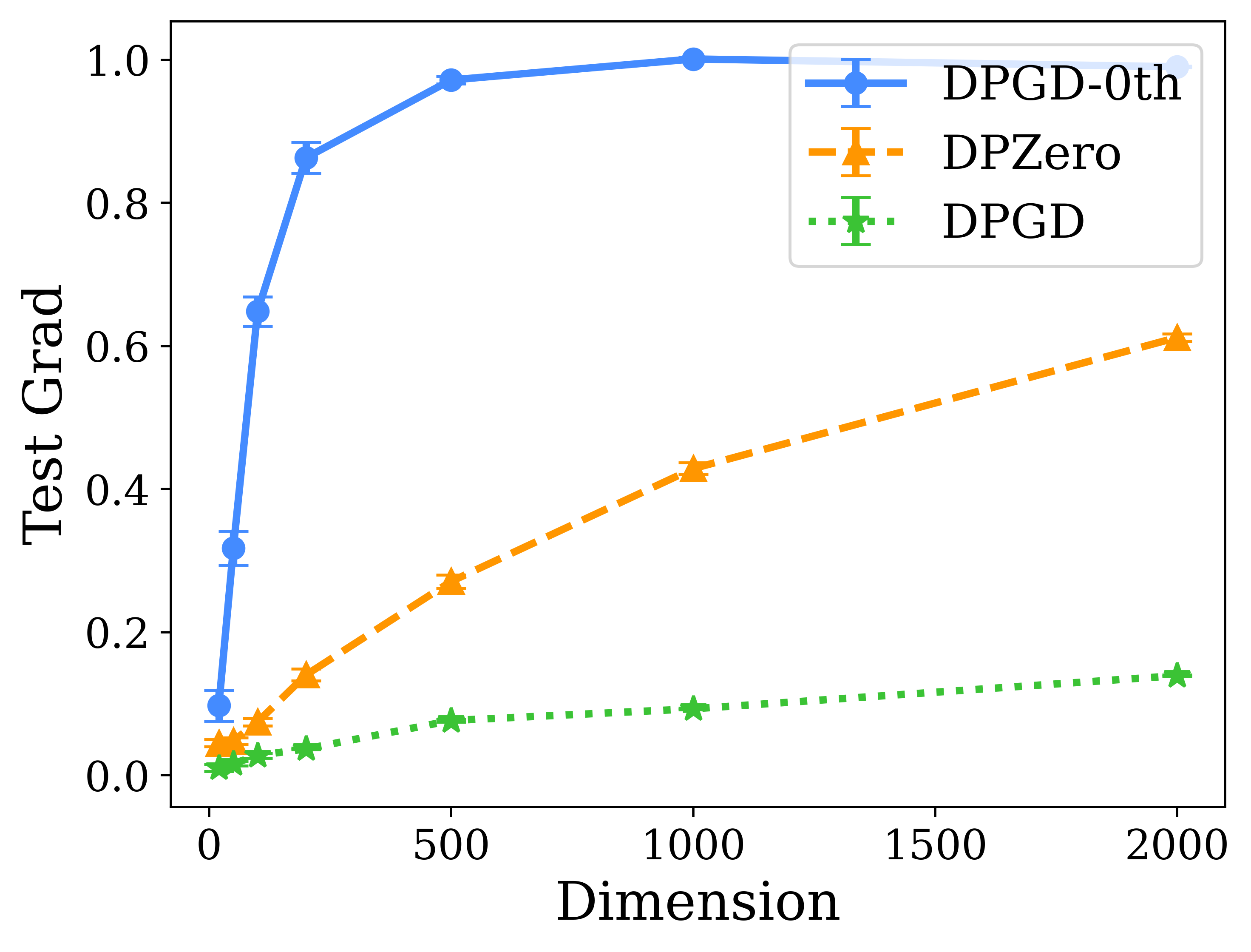}
        \caption{$\tr(A)=\cO(d)$.}
    \end{subfigure}
    \hfill
    \begin{subfigure}{0.32\textwidth}
        \includegraphics[width=\textwidth]{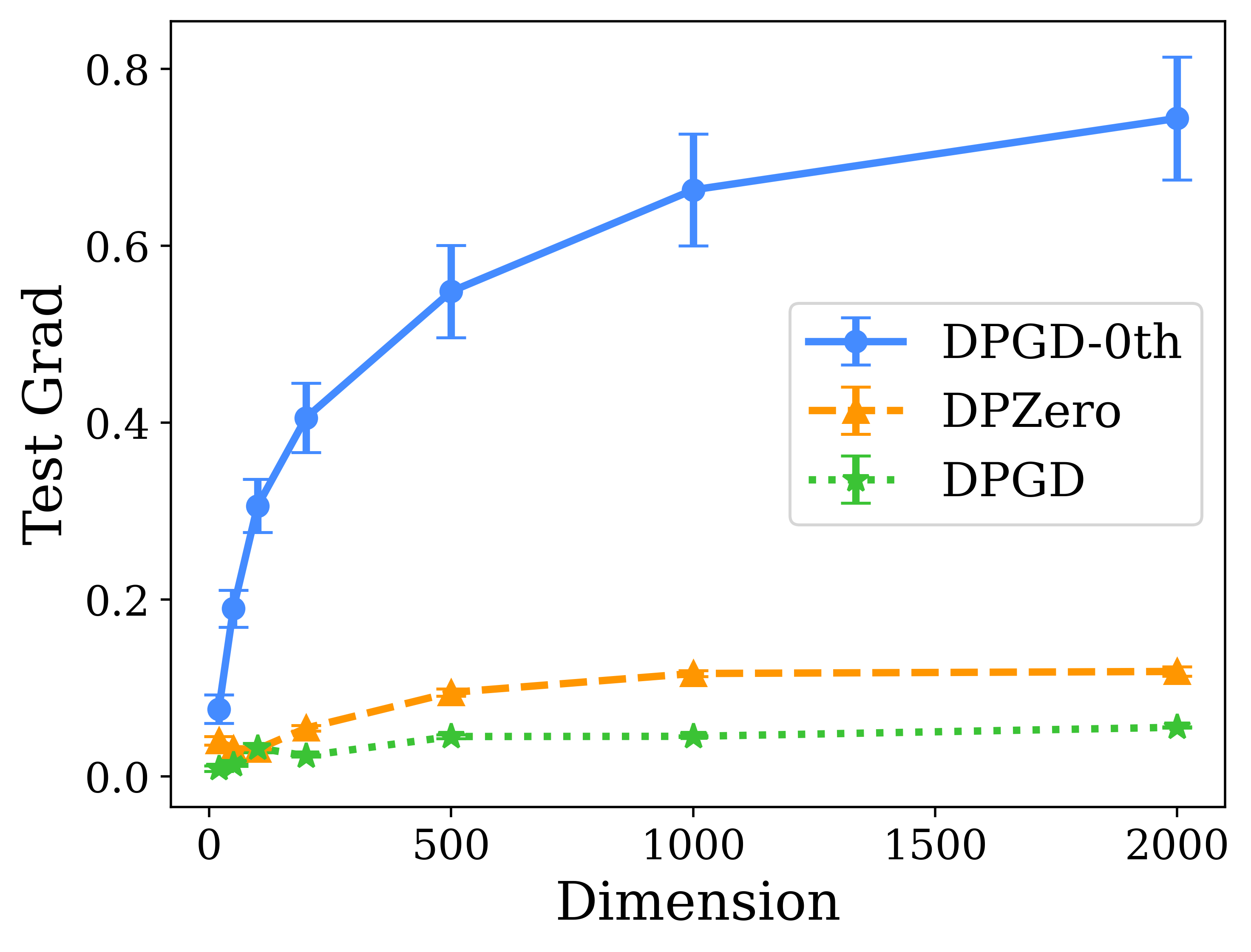}
        \caption{$\tr(A)=\cO(\sqrt{d})$.}
    \end{subfigure}
    \hfill
    \begin{subfigure}{0.32\textwidth}
        \includegraphics[width=\textwidth]{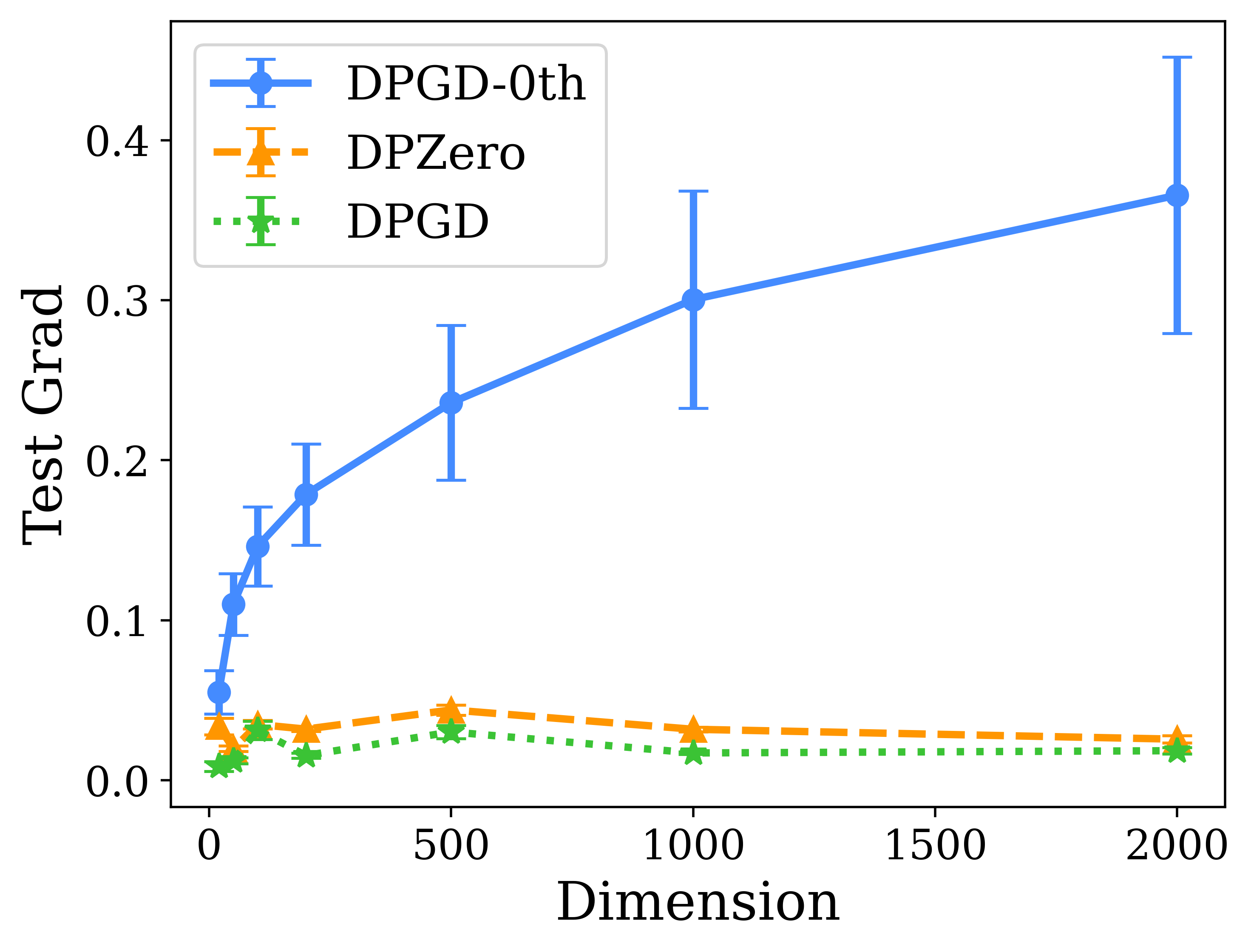}
        \caption{$\tr(A)=\cO(\log d)$.}
    \end{subfigure}
    \caption{Experiments on the quadratic loss with effective rank $\tr(A)$ (Assumption \ref{asp:rank}). For three different modes of the effective rank, we demonstrate how the norm of the train ((a), (b), and (c)) and test ((d), (e), and (f)) gradient depends on the problem dimension. DPGD-0th (Algorithm \ref{algo:d-dependent}) has a strong dimension dependence regardless of the effective rank, while \DPZero (Algorithm \ref{algo:d-free}) achieves dimension-independent performance when effective rank is small (right panel), similar to the standard first-order method DP-GD. Insights for the saturation of DPGD-0th when the dimension increases can be found in Remark \ref{rmk:bound}.}
    \label{fig:toy_grad}
\end{figure}

Our first evaluation compares the performance of Algorithm \ref{algo:d-free} (\DPZerons) with Algorithm \ref{algo:d-dependent} (DPGD-0th) and DP-GD on problems with different effective ranks. In particular, we use a quadratic loss
\begin{equation*}
    \min_{x\in\bR^d} F_S(x) = \frac{1}{2n}\sum_{i=1}^n (x - x_i)^\top A (x - x_i),
\end{equation*}
with three choices of the Hessian matrix, $A$, whose effective ranks are designed to be ${\cal O}(d)$, ${\cal O}(\sqrt{d})$, and ${\cal O}(\log d)$, respectively.
All methods are trained with $(\eps=2, \delta=10^{-6})$-DP on a training set $\{x_1,\cdots, x_n\}$ with $n=10,000$ and evaluated on a test set of the same size. The problem dimension is increased from 20 to 2,000. We perform a parameter search and plot the best gradient norm evaluated on both the training set and the test set in Figure \ref{fig:toy_grad}.
Every method scales with the dimension $d$ when the effective rank is $d$ (as in Figures \ref{fig:toy_grad}(a) and \ref{fig:toy_grad}(d)), and DPGD-0th has the worst performance. When the effective rank reduces to $\log d$ (as in Figures \ref{fig:toy_grad}(c) and \ref{fig:toy_grad}(f)), both DP-GD and \DPZero become nearly dimension-independent, which validates the dimension independence of \DPZerons. Appendix \ref{app:toy} includes more results measuring the loss for both training and test datasets.

\subsection{Fine-tuning on RoBERTa} 

\begin{table}[t]
    \centering
    \caption{Experiments on RoBERTa (355M). We report both mean and standard error of the accuracy (\%) across three random seeds.
    Zero-shot results with no fine-tuning provide lower bounds (taken from \citet{malladi2023fine}), since they can be achieved with no private data.
    MeZO is not private and serves as an upper bound of \DPZerons. 
    LoRA \citep{hu2022lora} and DP-LoRA adopt AdamW \citep{loshchilov2018decoupled} as their optimizer. All first-order methods (AdamW, LoRA, and their private versions) utilize the implementation by \citet{li2022large}. Thanks to \DPZerons, the performance gaps between zeroth and first-order methods are made smaller in private fine-tuning.}
    \begin{tabular}{ccccccc}
        \toprule
        \multirow{2}{*}{Task} & SST-2 & SST-5 & SNLI & MNLI & RTE & TREC \\
        & \multicolumn{2}{c}{------ {\small Sentiment} ------}
        & \multicolumn{3}{c}{------ {\small Natural Language Inference} ------}
        & \multicolumn{1}{c}{---  {\small Topic} ---} \\
        \midrule
        AdamW
        & $93.1\pm 0.3$
        & $56.6\pm 0.3$
        & $86.4\pm 0.8$
        & $81.4\pm 0.9$
        & $83.6\pm 1.6$
        & $95.9\pm 0.2$ \\
        DP-AdamW ($\eps=6$)
        & $91.6\pm 1.2$
        & $49.0\pm 0.3$
        & $81.5\pm 1.4$
        & $76.3\pm 0.9$
        & $77.3\pm 1.1$
        & $89.9\pm 0.8$ \\
        DP-AdamW ($\eps=2$)
        & $90.5\pm 1.5$
        & $47.5\pm 0.5$
        & $74.6\pm 1.0$
        & $70.3\pm 0.8$
        & $72.8\pm 0.9$
        & $85.0\pm 0.5$\\
        \midrule
        LoRA
        & $93.3\pm 0.4$
        & $55.3\pm 1.0$
        & $85.9\pm 0.7$
        & $82.2\pm 0.7$
        & $84.2\pm 0.4$
        & $94.6\pm 0.4$ \\
        DP-LoRA ($\eps=6$)
        & $91.0\pm 1.3$
        & $48.8\pm 0.5$
        & $81.0\pm 1.5$
        & $72.8\pm 1.8$
        & $74.7\pm 1.3$
        & $89.2\pm 0.8$ \\
        DP-LoRA ($\eps=2$)
        & $90.2\pm 1.2$
        & $47.1\pm 0.4$
        & $74.7\pm 1.6$
        & $65.7\pm 0.9$
        & $69.2\pm 1.1$
        & $83.2\pm 2.3$ \\
        \midrule
        MeZO
        & $92.5\pm 0.3$
        & $50.8\pm 0.8$
        & $80.4\pm 0.6$
        & $69.2\pm 0.3$
        & $72.8\pm 1.0$
        & $88.9\pm 0.1$ \\
        \DPZero ($\eps=6$)
        & $92.2\pm 0.3$
        & $49.3\pm 0.6$
        & $77.8\pm 1.0$
        & $67.4\pm 0.3$
        & $71.9\pm 0.9$
        & $87.6\pm 0.9$ \\
        \DPZero ($\eps=2$)
        & $91.8\pm 0.1$
        & $47.1\pm 0.9$
        & $73.6\pm 0.9$
        & $62.7\pm 0.9$
        & $70.4\pm 0.7$
        & $82.0\pm 1.6$ \\
        \midrule
        Zero-Shot
        & 79.0
        & 35.5
        & 50.2
        & 48.8
        & 51.4
        & 32.0 \\
        \bottomrule
    \end{tabular}
    \label{tab:dpzero-mezo}
\end{table}

\begin{table}[t]
    \centering
    \caption{Runtime per iteration (s) and memory consumption (MiB) when fine-tuning RoBERTa (355M) for SST-2.
    Private methods in the table ensure $(\eps=2, \delta=10^{-5})$-DP.
    \DPZero is as memory and runtime efficient as the non-private zeroth-order method MeZO \citep{malladi2023fine}.
    First-order methods DP-AdamW and DP-LoRA (AdamW as the optimizer) both introduce considerable memory and runtime overhead compared to their non-private baselines.
    All first-order methods use the implementation by \citet{li2022large}. Comparisons with other implementations of DP first-order methods can be found in Table \ref{tab:others} in the appendix.}
    \begin{tabular}{ccc}
        \toprule
        Method & Time (s/iter) & Memory (MiB) \\
        \midrule
        AdamW
        & 1.25
        & 15820 \\
        DP-AdamW
        & 2.12
        & 17126 \\
        \midrule
        LoRA
        & 0.821
        & 10366 \\
        DP-LoRA
        & 1.05
        & 10496 \\
        \midrule
        MeZO
        & 0.345
        & 2668 \\
        \DPZerons
        & 0.347
        & 2668 \\
        \bottomrule
    \end{tabular}
    \label{tab:dpero-dpgd}
\end{table}

Next, we follow the experimental setting in \citet{malladi2023fine} and evaluate \DPZero on fine-tuning RoBERTa \citep{liu2019roberta} with 355M parameters across six different sentence classification tasks. We consider the few-shot scenario with 512 samples per class. We report the test accuracy for \DPZero trained with $(\eps=\{2,6\}, \delta=10^{-5})$-DP and non-private zeroth-order baseline MeZO \citep{malladi2023fine} and compare them with first-order methods in Table \ref{tab:dpzero-mezo}.
The memory consumption and per-iteration runtime are shown in Table \ref{tab:dpero-dpgd}.
DP first-order methods introduce additional overhead in both memory and runtime compared to non-DP baselines, with a maximum accuracy drop of 9.5\% when $\eps=6$. However, \DPZero enjoys the same benefit as MeZO on memory efficiency and achieves near-zero additional costs, with at max only a 2.6\% drop in the accuracy.
In our experiments, we notice that the clipping threshold of \DPZero is typically larger compared to DP first-order methods; see Figure \ref{fig:clip_lr} in the appendix. This is consistent with the results in Theorem \ref{thm:d-free} regarding the selection of the clipping threshold $C$.

Compared with DP first-order methods, the main benefit of \DPZero is memory efficiency. Such memory savings are even greater than those observed in non-DP domains, thanks to \DPZerons's efficient clipping (cf. Remark \ref{rmk:clipping}).
We note that the aim of Table \ref{tab:dpero-dpgd} is to explain that DP first-order methods need considerable memory and runtime overhead compared to non-DP methods, while \DPZero does not. Such comparisons happen between DP and non-DP algorithms, respectively. We do not intend to directly compare the runtime of \DPZero to DP first-order methods as it depends on the implementation. In general, zeroth-order methods require more iterations to attain the same level of performance as first-order methods \citep{malladi2023fine}. In our case, DP first-order methods take 1,000 iterations while \DPZero need 10,000 iterations.
This aligns with Theorem \ref{thm:d-free}, which states that \DPZero requires $\cO(r)$ times more iterations than DP-GD to attain the same level of error rate, where $r$ is the effective rank.
However, \DPZero can still be efficient for large models in terms of GPU hours, because first-order methods often require communication-heavy distributed training over more GPUs each with limited memory; see Appendix F.6 of \citet{malladi2023fine}.

\subsection{Fine-tuning on OPT}

\begin{table}[t]
    \centering
    \caption{Experiments on OPT for classification tasks. We report mean and standard error of the accuracy (\%) across three random seeds.}
    \begin{tabular}{ccccccc}
        \toprule
        Model & \multicolumn{2}{c}{OPT-1.3B} & \multicolumn{2}{c}{OPT-2.7B} & \multicolumn{2}{c}{OPT-6.7B} \\
        Task & SST-2 & BoolQ & SST-2 & BoolQ & SST-2 & BoolQ \\
        \midrule
        MeZO
        & $88.2\pm 0.9$
        & $63.2\pm 0.8$
        & $91.9\pm 0.5$
        & $65.3\pm 1.3$
        & $93.0\pm 0.2$
        & $67.4\pm 2.3$ \\
        \midrule
        \DPZero ($\eps=6$)
        & $88.2\pm 1.1$
        & $62.4\pm 0.8$
        & $91.5\pm 1.7$
        & $65.4\pm 1.6$
        & $92.6\pm 0.7$
        & $66.8\pm 1.6$ \\
        \DPZero ($\eps=2$)
        & $86.8\pm 1.7$
        & $61.6\pm 1.1$
        & $90.5\pm 0.9$
        & $63.7\pm 0.7$
        & $90.6\pm 1.3$
        & $63.7\pm 0.7$ \\
        \midrule
        Zero-Shot
        & 53.6
        & 45.3
        & 56.3
        & 47.7
        & 61.2
        & 59.4 \\
        \bottomrule
    \end{tabular}
    \label{tab:opt-cls}
\end{table}

\begin{table}[t]
    \centering
    \caption{Experiments on OPT for generation tasks. We report both mean and standard error of the f1 score (\%) across three random seeds.}
    \begin{tabular}{ccccccc}
        \toprule
        Model & \multicolumn{2}{c}{OPT-1.3B} & \multicolumn{2}{c}{OPT-2.7B} & \multicolumn{2}{c}{OPT-6.7B} \\
        Task & SQuAD & DROP & SQuAD & DROP & SQuAD & DROP \\
        \midrule
        MeZO
        & $73.5\pm 1.2$
        & $24.4\pm 0.2$
        & $76.3\pm 0.8$
        & $25.5\pm 1.2$
        & $79.7\pm 1.1$
        & $28.8\pm 0.7$ \\
        \midrule
        \DPZero ($\eps=6$)
        & $72.6\pm 0.8$
        & $24.7\pm 1.0$
        & $75.7\pm 1.5$
        & $24.6\pm 0.5$
        & $79.5\pm 0.9$
        & $28.4\pm 1.3$ \\
        \DPZero ($\eps=2$)
        & $70.1\pm 1.6$
        & $23.9\pm 1.2$
        & $71.9\pm 1.2$
        & $23.1\pm 0.9$
        & $77.1\pm 1.0$
        & $27.6\pm 0.7$ \\
        \midrule
        Zero-Shot
        & 26.8
        & 11.1
        & 29.8
        & 9.7
        & 36.5
        & 17.8 \\
        \bottomrule
    \end{tabular}
    \label{tab:opt-gen}
\end{table}

We also provide experiments on fine-tuning OPT \citep{zhang2022opt} in the few-shot setting to illustrate the scalability of \DPZerons. On our device (a GPU with 24 GiB memory), the largest model that can fit in for zeroth-order methods is OPT-6.7B, while first-order methods already run out of memory for OPT-1.3B; see Table \ref{tab:opt-memory} in the appendix for a detailed comparison of the memory consumption. The results of \DPZerons's test performance on four downstream tasks are reported in Tables \ref{tab:opt-cls} and \ref{tab:opt-gen}. \DPZero demonstrates the same level of scalability as MeZO, with the ability to fine-tune models wherever MeZO is applicable, and experiences only small drops in performance due to privacy (up to 0.9\% when $\eps=6$). Our results indicate the effectiveness of \DPZero for privately fine-tuning pretrained LLMs and confirm that it does not suffer in high dimensions.

\section{Conclusion}

\DPZero is proposed to privately fine-tune language models in a memory efficient manner by avoiding backpropagation. Theoretically, \DPZero enjoys a provably near dimension-free rate under low-rank structures, clearing the barriers for scaling private fine-tuning of LLMs. When deploying \DPZerons, the elimination of gradient computation not only significantly saves memory, but avoids the overhead in gradient clipping as well. Thus the benefit of using  zeroth-order method is more significant for private optimization. 
The theoretical guarantees on scalability and the practical merits of \DPZero are validated on private fine-tuning of RoBERTa and OPT on several downstream tasks.

\DPZero uses the full batch gradient every iteration, and the analysis guarantees an upper bound on the empirical average gradient assuming smooth nonconvex objectives. We defer extensions to the stochastic mini-batch setting, guarantees on the population loss leveraging the stability of zeroth-order methods \citep{nikolakakis2022black}, and considerations of other assumptions on objective functions like convexity or nonsmoothness to future research.
We believe this work opens up a plethora of other prospective directions in DP zeroth-order optimization. These include, but are not limited to, understanding advantages of the intrinsic noise in zeroth-order gradient estimators, discovering other structural assumptions like the restricted Lipschitz condition \citep{li2022does} for dimension-independent rates, exploring alternative private mechanisms for the privacy guarantees of \DPZero (e.g., the Laplace mechanism for pure DP \citep{tang2024private}), and utilizing momentum \citep{tran2022momentum} or variance reduction \citep{arora2023faster} techniques for an improved rate and computational complexity.

\section*{Acknowledgements}
We are grateful to Gavin Brown and Divyansh Pareek for their insightful discussions regarding the proofs. We also thank Fanny Yang for proofreading of the paper. Additionally, we thank all anonymous reviewers for their valuable suggestions.
L.Z. gratefully acknowledges funding by the Max Planck ETH Center for Learning Systems (CLS).
This work does not relate to the current position of K.T. at Amazon. 
N.H. is supported by ETH research grant funded through ETH Zurich Foundations and Swiss National Science Foundation Project Funding No. 200021-207343.
S.O. is supported in part by the National Science Foundation under grant no.~2019844, 2112471, and 2229876 supported in part by funds provided by the National Science Foundation, by the Department of Homeland Security, and by IBM. Any opinions, findings, and conclusions or recommendations expressed in this material are those of the author(s) and do not necessarily reflect the views of the National Science Foundation or its federal agency and industry partners.

\section*{Impact Statement}
A major concern with current use-cases of large language models is privacy of the fine-tuning data. Fine-tuning on in-domain data greatly improves performance and is now a default option. However, in-domain data can contain sensitive information about the participants of the dataset. 
The proposed solution makes privacy protection easier, consuming less resources, thus democratizing the use of privacy enhancing technology beyond those who have access to large amounts of resources.

\bibliography{ref}
\bibliographystyle{plainnat}

\newpage
\appendix

\section{Additional Related Works}
\label{app:review}

\paragraph{Zeroth-order optimization.}
\citet{nesterov2017random} pioneered the formal analysis of the convergence rate of zeroth-order methods, i.e., zeroth-order (stochastic) gradient descent (ZO-SGD) that replaces gradients in SGD by their zeroth-order estimators.
This is motivated by renewed interest in adopting zeroth-order methods in industry due to, for example, fast differentiation techniques that require storing all intermediate computations reaching the memory limitations. Their findings on nonsmooth convex functions are later refined by \citet{shamir2017optimal}. \citet{lin2022gradient} contributed to further advancements on nonsmooth nonconvex functions recently. Additionally, \citet{ghadimi2013stochastic} extended the results for smooth functions into the stochastic setting. Zeroth-order methods have also been expanded to incorporate approaches such as coordinate descent \citep{lian2016comprehensive}, conditional gradient descent \citep{balasubramanian2018zeroth}, variance reduction techniques \citep{liu2018zeroth, fang2018spider, ji2019improved}, SignSGD \citep{liu2018signsgd}, and minimax optimization \citep{wang2022zeroth}.
Additionally, zeroth-order methods find applications in fields such as black-box machine learning \citep{grill2015black, chen2017zoo, chen2019zo}, bandit optimization \citep{flaxman2005online, shamir2017optimal}, reinforcement learning \citep{salimans2017evolution, choromanski2018structured, mania2018simple}, and distributed learning~\citep{fang2022communication, zelikman2023just, xu2023federated} to reduce communication overhead.

These well-established results indicate that the norm of the zeroth-order gradient scales with the dimension $d$ and the required stepsize is $d$-times smaller than that in first-order gradient-based methods, leading to a $d$-times increase in the final time complexity. For example, the convergence rate of gradient descent for minimizing a smooth convex function $f(x)$ is $f(\bar x_T) - \min_{x\in\bR^d} f(x) \leq \cO(1/T)$ where $\bar x_T$ is the average of $T$ iterates \citep{nesterov2003introductory}, while the zeroth-order method only achieves a rate $\cO(d/T)$. It has been shown that such dimension dependence of zeroth-order methods is inevitable without additional structures \citep{wibisono2012finite, duchi2015optimal}.

There are several recent works that relax the dimension dependence in zeroth-order methods leveraging problem structures. \citet{wang2018stochastic} and \citet{cai2022zeroth} assumed certain sparsity structure in the problem and applied sparse recovering algorithms, e.g. LASSO, to obtain sparse gradients from zeroth-order observations. \citet{golovin2020gradientless} analyzed the case when the objective function is $f(Px)$ for some low-rank projection matrix $P$. These works either require the objective or the algorithm to be modified to have a dimension-independent guarantee. \citet{balasubramanian2018zeroth} demonstrated that ZO-SGD can directly identify the sparsity of the problem and proved a dimension-independent rate when the support of gradients remains unchanged \citep{cai2022zeroth}. Recently, \citet{yue2023zeroth} and \citet{malladi2023fine} relaxed the dependence on dimension $d$ to a quantity related to the trace of the loss's Hessian.

\paragraph{Differentially private optimization.}
Previous works on DP optimization mostly center around first-order methods. For constrained convex problems, tight utility guarantees on both excess empirical \cite{chaudhuri2011differentially, bassily2014private, wu2017bolt, zhang2017efficient, wang2017differentially} and population \cite{bassily2019private, bassily2020stability, feldman2020private, asi2021private, kulkarni2021private, zhang2022bring} losses are well-understood. As an example, a typical result states that the optimal rate on the excess empirical loss for convex objectives is $\Theta(\sqrt{d\log(1/\delta)}/(n\eps))$, where $(\eps,\delta)$ are privacy parameters, $n$ is the number of samples, and $d$ is the dimension. The dimension dependence is fundamental as both the upper bound \cite{bassily2014private}, using differentially private (stochastic) gradient descent (DP-GD) introduced in \cite{song2013stochastic}, and the lower bound \cite{bassily2014private}, using a reduction to finger printing codes, have the same dependence.

When the problem is nonconvex, i.e., the setting of our interest, DP-GD achieves a rate of $\cO(\sqrt{d\log(1/\delta)}/(n\eps))$ on the squared norm of the gradient \citep{wang2017differentially, zhou2020private}. We show that \DPZero matches this rate with access only to the zeroth-order oracle in Theorem~\ref{thm:d-free}.
Given access to the first-order oracle, it has been recently shown that such rate can be improved to $\cO((\sqrt{d\log(1/\delta)}/(n\eps))^{4/3})$ leveraging momentum \citep{tran2022momentum} or variance reduction techniques \citep{arora2023faster}.
Further, the convergence to second-order stationary points in nonconvex DP optimization is studied in \citep{ganesh2023private}. 
Recent advancements in DP optimization have also delved into the understanding of the potential of public data \citep{ganesh2023public, lowy2023optimal}, the convergence properties of per-sample gradient clipping \citep{yang2022normalized, fang2023improved, koloskova2023revisiting, zhang2023differentially}, and the relaxation of the dimension dependence in the utility upper bound \citep{ma2022dimension, li2022does}.

Early works established that dimension-independent rates can be attained when the gradients lie in some fixed low-rank subspace \citep{jain2014near, song2021evading}. By first identifying this gradient subspace, dimension-independent algorithms can be designed \citep{zhou2021bypassing, kairouz2021nearly}. 
Closest to our result is \citet{song2021evading}, which demonstrated that the rate of DP-GD for smooth nonconvex optimization can be improved  to $\cO(\sqrt{r\log(1/\delta)}/(n\eps))$ under certain structural assumptions, i.e., for generalized linear models (GLMs) with a rank-$r$ feature matrix. \DPZero matches this result with access only to the zeroth-order oracle in Theorem~\ref{thm:d-free} for more general problems beyond low-rank GLMs.
Our result is inspired by \citet{li2022does} that introduced a relaxed Lipschitz condition for the gradients and provided dimension-free bounds when the loss is convex and the relaxed Lipschitz parameters decay rapidly. Similarly, \citet{ma2022dimension} suggested that the dependence on $d$ in the utility upper bound for DP stochastic convex optimization can be improved  to a dependence on the trace of the Hessian.
There is also a line of work on DP Riemannian optimization that achieves utility bounds dependent on the intrinsic dimension of the manifold \citep{reimherr2021differential, utpala2023differentially, utpala2023improved, han2024differentially}. Further exploration of its connection to the low-rank structure in this work is reserved for future.

Literature on DP optimization beyond first-order methods remains less explored. \citet{ganesh2023faster} investigated the potential of second-order methods for DP convex optimization. \citet{gratton2021privacy} proposed to use zeroth-order methods for DP-ADMM \citep{huang2019dp} in distributed learning. They state that the noise intrinsic in zeroth-order methods is enough to provide privacy guarantee and rely on the output of zeroth-order methods being Gaussian, which is unverified to the best of our knowledge. \citet{liu2023generating} proposed a private genetic algorithm based on zeroth-order optimization heuristics for private synthetic data generation.
Recently, \citet{zhang2024private} studied the problem of private zeroth-order nonsmooth nonconvex optimization and achieved a rate that depends on the dimension $d$.
After the workshop version of our paper \citep{zhang2023dpzero} was released, \citet{tang2024private} concurrently discovered the same algorithm as \DPZero (up to a minor difference in how $u_t$ is drawn) and showed empirical benefits when applied to fine-tuning OPT models but without theoretical analysis.
Also building upon the workshop version of our paper, \citet{liu2024differentially} introduced DP-ZOSO, a stage-wise zeroth-order method with an additional quadratic regularizer.
With extra hyper-parameters to be tuned, DP-ZOSO demonstrates further empirical gain over \DPZerons. However, \citet{liu2024differentially} only provided dimension-dependent guarantees.
As far as we are aware, no prior studies have addressed the challenge of deriving a dimension-independent rate in DP zeroth-order optimization.

\paragraph{Other relevant works.}
\citet{du2023dp} introduced a novel noise adding mechanism that happens in the forward pass of training. Although the algorithm is termed ``DP-Forward'', it still requires backpropagation for training.
In a separate context, \citet{bu2023zero} coincidentally proposed DP-ZeRO, a term identical to ours, denoting a private version of the zero redundancy optimizer (ZeRO) by \citet{rajbhandari2020zero} that aims at enhancing memory efficiency in data and model parallelisms.
DP prompt tuning \citep{hong2024dpopt} and DP in-context learning \citep{tang2023privacy} provide resource-efficient alternatives compared to private fine-tuning, enabling the private adaptation of pretrained LLMs to specific tasks without extensive computational demands. Investigating how \DPZero performs relative to these methods and whether different techniques can be integrated is an interesting research problem.
More recently, \citet{chen2024privacy} proposed differentially private algorithms that enforce weight flatness to improve generalization, which can also handle zeroth-order oracles.
There is also another line of research \citep{guha2013nearly, tossou2016algorithms, shariff2018differentially} on the design of differentially private algorithms for the stochastic bandit problem based on upper confidence bound \citep{auer2002finite}. Their algorithms are not directly applicable to our setting.

\section{Additional Experiment Details}
\label{app:exp}

In this section, we discuss our experimental setups in detail.

\subsection{Synthetic Example on a Quadratic Loss}
\label{app:toy}

\begin{figure}[t]
    \centering
    \begin{subfigure}{0.32\textwidth}
        \includegraphics[width=\textwidth]{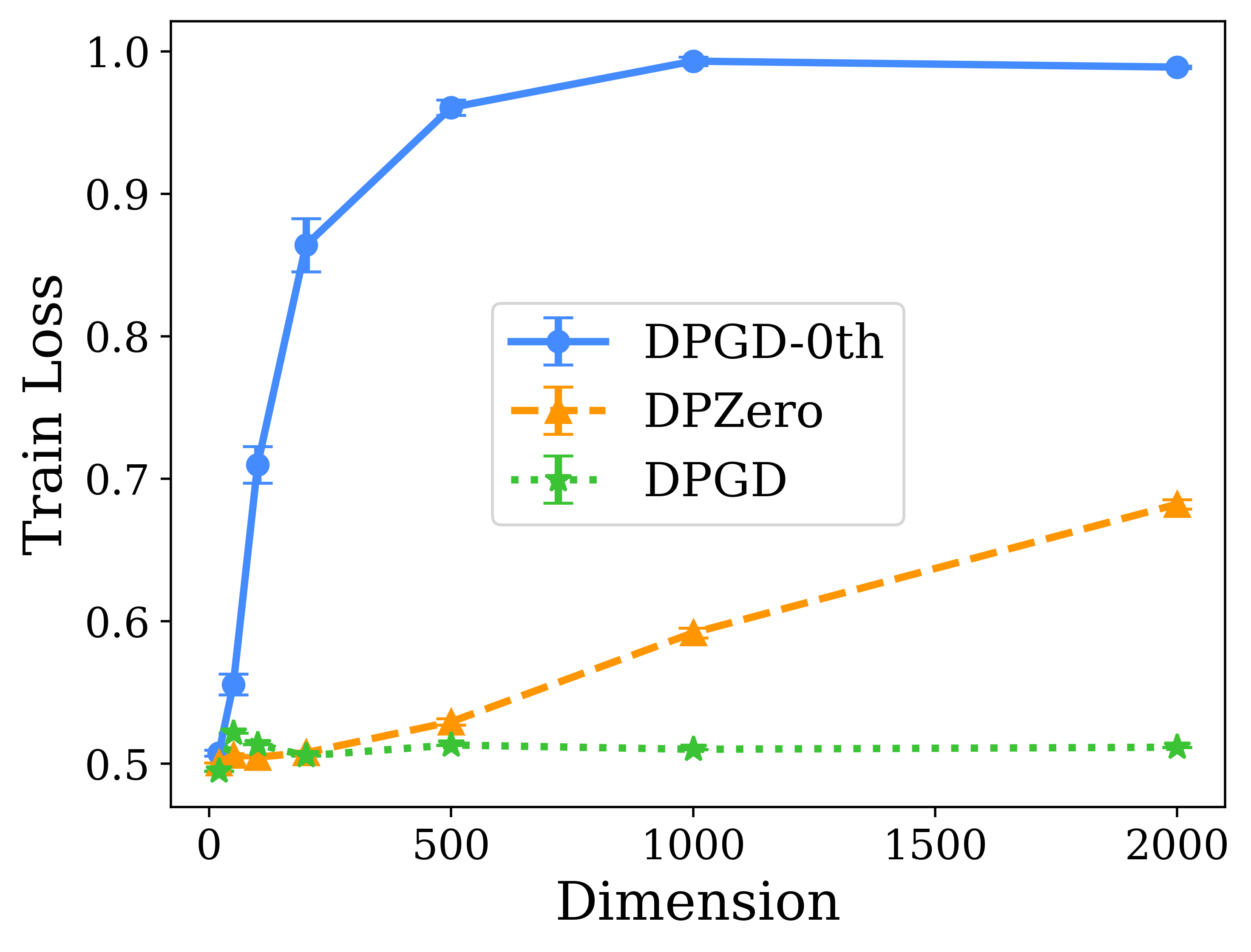}
        \caption{$\tr(A)=\cO(d)$.}
    \end{subfigure}
    \hfill
    \begin{subfigure}{0.32\textwidth}
        \includegraphics[width=\textwidth]{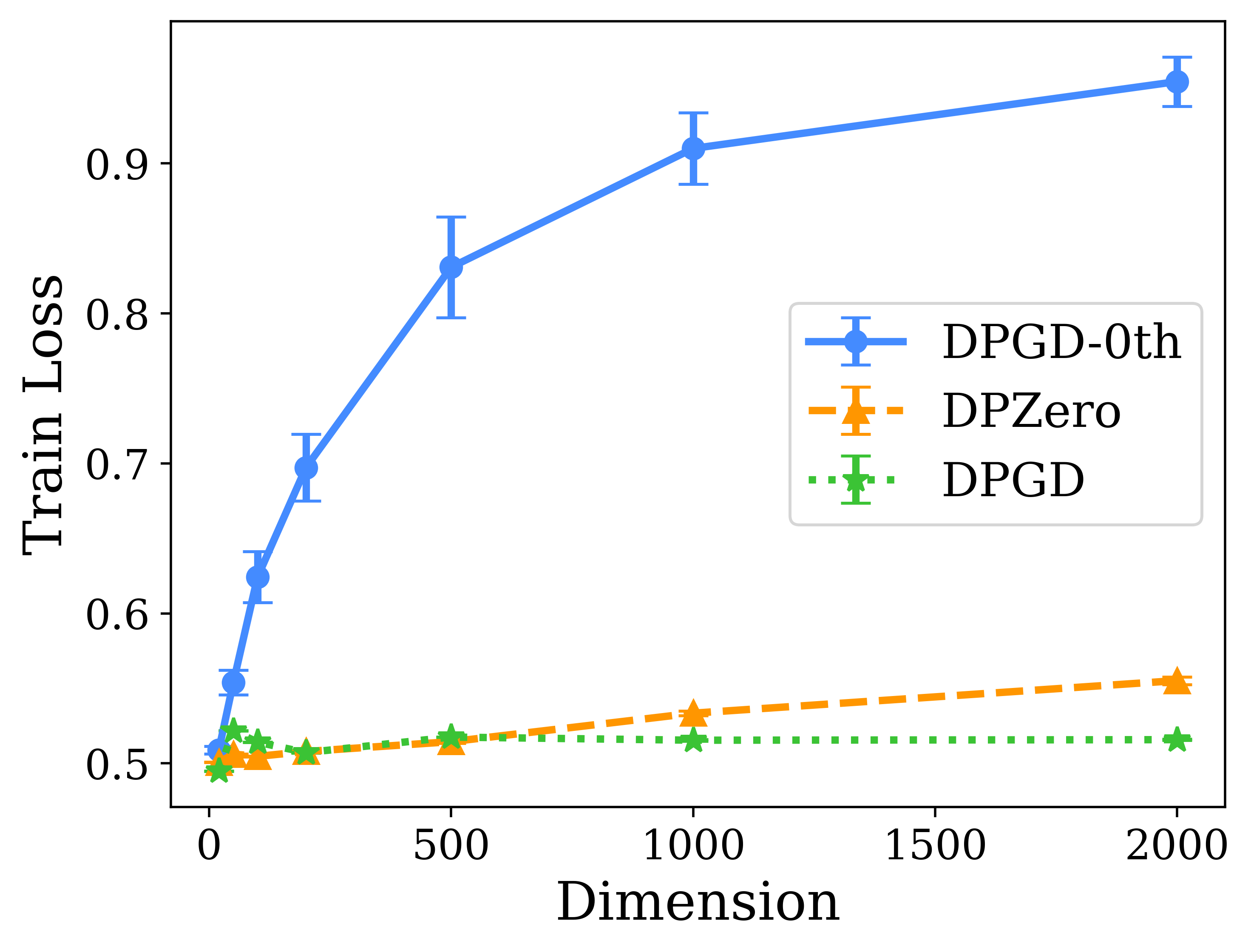}
        \caption{$\tr(A)=\cO(\sqrt{d})$.}
    \end{subfigure}
    \hfill
    \begin{subfigure}{0.32\textwidth}
        \includegraphics[width=\textwidth]{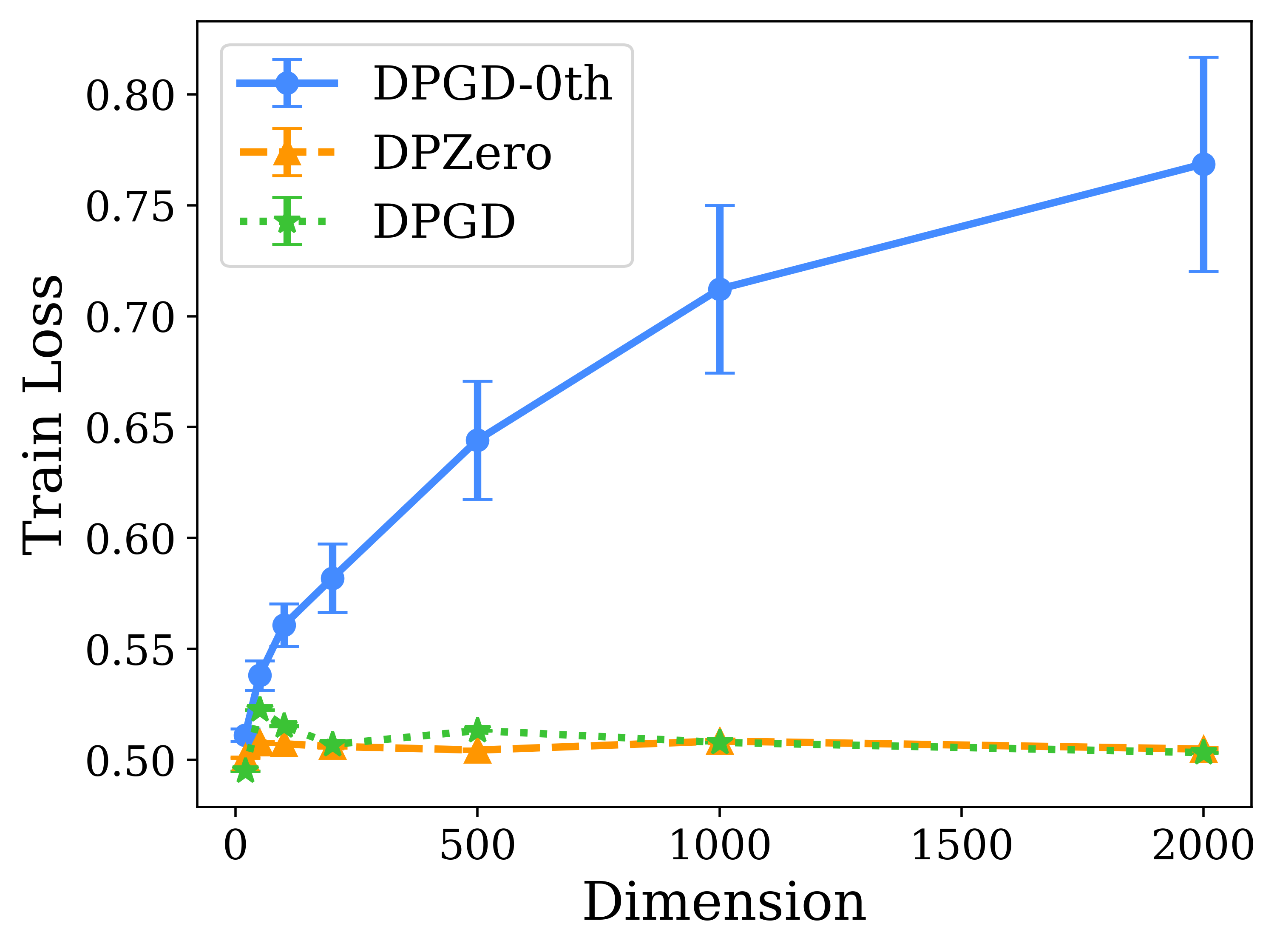}
        \caption{$\tr(A)=\cO(\log d)$.}
    \end{subfigure}

    \begin{subfigure}{0.32\textwidth}
        \includegraphics[width=\textwidth]{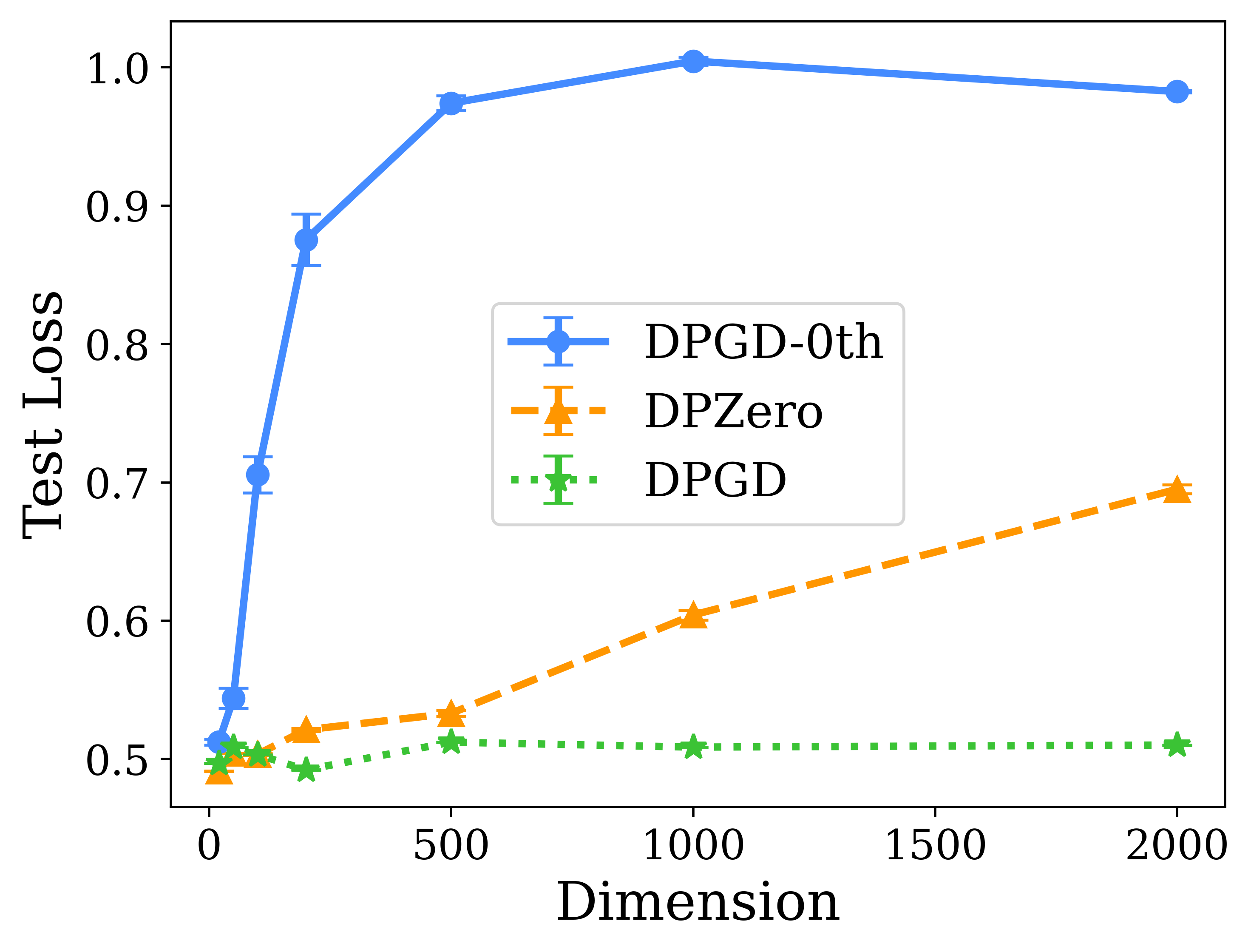}
        \caption{$\tr(A)=\cO(d)$.}
    \end{subfigure}
    \hfill
    \begin{subfigure}{0.32\textwidth}
        \includegraphics[width=\textwidth]{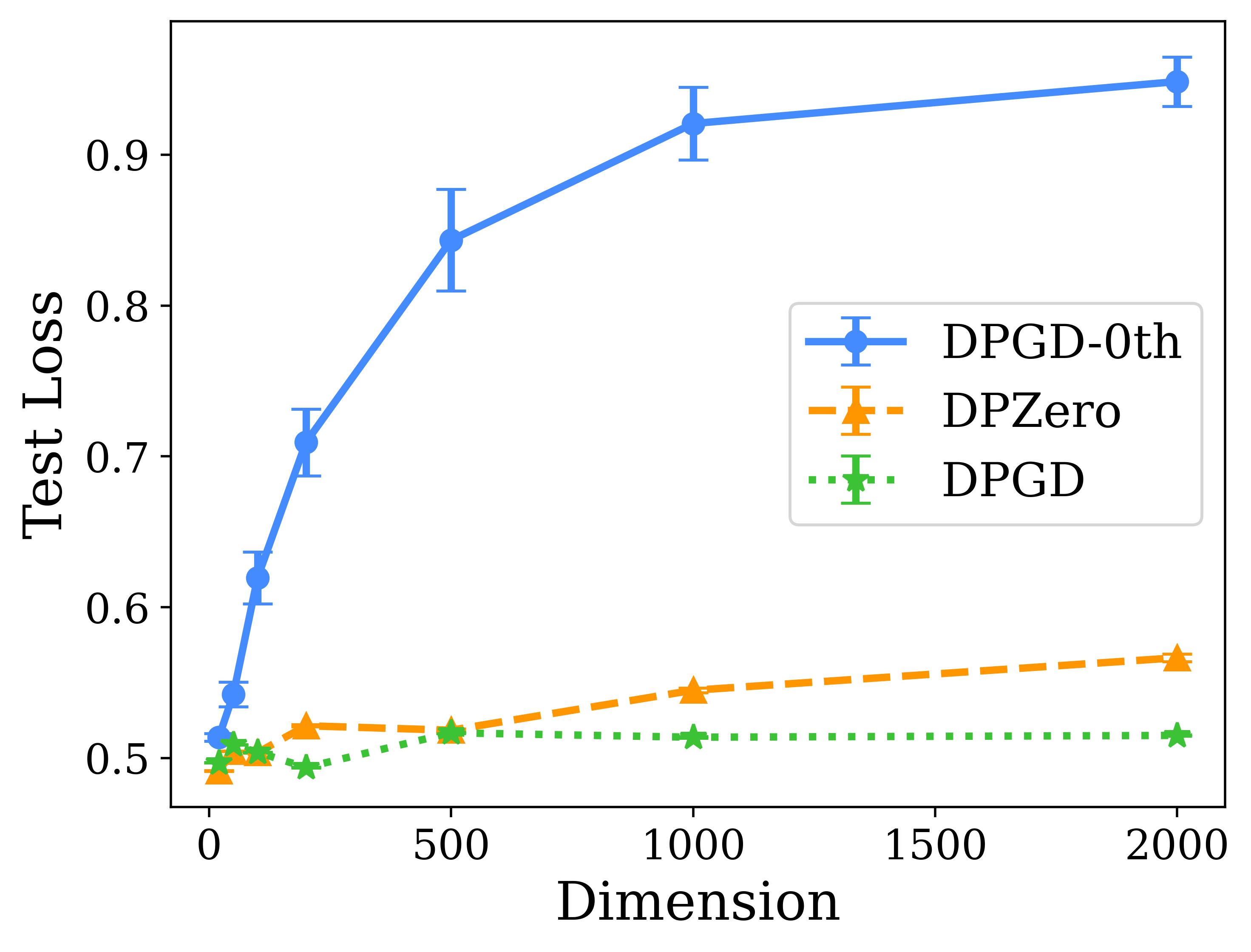}
        \caption{$\tr(A)=\cO(\sqrt{d})$.}
    \end{subfigure}
    \hfill
    \begin{subfigure}{0.32\textwidth}
        \includegraphics[width=\textwidth]{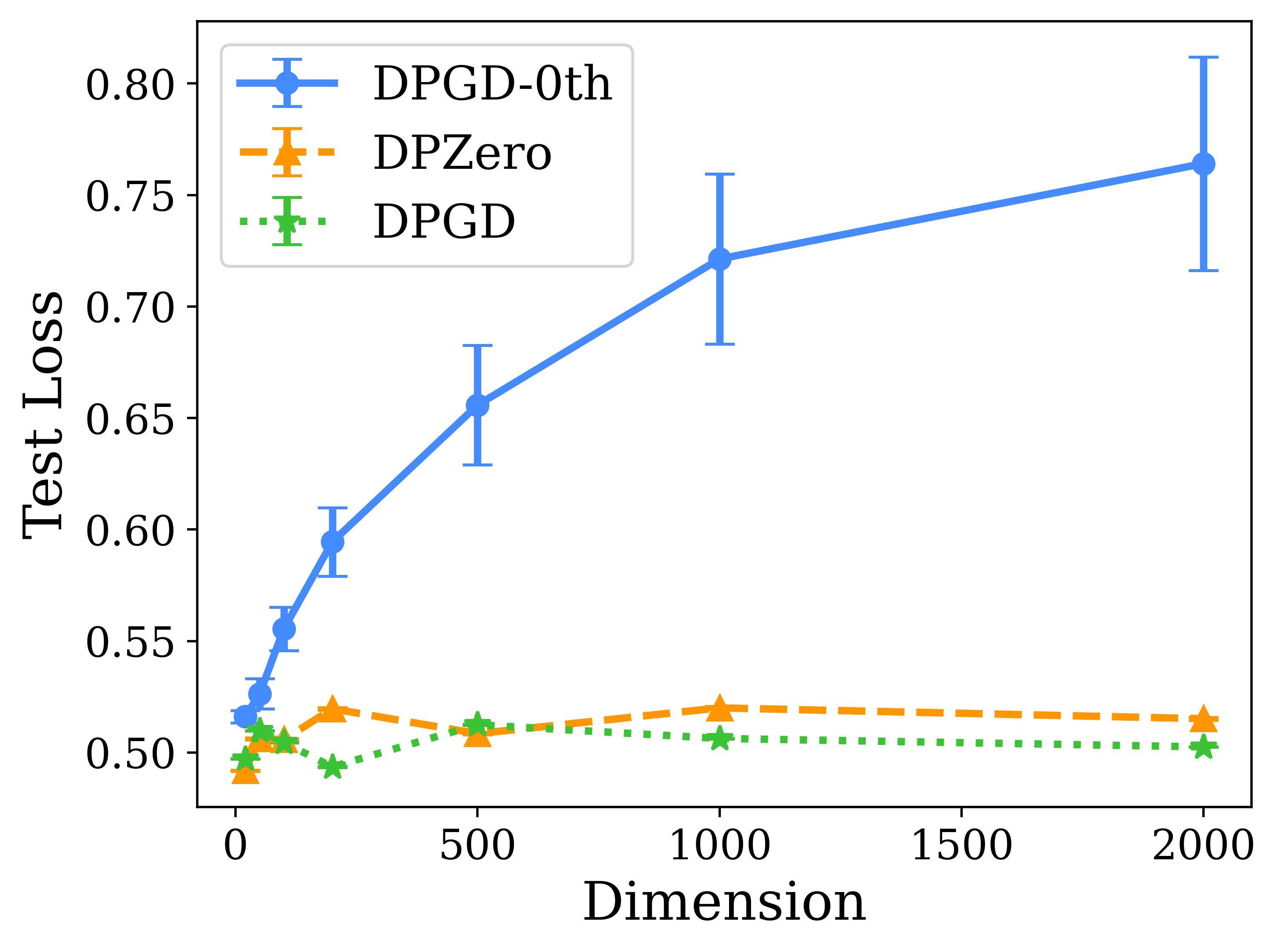}
        \caption{$\tr(A)=\cO(\log d)$.}
    \end{subfigure}
    \caption{Experiments on the quadratic loss with effective rank $\tr(A)$. For three different modes, we increase the dimension and report the best loss evaluated on both training set ((a), (b), and (c)) and test set ((d), (e), and (f)).}
    \label{fig:toy_loss}
\end{figure}

Given a training dataset $S=\{x_1,\cdots,x_n\}$ with each coordinate of $x_i\in\bR^d$ sampled independently from the Gaussian $\cN(1,1)$, we implement \DPZero on the quadratic loss
\begin{equation*}
    \min_{x\in\bR^d} F_S(x) = \frac{1}{2n}\sum_{i=1}^n (x - x_i)^\top A (x - x_i),
\end{equation*}
with a fixed Hessian $A\in\bR^{d\times d}$ that can be designed to implement different effective ranks $r=\tr(A)/\norm{A}_2$ according to Assumption \ref{asp:rank}.
We compare \DPZero (Algorithm \ref{algo:d-free}) with DPGD-0th (Algorithm \ref{algo:d-dependent}) and first-order algorithm DP-GD on three patterns of the effective rank
\begin{align*}
    & \text{(a)} \; \tr(A) = \cO(d): \; A = \text{diag}\{1, 1, \cdots, 1\}; \\
    \\[-0.15in]
    & \text{(b)} \; \tr(A) = \cO(\sqrt{d}): \; A = \text{diag}\{1, 1/\sqrt{2}, \cdots, 1/\sqrt{d}\}; \\
    \\[-0.12in]
    & \text{(c)} \; \tr(A) = \cO(\log d): \; A = \text{diag}\{1, 1/2, \cdots, 1/d\}.
\end{align*}
Since $\norm{A}_2=1$ in all cases, the effective rank $r=\tr(A)$. For each mode of the effective rank, we increase the problem dimension $d$ from 20 to 2000. We perform a parameter search and plot the best gradient norm evaluated on the training set and a test set that follows the same distribution of the training set in Figure \ref{fig:toy_grad}. For completeness, we also plot both training and test loss in Figure \ref{fig:toy_loss}. The key hyper-parameters used for the experiments are summarized in Table \ref{tab:toy-param}.

\begin{table}[ht]
    \centering
    \caption{Hyper-parameters used for the synthetic example on the quadratic loss. The number of iterations, stepsize, and clipping threshold are optimized through a grid search using given values. Other parameters are fixed to the values.}
    \begin{tabular}{cc}
        \toprule
        Hyper-parameters & Values \\
        \midrule
        Number of training samples & 10000 \\
        Number of test samples & 10000 \\
        Dimension $d$ & $\{20, 50, 100, 200, 500, 1000, 2000\}$ \\
        Privacy & $(\eps=2, \delta=10^{-6})$ \\
        Smoothing $\lambda$ (\DPZero and DPGD-0th)  & $10^{-4}$ \\
        \midrule
        Number of iterations & $\{10, 20, 40, 80, 160, 320, 640, 1280, 2560, 5120\}$ \\
        Stepsize & $\{10^{-5}, 3\times10^{-5}, 10^{-4}, 3\times10^{-4}, 0.001, 0.003, 0.01, 0.03, 0.1, 0.3, 1\}$ \\
        Clipping & $\{0.1, 0.3, 1, 3, 10, 30, 100, 300\}$ \\
        \bottomrule
    \end{tabular}
    \label{tab:toy-param}
\end{table}

In all figures, we observe that the performance of each method is improved with smaller effective rank. For each pattern of the effective rank, DPGD-0th (Algorithm \ref{algo:d-dependent}) has the worst performance, while DP-GD consistently achieves the best results. When the effective rank is $d$, every method scales with the dimension. When the effective rank improves to $\log d$, \DPZero and DP-GD become nearly dimension-independent, and \DPZero matches the performance of the first-order method DP-GD. This validates our theoretical findings, as summarized in Table \ref{tab:compare}, and demonstrates the effectiveness of \DPZerons. We want to mention that a similar set of experiments to verify the performance of DP-GD when dimension increases was also provided by \citet{li2022does}. Our implementation of this synthetic example is based on their code.

\subsection{Private Fine-Tuning of the Language Model RoBERTa}

We follow experiment settings in \citet{malladi2023fine} to evaluate the performance of \DPZero in the private fine-tuning of RoBERTa \citep{liu2019roberta} across six sentence classification datasets: SST-2 and SST-5 \citep{socher2013recursive} for sentiment classification, SNLI \citep{bowman2015large}, MNLI \citep{williams2018broad}, and RTE \citep{dagan2005pascal, haim2006second, giampiccolo2007third, bentivogli2009fifth, wang2018glue} for natural language inference tasks, and TREC \citep{voorhees2000building} for topic classification. In our experiments, we employ the same prompts as used in \citet{malladi2023fine}, which are adapted from \citet{gao2021making}.

\paragraph{Implementation details.} Our implementation of \DPZero utilizes the codebase provided by \citet{malladi2023fine}. For easier implementation and better memory efficiency, we follow \citet{malladi2023fine} to sample the zeroth-order direction $u_t$ from the Gaussian distribution $\cN(0, \rI_d)$ instead of the sphere as stated in Algorithm \ref{algo:d-free}. Table \ref{tab:noise} compares the performance of \DPZero on SST-2 and SST-5 when $u_t$ is sampled from Gaussian and sphere. Given the negligible differences between the two sampling strategies, we continue with the Gaussian sampling for its simplicity. Another strategy in the implementation to further save memory involves storing only the random seed for the generation of the zeroth-order direction $u_t$, rather than the complete vector, and regenerating this direction whenever it's used. Although \DPZero is stated for the full-batch case in Algorithm \ref{algo:d-free}, we adopt a mini-batch setting in the experiments.

\begin{table}[ht]
    \centering
    \caption{Test accuracy (mean \% $\pm$ standard error \%) of \DPZero when fine-tuning RoBERTa (355M) for SST-2 and SST-5 with $(\eps=\{2,6\}, \delta=10^{-5})$-DP and using different sampling strategies of the zeroth-order update direction $u_t$. No notable difference is observed when $u_t$ is sampled from either the Gaussian distribution or the Euclidean sphere.}
    \begin{tabular}{ccccc}
        \toprule
        \multirow{2}{*}{Randomness}
        & \multicolumn{2}{c}{Gaussian}
        & \multicolumn{2}{c}{Sphere} \\
        \cline{2-5}
        & $\eps=6$ & $\eps=2$
        & $\eps=6$ & $\eps=2$ \\
        \midrule
        SST-2
        & $92.2\pm0.3$
        & $91.8\pm0.1$
        & $91.8\pm0.1$
        & $91.5\pm0.5$ \\
        \midrule
        SST-5
        & $49.3\pm0.6$
        & $47.1\pm0.9$
        & $49.9\pm1.3$
        & $47.4\pm1.3$ \\
        \bottomrule
    \end{tabular}
    \label{tab:noise}
\end{table}

\paragraph{Hyper-parameter selection.} For all experiments, we employ a few-shot setting, utilizing 512 samples per class in the training set, randomly selected from the original dataset. The test set is also composed of 1000 randomly selected samples from the original test dataset.
We fix the total number of iterations to be 10000, the batch size to be 64, and the smoothing parameter $\lambda=10^{-3}$ for both \DPZero and the non-private zeroth-order baseline MeZO \citep{malladi2023fine}. Note that the original results of MeZO reported in \citet{malladi2023fine} run for 100000 iterations.
A parameter search of the learning rate for MeZO is performed, and it turns out $10^{-6}$ consistently yields the best performance. We then fix the learning rate to be $10^{-6}$ for \DPZero and only search for the clipping threshold for different tasks. There is potential for improved performance by well-optimizing other hyper-parameters, such as the learning rate and the number of iterations.
All results are averaged through three different random seeds $\{42, 13, 21\}$ for selecting the few-shot datasets. The hyper-parameters used for our language model fine-tuning experiments are summarized in Table \ref{tab:llm-param}.

\begin{table}[t]
    \centering
    \caption{Hyper-parameters used in \DPZero for fine-tuning RoBERTa (355M). We only optimize the clipping threshold through a grid search from 50 to 400. Other parameters are fixed to the listed values.}
    \begin{tabular}{cc}
        \toprule
        Hyper-parameters & Values \\
        \midrule
        Number of training samples & 512 per class \\
        Number of test samples & 1000 \\
        Number of iterations & 10000 \\
        Batch size & 64 \\
        Privacy & $(\eps=\{2,6\}, \delta=10^{-5})$ \\
        Smoothing $\lambda$  & $10^{-3}$ \\
        Stepsize & $10^{-6}$ \\
        Clipping & \{50, 100, 150, 200, 250, 300, 400\} \\
        \bottomrule
    \end{tabular}
    \label{tab:llm-param}
\end{table}

\paragraph{Comparison with first-order methods.} Regarding the first-order methods, we use the same few-shot setting as before, and the results are averaged over three different random seeds $\{42, 13, 21\}$. The number of iterations is set to be 1000, and the batch size is fixed to be 64. The learning rate is optimized by a grid search over $\{5\times10^{-5}, 10^{-4}, 5\times10^{-4}, 10^{-3}\}$, and the clipping threshold is optimized by a grid search over $\{0.1, 0.5, 1, 10\}$. In the experiments for LoRA, we set the rank to be 8 and the LoRA $\alpha=16$, which remain the same as in the original paper  \citep{hu2022lora}. All other parameters are fixed to their default values. In addition to \citet{li2022large} in Tables \ref{tab:dpzero-mezo} and \ref{tab:dpero-dpgd}, we also compare the performance of \DPZero to two other implementations of DP first-order methods, \citet{yu2022differentially} and \citet{bu2023differentially}, in Table \ref{tab:others}. \DPZero achieves similar performance on SST-2 as DP first-order methods, while saving a significant amount of memory. Such memory savings are greater than the savings of MeZO \citep{malladi2023fine} over AdamW \citep{loshchilov2018decoupled} and LoRA \citep{hu2022lora} (AdamW as the optimizer), due to DPZero’s simpler clipping (cf. Remark \ref{rmk:clipping}).

\begin{table}[ht]
    \centering
    \caption{Test accuracy (\%), runtime per iteration (s), and memory consumption (MiB) when fine-tuning RoBERTa (355M) for SST-2.
    Private methods in the table guarantee $(\eps=2, \delta=10^{-5})$-DP. A fair comparison is ensured among \citet{li2022large} and \citet{bu2023differentially}, as they are implemented using the same codebase. It is important to note, however, that they cannot be directly compared with those of \citet{yu2022differentially}, due to differences in implementations. LoRA \citep{hu2022lora} and DP-LoRA use the first-order method AdamW \citep{loshchilov2018decoupled} as the optimizer. DP first-order methods introduce considerable overheads in both memory and runtime compared to their non-DP baselines, while \DPZero does not, thanks to its novel design of the efficient clipping. Also note that such comparisons between DP and non-DP algorithms are fair since they use the same codebase.}
    \begin{tabular}{cccc}
        \toprule
        Method & Acc. & Time (s/iter) & Memory (MiB) \\
        \midrule
        AdamW \citep{li2022large}
        & 93.1
        & 1.25
        & 15820 \\
        DP-AdamW \citep{li2022large}
        & 90.5
        & 2.12
        & 17126 \\
        DP-AdamW \citep{bu2023differentially}
        & 91.1
        & 1.55
        & 18372 \\
        \midrule
        AdamW \citep{yu2022differentially}
        & 94.4
        & 0.425
        & 16960 \\
        DP-AdamW \citep{yu2022differentially}
        & 92.3
        & 2.33
        & 21494 \\
        \midrule
        LoRA \citep{li2022large}
        & 93.3
        & 0.821
        & 10366 \\
        DP-LoRA \citep{li2022large}
        & 90.2
        & 1.05
        & 10496 \\
        \midrule
        LoRA \citep{yu2022differentially}
        & 94.3
        & 0.301
        & 11512 \\
        DP-LoRA \citep{yu2022differentially}
        & 91.3
        & 0.332
        & 11522 \\
        \midrule
        MeZO
        & 92.5
        & 0.345
        & 2668 \\
        \DPZerons
        & 91.8
        & 0.347
        & 2668 \\
        \bottomrule
    \end{tabular}
    \label{tab:others}
\end{table}

\paragraph{Comparison with DPGD-0th.} In the previous synthetic example, DPGD-0th suffers from worse performance in larger dimensions. To provide a more complete comparison, we also evaluate the performance of DPGD-0th (Algorithm \ref{algo:d-dependent}) for fine-tuning RoBERTa-large on the dataset TREC with a privacy budget of $\varepsilon=2$ (the same setting as Table \ref{tab:dpzero-mezo}). DPGD-0th only achieves a test accuracy of 67.0, while \DPZero attains 82.0. Moreover, DPGD-0th still requires per-sample clipping of the gradient estimator, which is costly in both memory and runtime compared to \DPZerons.

\begin{figure}[t]
    \centering
    \begin{subfigure}{0.32\textwidth}
        \includegraphics[width=\textwidth]{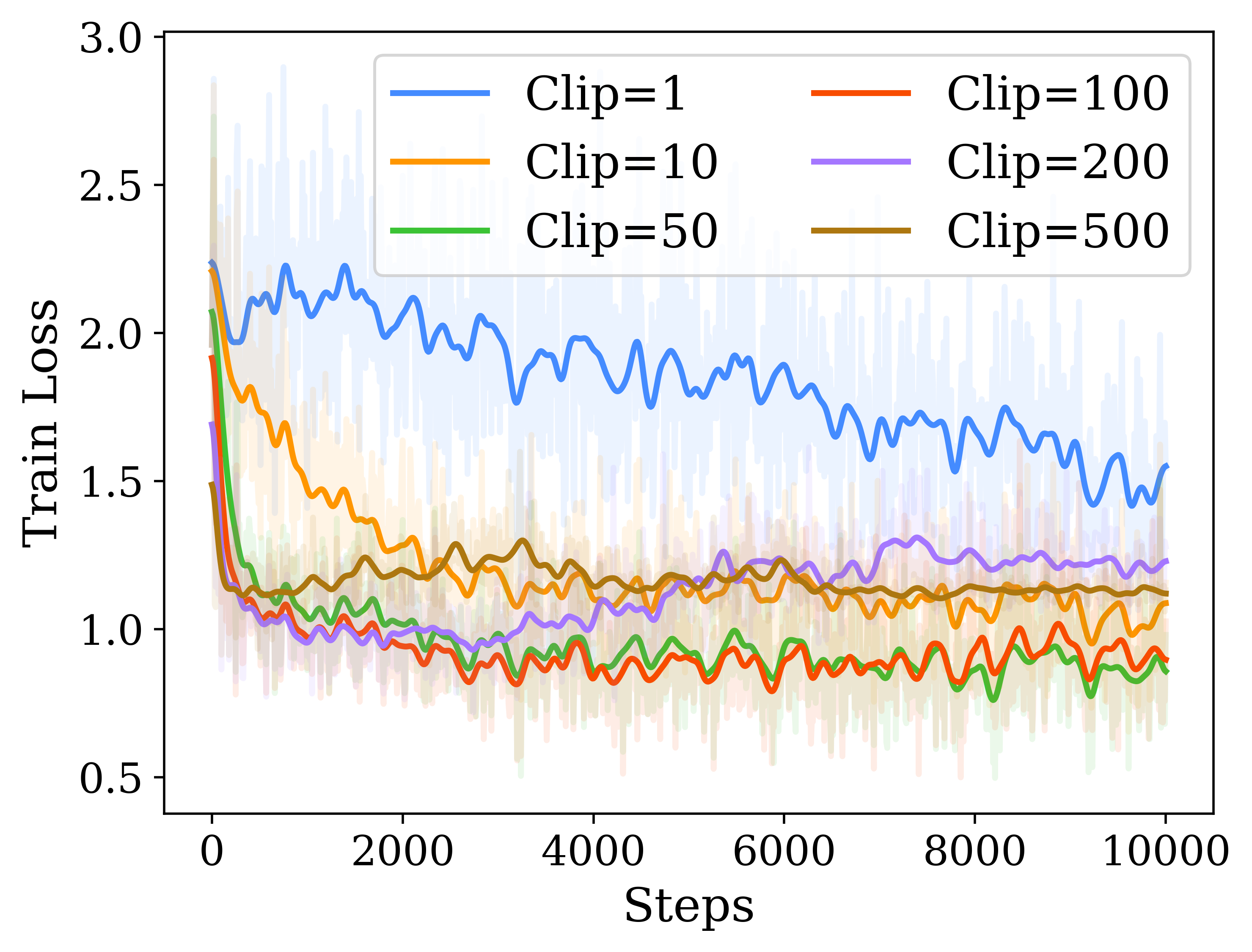}
        \caption{Training Loss.}
    \end{subfigure}
    \hfill
    \begin{subfigure}{0.32\textwidth}
        \includegraphics[width=\textwidth]{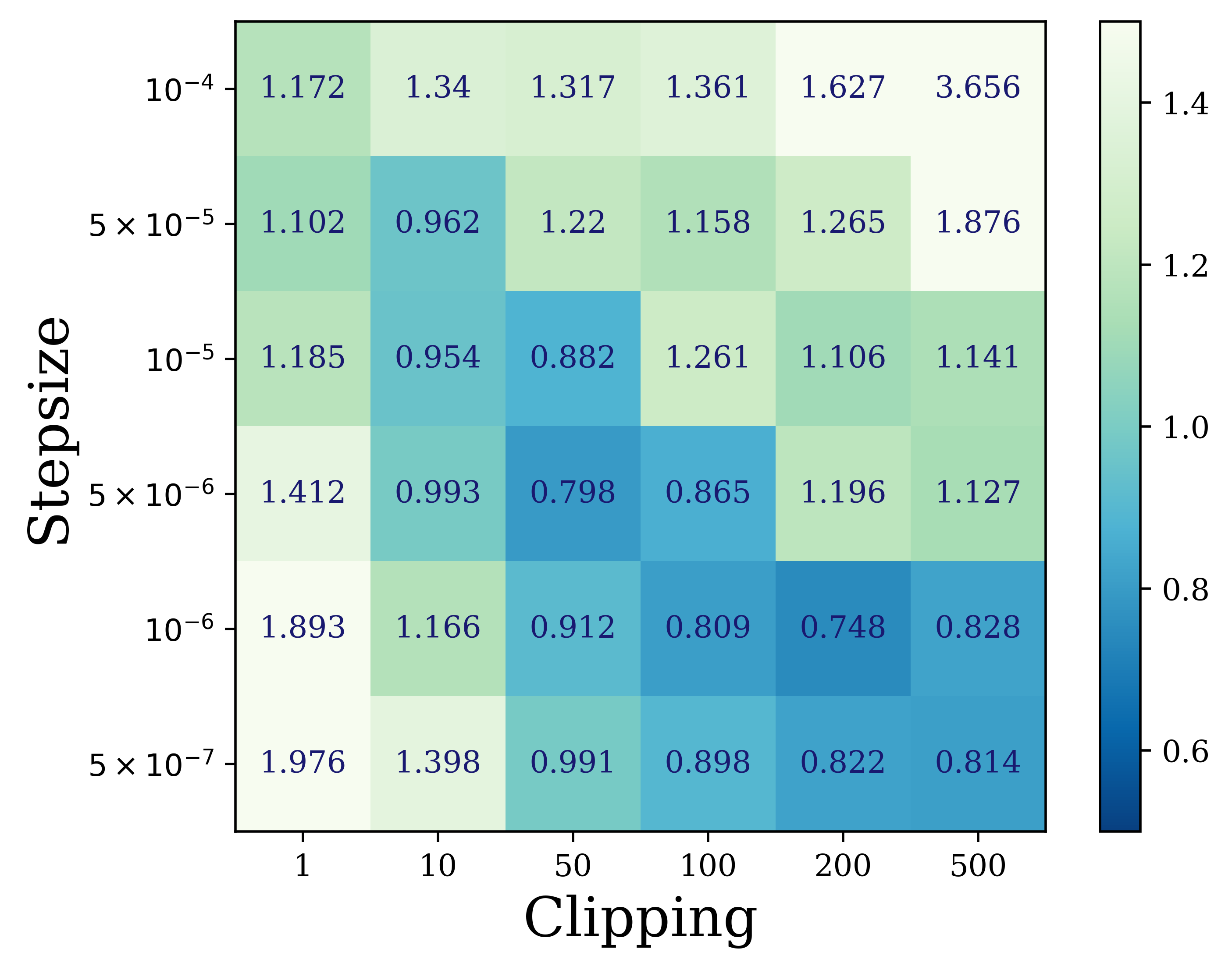}
        \caption{Test Loss.}
    \end{subfigure}
    \hfill
    \begin{subfigure}{0.32\textwidth}
        \includegraphics[width=\textwidth]{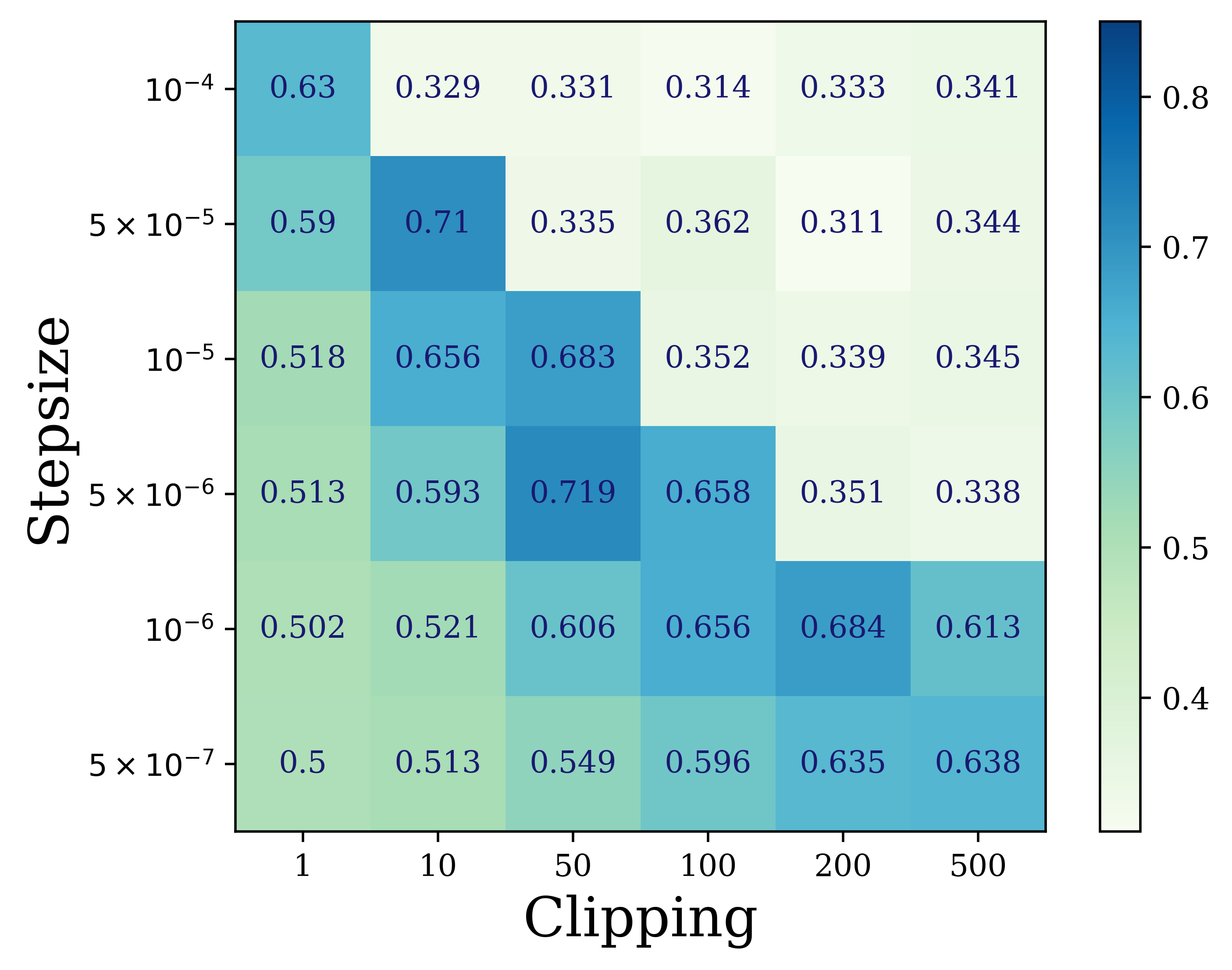}
        \caption{Test Accuracy (\%).}
    \end{subfigure}
    
    \caption{Experiments on private fine-tuning RoBERTa (125M) for SNLI with \DPZerons. (a) (Smoothed) training curves when fixing the stepsize to be $5\times 10^{-6}$ and varying the clipping threshold from 1 to 500. In the choice of clipping, a tradeoff emerges; larger clipping values result in unnecessarily high privacy noise, while smaller values can induce increased bias in the optimization process. (b) and (c) Test loss and accuracy (\%) when varying the stepsize and clipping threshold together. Consistent with first-order methods \citep{li2022large}, we observe that larger clipping necessitates smaller stepsizes, whereas smaller clipping favors larger stepsizes.}
    \label{fig:clip_lr}
\end{figure}

\paragraph{Clipping threshold.} Our findings indicate that the optimal clipping threshold for \DPZero tends to be higher than that for first-order methods. This observation aligns with the theoretical outcomes presented in Theorem \ref{thm:d-free}, where the clipping threshold for \DPZero is $C=\cO(L\sqrt{\log(nd)})$, in contrast to the $\cO(L)$ threshold adequate for first-order methods.
In the concurrent study by \citep{tang2024private}, the chosen clipping threshold is 0.05. However, their implementation applies the clipping to the term $f(x+\lambda u;\xi) - f(x-\lambda u;\xi)$. After normalization by $\lambda=10^{-3}$, it aligns with the order of magnitude used in our method.
The validity of opting for a larger clipping threshold in \DPZero is further confirmed through the private fine-tuning of RoBERTa (125M) on the SNLI dataset in Figure \ref{fig:clip_lr}.
An additional observation from our experiments is that the non-private baseline MeZO also appears to benefit from clipping. For instance, without clipping, the original MeZO encounters non-convergence issues at a stepsize of $5\times10^{-6}$. Conversely, incorporating clipping permits the use of larger stepsizes and yields better results. A thorough investigation of this phenomenon is reserved for future research.

\subsection{Private Fine-Tuning of the Language Model OPT}

\begin{table}[h!]
    \centering
    \caption{Hyper-parameters used for fine-tuning OPT. We randomly sample 1000 samples for training and 1000 samples for testing. Stepsize and clipping are optimized through a grid search over the listed values. Other parameters are fixed.}
    \begin{tabular}{cc}
        \toprule
        Hyper-parameters & Values \\
        \midrule
        Number of training samples & 1000 \\
        Number of test samples & 1000 \\
        Number of iterations & 20000 \\
        Batch size & 8 \\
        Privacy & $(\eps=\{2,6\}, \delta=10^{-5})$ \\
        Smoothing $\lambda$  & $10^{-3}$ \\
        Stepsize & $\{10^{-6}, 10^{-7}\}$ \\
        Clipping & \{10, 50, 100, 200\} \\
        \bottomrule
    \end{tabular}
    \label{tab:opt-param}
\end{table}

\begin{table}[h!]
    \centering
    \caption{Memory consumption (MiB) when fine-tuning OPT for BoolQ with batch size 8. All experiments are tested on a single GPU with 24 GiB memory. `-' in the table denotes out of memory. MeZO and \DPZero can fit models up to OPT-6.7B, while the first-order method AdamW already runs out of memory on OPT-1.3B.}
    \begin{tabular}{ccccc}
        \toprule
        Method & OPT-1.3B & OPT-2.7B & OPT-6.7B & OPT-13B \\
        \midrule
        AdamW & - & - & - & - \\
        \midrule
        MeZO  & 7866 & 11602 & 20548 & - \\
        \DPZero & 7866 & 11602 & 20548 & - \\
        \bottomrule
    \end{tabular}
    \label{tab:opt-memory}
\end{table}

We follow experiment settings in \citet{malladi2023fine} to evaluate the performance of \DPZero in the private fine-tuning of OPT \citep{zhang2022opt} across four different datasets: SST-2 \citep{socher2013recursive} for sentiment classification and BoolQ \citep{clark2019boolq}, SQuAD \citep{rajpurkar2016squad}, and DROP \citep{dua2019drop} for question answering. In our experiments, we employ the same prompts as used in \citet{malladi2023fine} and use the same implementation as explained before. All results are averaged over three random seeds $\{0, 29, 83\}$. The hyper-parameters used for our experiments are summarized in Table \ref{tab:opt-param}, and the memory usages on the dataset BoolQ are reported in Table \ref{tab:opt-memory}.

\section{Technical Lemmas}
\label{app:technical} 

\begin{lemma}
    Let $u$ be uniformly sampled from the Euclidean sphere $\sqrt{d}\,\bS^{d-1}$, $a\in\bR^d$ be some fixed vector independent of $u$, and $H\in\bR^{d\times d}$ be some fixed matrix independent of $u$. We have that
    \begin{itemize}
        \item[$(i)$] $\bE[u]=0\,$ and $\,\bE[uu^\top]=\rI_d$.

        \item[$(ii)$] $\bE_u[u^\top a]=0,\, \bE_u[(u^\top a)^2]=\norm{a}^2\,$ and $\forall\, C\geq 0$,
        \begin{equation*}
            \bP(\abs{u^\top a} \geq C) \leq 2\sqrt{2\pi} \exp\left(-\frac{C^2}{8\norm{a}^2}\right).
        \end{equation*}
        
        \item[$(iii)$] $\bE_u[(u^\top a) u] = a\,$ and
        \begin{align*}
            & \bE_u[(u^\top a)^2 \norm{u}^2] = d\norm{a}^2, \\
            & \bE_u[(u^\top a)^2 uu^\top] = \frac{d}{d+2}\left(2aa^\top + \norm{a}^2\rI_d\right).
        \end{align*}

        \item[$(iv)$] $\bE_u[u^\top H u]=\emph\tr(H)\,$ and
        \begin{equation*}
            \bE_u[(u^\top a)^2 u^\top H u] = \frac{d}{d+2}\left(2a^\top Ha + \norm{a}^2\emph\tr(H)\right).
        \end{equation*}
    \end{itemize}
    \label{lm:sphere}
\end{lemma}

\begin{proof}
    $(i)$ is a standard result, e.g., in \citet{duchi2015optimal}, and follows by the symmetry of the sphere. For any $u\in\sqrt{d}\cdot\bS^{d-1}$, it must be the case that $-u\in\sqrt{d}\cdot\bS^{d-1}$ as well, which suggests that $\bE[u]=0$. Since $\bE[\sum_{i=1}^d u_i^2]=\bE\norm{u}^2=d$, we immediately have that $\bE[u_i^2]=1$ for every $i$ by symmetry. Then for the off-diagonal terms, since for any $u=(u_1,\cdots,u_i,\cdots,u_j,\cdots,u_d)\in\sqrt{d}\cdot\bS^{d-1}$, it must be the case that $(u_1,\cdots,u_i,\cdots, -u_j, \cdots,u_d)\in\sqrt{d}\cdot\bS^{d-1}$ as well, which suggests that $\bE[u_iu_j]=0$ when $i\neq j$. As a result, we can conclude that the matrix $\bE[uu^\top] = \rI_d$.

    We then show $(ii)$. Applying $(i)$, we have that $\bE_u[u^\top a]=0$, and that
    \begin{align*}
        \bE_u[(u^\top a)^2]
        & =
        \sum_{i=1}^d a_i^2 \bE[u_i^2] + \sum_{i\neq j} a_ia_j \bE[u_iu_j] \\
        & =
        \norm{a}^2.
    \end{align*}
    The tail bound follows from Example 3.12 in \citet{wainwright2019high}, where they showed that for any function $h: \bS^{d-1} \rightarrow \bR$ such that $\forall x,y\in\bS^{d-1}$,
    \begin{equation*}
        \abs{h(x) - h(y)} \leq \arccos(x^\top y),
    \end{equation*}
    when $x$ is uniformly sampled from $\bS^{d-1}$, it holds that $\forall\,\gamma\geq 0$,
    \begin{equation}
        \bP(\abs{h(x) - \bE[h(x)]} \geq \gamma) \leq 2\sqrt{2\pi} \exp\left(-\frac{d\gamma^2}{8}\right).
        \label{eq:tail-unit}
    \end{equation}
    Let $h(x)=x^\top a/\norm{a}$ for $x\in\bS^{d-1}$. First, we have that $\forall x,y\in\bS^{d-1}$,
    \begin{align*}
        \abs{h(x) - h(y)}^2
        & =
        \frac{\abs{(x-y)^\top a}^2}{\norm{a}^2} \\
        & \leq
        \norm{x-y}^2 \\
        & =
        2 (1 - x^\top y) \\
        & \leq
        (\arccos(x^\top y))^2,
    \end{align*}
    where we use the inequality that $\theta^2/2 + \cos(\theta) - 1\geq0$ for $\theta\in[0,\pi]$ and let $x^\top y=\cos(\theta)$ such that $\arccos(x^\top y)=\theta$ for some $\theta\in[0,\pi]$. When $u$ is uniformly sampled from $\sqrt{d}\cdot\bS^{d-1}$, we know $u/\sqrt{d}$ is uniformly from $\bS^{d-1}$. Applying \eqref{eq:tail-unit} for $h(x)=x^\top a/\norm{a}$ where $x\in\bS^{d-1}$, we obtain that
    \begin{equation*}
        \bP\left(\abs*{\frac{u^\top a}{\sqrt{d}\norm{a}} - \frac{\bE[u^\top a]}{\sqrt{d}\norm{a}}} \geq \gamma\right) \leq 2\sqrt{2\pi} \exp\left(-\frac{d\gamma^2}{8}\right).
    \end{equation*}
    Setting $C=\gamma\sqrt{d}\norm{a}$, the proof is complete since $\bE[u^\top a]=0$. Similar results also exist in Theorem 5.1.4 of \citet{vershynin2018high}, with all constants hidden behind some absolute $c$.

    Next, we prove $(iii)$. Applying $(i)$, we have that
    \begin{align*}
        \bE_u[(u^\top a)u_i]
        & =
        a_i\bE[u_i^2] + \sum_{j\neq i} a_j \bE[u_iu_j] \\
        & =
        a_i.
    \end{align*}
    This implies that $\bE_u[(u^\top a)u]=a$. Applying $(ii)$, we obtain that
    \begin{align*}
        \bE_u[(u^\top a)^2 \norm{u}^2]
        & =
        d\cdot\bE_u[(u^\top a)^2] \\
        & =
        d\norm{a}^2.
    \end{align*}
    For the expectation of the matrix, we start from the diagonal terms.
    \begin{equation}
        \begin{split}
            \bE_u[(u^\top a)^2 u_i^2]
            & =
            \sum_{j=1}^d a_j^2 \bE[u_j^2 u_i^2] + \sum_{j\neq k} a_ja_k \bE[u_ju_k u_i^2] \\
            & =
            a_i^2 \bE[u_i^4] + \sum_{j\neq i} a_j^2 \bE[u_j^2u_i^2].
        \end{split}
        \label{eq:diagonal}
    \end{equation}
    Here, we use the property that $\bE[u_ju_ku_i^2]=0$ for every $i$ when $j\neq k$. This follows from symmetry of the sphere such that for any $u=(u_1,\cdots,u_j,\cdots,u_k,\cdots,u_d)\in\sqrt{d}\cdot\bS^{d-1}$, it must be the case that $(u_1,\cdots,u_j,\cdots, -u_k, \cdots,u_d)\in\sqrt{d}\cdot\bS^{d-1}$ as well. Again by symmetry, we have $\bE[u_i^4]$ remains the same for every $i$, and $\bE[u_i^2u_j^2]$ is the same for every $i\neq j$. Denote $w_1=\bE[u_i^4]$ and $w_2=\bE[u_i^2u_j^2]$. Since it holds that
    \begin{align*}
        \sum_{i=1}^d \bE_u[(u^\top a)^2 u_i^2]
        & =
        \bE_u[(u^\top a)^2 \norm{u}^2] \\
        & =
        d\norm{a}^2,
    \end{align*}
    taking summation over \eqref{eq:diagonal}, we can have that
    \begin{align*}
        d\norm{a}^2
        & =
        \sum_{i=1}^d a_i^2 \bE[u_i^4] + \sum_{i=1}^d \sum_{j=1, j\neq i}^d a_j^2 \bE[u_j^2u_i^2] \\
        & =
        w_1\norm{a}^2 + w_2 \sum_{i=1}^d (\norm{a}^2 - a_i^2) \\
        & =
        w_1\norm{a}^2 + (d-1)w_2\norm{a}^2.
    \end{align*}
    This holds for arbitrary $a\in\bR^d$, and thus we obtain that
    \begin{equation}
        w_1 + (d-1) w_2 = d.
        \label{eq:exp-i4-i2j2}
    \end{equation}
    We only compute $w_1=\bE[u_i^4]$ by showing that $u_i^2/d$ actually follows the Beta distribution, and the value of $w_2$ can be derived from \eqref{eq:exp-i4-i2j2}. First, $z/\norm{z}$ is uniformly distributed on the unit sphere $\bS^{d-1}$ for $z\in\bR^d$ sampled from the standard multivariate Gaussian $\cN(0,\rI_d)$ \citep{muller1959note, marsaglia1972choosing}. This means that $z_i^2$ is distributed according to the $\chi^2$-distribution with 1 degree of freedom, and $\bar z_i^2:=\sum_{j\neq i} z_j^2$ is distributed according to the $\chi^2$-distribution with degree $(d-1)$. Since $\chi^2$-distribution is a special case of the Gamma distribution and $z_i^2$, $\bar z_i^2$ are independent, we conclude that $z_i^2/(z_i^2 + \bar z_i^2)$ has the Beta distribution with parameters $1/2$ and $(d-1)/2$ \citep{cramer1999mathematical, gupta2004handbook}. Finally, since $u/\sqrt{d}$ is uniformly distributed on $\bS^{d-1}$, by symmetry of the sphere, we know that $u_i^2/d$ has the same Beta distribution as $z_i^2/(z_i^2 + \bar z_i^2)$. The mean and variance of Beta$(1/2, (d-1)/2)$ is $1/d$ and $2(d-1)/(d^2(d+2))$. This suggests that $\bE[u_i^2]=1$, as already proved in $(i)$, and that
    \begin{align*}
        w_1
        & =
        \bE[(u_i^2 - \bE[u_i^2])^2] + (\bE[u_i^2])^2 \\
        & =
        d^2\left(\frac{2(d-1)}{d^2(d+2)} + \frac{1}{d^2}\right) \\
        & =
        \frac{3d}{d+2}.
    \end{align*}
    By \eqref{eq:exp-i4-i2j2}, we know $w_2=d/(d+2)$. According to \eqref{eq:diagonal}, we have that the diagonal terms
    \begin{align*}
        \bE_u[(u^\top a)^2 u_i^2]
        & =
        w_1a_i^2 + w_2(\norm{a}^2 - a_i^2) \\
        & =
        \frac{2d}{d+2}a_i^2 + \frac{d}{d+2}\norm{a}^2.
    \end{align*}
    Then we compute the off-diagonal entries for $i\neq j$. By the same reasoning as \eqref{eq:diagonal}, we have that
    \begin{align*}
        \bE_u[(u^\top a)^2 u_iu_j]
        & =
        \sum_{i\neq j} a_ia_j \bE[u_i^2 u_j^2] \\
        & =
        \frac{2d}{d+2} a_ia_j.
    \end{align*}
    All other terms equal to 0 by symmetry of the sphere. Combining both diagonal and off-diagonal elements, we have that $\bE_u[(u^\top a)^2uu^\top]=(d/(d+2))(2aa^\top + \norm{a}^2\rI_d)$. Similar results are also shown in Appendix F of \citet{malladi2023fine}.
    
    Finally, we give the proof of $(iv)$. For the first statement, applying $(i)$ in this lemma, we have that
    \begin{align*}
        \bE_u\left[u^\top Hu\right]
        & =
        \bE\Big[\tr(uu^\top H)\Big] \\
        & =
        \tr\Big(\bE[uu^\top]\cdot H\Big) \\
        & =
        \tr(H).
    \end{align*}
    Similarly for the second statement, we apply $(iii)$ in this lemma and obtain that
    \begin{align*}
        \bE_u\left[(u^\top a)^2 u^\top H u\right]
        & =
        \bE\Big[(u^\top a)^2 \cdot\tr(uu^\top H)\Big] \\
        & =
        \bE\Big[\tr\Big((u^\top a)^2uu^\top\cdot H\Big)\Big] \\
        & =
        \tr\Big(\bE\Big[(u^\top a)^2uu^\top\Big]\cdot H\Big) \\
        & =
        \frac{2d}{d+2}\tr(aa^\top H) + \frac{d}{d+2}\norm{a}^2\tr(H) \\
        & =
        \frac{2d}{d+2}a^\top Ha + \frac{d}{d+2}\norm{a}^2\tr(H).
    \end{align*}
    This concludes the proof.
\end{proof}

\begin{lemma}
    Let $u$ be uniformly sampled from the Euclidean sphere $\sqrt{d}\,\bS^{d-1}$ and $v$ be uniformly sampled from the Euclidean ball $\sqrt{d}\,\bB^d=\{x\in\bR^d \,|\, \norm{x}\leq \sqrt{d}\}$. For any function $f(x):\bR^d\rightarrow\bR$ and $\lambda>0$, we define its zeroth-order gradient estimator as $g_\lambda(x) = ((f(x+\lambda u) - f(x-\lambda u))/(2\lambda) ) u$ and the smoothed function as $f_\lambda(x)=\bE_v[f(x+\lambda v)]$. The following properties hold:
    \begin{itemize}
        \item[$(i)$] $f_\lambda(x)$ is differentiable and $\bE_u[g_\lambda(x)]=\nabla f_\lambda(x)$.

        \item[$(ii)$] If $f(x)$ is $\ell$-smooth, then we have that
        \begin{eqnarray*}
             \norm{\nabla f(x) - \nabla f_\lambda(x)} 
             &\leq& \frac{\ell}{2}\lambda d^{3/2}, \\
             \bE_u[\,\norm{g_\lambda(x)}^2\, ] &\leq& 2d\cdot\norm{\nabla f(x)}^2 + \frac{\ell^2}{2}\lambda^2d^3.
        \end{eqnarray*}
    \end{itemize}
    The above results are consistent with $(iii)$ in Lemma \ref{lm:sphere} when $\lambda\to 0$ and $f(x)$ is differentiable such that the two-point estimator reduces to the directional derivative $g_0(x)=u^\top \nabla f(x) u$.
    \label{lm:zero-grad}
\end{lemma}

\begin{proof}
    We first show $(i)$. Similarly to Lemma 10 in \citet{shamir2017optimal}, we have that
    \begin{equation*}
        \bE_{u\in\sqrt{d}\cdot\bS^{d-1}}[g_\lambda(x)] = \bE_{u\in\sqrt{d}\cdot\bS^{d-1}}\left[\frac{f(x+\lambda u)u}{\lambda}\right].
    \end{equation*}
    Applying Lemma 2.1 in \citet{flaxman2005online}, we know
    \begin{equation*}
        \bE_{u'\in\bS^{d-1}}[f(x+\lambda' u')u'] = \frac{\lambda'}{d}\nabla \bE_{v'\in\bB^d}[f(x+\lambda' v')].
    \end{equation*}
    Introducing $u=\sqrt{d}u'$, $v=\sqrt{d}v'$ and $\lambda=\lambda'/\sqrt{d}$, we thus obtain
    \begin{align*}
        \bE_{u\in\sqrt{d}\cdot\bS^{d-1}}\left[\frac{f(x+\lambda u)u}{\lambda}\right]
        & =
        \bE_{u'\in\bS^{d-1}}\left[\frac{f(x+\lambda' u')u'd}{\lambda'}\right] \\
        & =
        \nabla \bE_{v'\in\bB^d}[f(x+\lambda' v')] \\
        & =
        \nabla \bE_{v\in\sqrt{d}\cdot\bB^d}[f(x+\lambda v)].
    \end{align*}

    The proof of $(ii)$ mostly follows from \citet{nesterov2017random}, where the results are originally obtained for the case that $u$ is sampled from the standard multivariate Gaussian distribution. By $(iii)$ in Lemma \ref{lm:sphere} and $(i)$ here, we have that for $u$ uniformly sampled from $\sqrt{d}\cdot\bS^{d-1}$,
    \begin{align*}
        \norm{\nabla f(x) - \nabla f_\lambda(x)}
        & =
        \norm*{\bE_u[(u^\top \nabla f(x))u] - \bE_u\left[\frac{f(x+\lambda u) - f(x-\lambda u)}{2\lambda}u\right]} \\
        & \leq
        \bE_u\norm*{\left(\frac{f(x+\lambda u) - f(x-\lambda u)}{2\lambda} - u^\top\nabla f(x)\right)u} \\
        & \leq
        \frac{\sqrt{d}}{2\lambda}\bE_u\abs{f(x+\lambda u) - f(x) - \lambda u^\top \nabla f(x)} \\
        & \quad +
        \frac{\sqrt{d}}{2\lambda}\bE_u\abs{f(x) - f(x-\lambda u) - \lambda u^\top \nabla f(x)} \\
        & \leq
        \frac{\ell}{2}\lambda d^{3/2},
    \end{align*}
    where in the last step we use smoothness of $f(x)$ such that $\abs{f(x+\lambda u) - f(x) - \lambda u^\top\nabla f(x)} \leq \ell\lambda^2d/2$ and the same holds for $\abs{f(x) - f(x-\lambda u) - \lambda u^\top \nabla f(x)}=\abs{f(x-\lambda u) - f(x) + \lambda u^\top\nabla f(x)}$. The last statement holds similarly:
    \begin{align}
        \bE_u[ \norm{g_\lambda(x)}^2]
        & =
        \frac{d}{4\lambda^2}\bE_u[(f(x+\lambda u) - f(x-\lambda u))^2] \nonumber \\
        & \leq
        2d\cdot\bE_u[(u^\top\nabla f(x))^2] + \frac{d}{2\lambda^2}\bE_u[(f(x+\lambda u) - f(x-\lambda u) - 2\lambda u^\top\nabla f(x))^2] \nonumber \\
        & \leq
        2d\cdot\bE_u[(u^\top\nabla f(x))^2] + \frac{d}{\lambda^2}\bE_u[(f(x+\lambda u) - f(x) - \lambda u^\top\nabla f(x))^2] \nonumber \\
        & \quad + \frac{d}{\lambda^2}\bE_u[(f(x) - f(x-\lambda u) - \lambda u^\top\nabla f(x))^2] \nonumber \\
        & \leq
        2d\cdot\norm{\nabla f(x)}^2 + \frac{\ell^2}{2}\lambda^2d^3,
        \label{eq:bound-finite-diff}
    \end{align}
    where in the last step we use Lemma \ref{lm:sphere} and smoothness of $f(x)$.
\end{proof}

\section{Detailed Proof and Analysis of DPGD-0th (Algorithm \ref{algo:d-dependent})}
\label{app:d-dependent}

\begin{proof}[Proof of Theorem \ref{thm:d-dependent}]
    The privacy guarantees directly follow from Lemma \ref{lm:composition} noticing that the sensitivity is $2C/n$. Note that the original advanced composition theorem in \citet{kairouz2015composition} is stated for the case where the output of $\cA$ is a scalar. Given the spherical symmetry properties of Gaussian noise, the results can be readily extended to multiple dimensions, as outlined in Lemma 1 of \citet{kenthapadi2013privacy} where the basis can be selected in a way such that $\cA(S)$ and $\cA(S')$ differ in exactly one dimension.
    
    We then focus on the utility guarantee on $\bE[ \norm{\nabla F_S(x_\tau)}^2 ] $. Since $f(x;\xi)$ is $L$-Lipschitz for every $\xi$ by Assumption \ref{asp:smooth} and $\norm{u_t}=\sqrt{d}$ by its construction, we have that
    \begin{align*}
        \norm{g_\lambda(x_t;\xi_i)}
        & =
        \frac{\abs{f(x_t+\lambda u_t;\xi_i) - f(x_t-\lambda u_t;\xi_i)}}{2\lambda}\norm{u_t} \\
        & \leq
        L\norm{u_t}^2 \\
        & =
        Ld.
    \end{align*}
    This means $\clip_C(g_\lambda(x_t;\xi_i))=g_\lambda(x_t;\xi_i)$ when setting $C= Ld$. For notation simplicity, we let
    \begin{align*}
        G_\lambda(x_t)
        & :=
        \frac{1}{n}\sum_{i=1}^n g_\lambda(x_t;\xi_i) \\
        & =
        \frac{1}{n}\sum_{i=1}^n \frac{f(x_t+\lambda u_t;\xi_i) - f(x_t-\lambda u_t;\xi_i)}{2\lambda}u_t \\
        & =
        \frac{F_S(x_t + \lambda u_t) - F_S(x_t - \lambda u_t)}{2\lambda}u_t.
    \end{align*}
    Algorithm \ref{algo:d-dependent} reduces to $x_{t+1} = x_t - \alpha(G_\lambda(x_t) + z_t)$. By smoothness of $F_S(x)$, we have that
    \begin{align*}
        F_S(x_{t+1})
        & \leq
        F_S(x_t) + \nabla F_S(x_t)^\top (x_{t+1} - x_t) + \frac{\ell}{2}\norm{x_{t+1} - x_t}^2 \\
        & =
        F_S(x_t) - \alpha \nabla F_S(x_t)^\top (G_\lambda(x_t) + z_t) + \frac{\ell}{2}\alpha^2\norm{G_\lambda(x_t)}^2 + \frac{\ell}{2}\alpha^2\norm{z_t}^2 + \ell\alpha^2 z_t^\top G_\lambda(x_t).
    \end{align*}
    Since $z_t$ is sampled from $\cN(0,\sigma^2\rI_d)$ and is independent of $x_t$, $u_t$ and $S$, we have that
    \begin{equation*}
        \bE_{z_t}[F_S(x_{t+1})] \leq F_S(x_t) - \alpha \nabla F_S(x_t)^\top G_\lambda(x_t) + \frac{\ell}{2}\alpha^2\norm{G_\lambda(x_t)}^2 + \frac{\ell}{2}\alpha^2\, d\, \sigma^2.
    \end{equation*}
    Define $F_\lambda(x):=\bE_v[F_S(x+\lambda v)]$ for $v$ sampled uniformly from the Euclidean ball $\sqrt{d}\cdot\bB^d$. By Lemma \ref{lm:zero-grad}, we know $\bE_{u_t}[G_\lambda(x_t)]=\nabla F_\lambda(x_t)$. Since $u_t$ is independent of $x_t$ and $S$, taking expectation with respect to $u_t$ and applying $(ii)$ in Lemma \ref{lm:zero-grad}, we obtain that
    \begin{align}
        \bE_{z_t, u_t}[F_S(x_{t+1})]
        & \leq
        F_S(x_t) - \alpha\nabla F_S(x_t)^\top \nabla F_\lambda(x_t) + \frac{\ell}{2}\alpha^2\bE_{u_t}[\norm{G_\lambda(x_t)}^2] + \frac{\ell}{2}\alpha^2\, d\,\sigma^2 \nonumber \\
        & =
        F_S(x_t) - \frac{\alpha}{2}\norm{\nabla F_S(x_t)}^2 - \frac{\alpha}{2}\norm{\nabla F_\lambda(x_t)}^2 + \frac{\alpha}{2}\norm{\nabla F_\lambda(x_t) - \nabla F_S(x_t)}^2 \nonumber \\
        & \quad +
        \frac{\ell}{2}\alpha^2\,\bE_{u_t}[\norm{G_\lambda(x_t)}^2] + \frac{\ell}{2}\alpha^2\,d\,\sigma^2 \nonumber \\
        & \leq
        F_S(x_t) - \frac{\alpha}{2}(1 - 2d\ell\alpha)\norm{\nabla F_S(x_t)}^2 + \frac{\ell^2}{8}\alpha(1+2\ell\alpha)\lambda^2d^3 + \frac{\ell}{2}\alpha^2\, d\,\sigma^2.
        \label{eq:onestep1}
    \end{align}
    Choosing $\alpha= 1/(4\ell d)$ such that $1-2d\ell\alpha=1/2$ and $2\ell\alpha<1$, we obtain that
    \begin{align*}
        \bE[\norm{\nabla F_S(x_t)}^2]
        & <
        \frac{4\,\bE[F_S(x_t) - F_S(x_{t+1})]}{\alpha} + \ell^2\lambda^2d^3 + 2\ell\alpha\, d\sigma^2 \nonumber\\
        & =
        \frac{4\,\bE[F_S(x_t) - F_S(x_{t+1})]}{\alpha} + \ell^2\lambda^2d^3 + \frac{64\ell C^2\, \alpha T\, d\log(e+ (\eps/\delta))}{n^2\eps^2} \nonumber \\
        & =
        \frac{4\,\bE[F_S(x_t) - F_S(x_{t+1})]}{\alpha} + \ell^2\lambda^2d^3 + \frac{64\ell L^2\, \alpha T\, d^3\log(e+ (\eps/\delta))}{n^2\eps^2}.
    \end{align*}
    As a result, taking summation from $t=0$ to $T-1$ and dividing both sides by $T$, we have that
    \begin{align*}
        \bE[\norm{\nabla F_S(x_\tau)}^2]
        & =
        \frac{1}{T}\sum_{t=0}^{T-1} \bE[\norm{\nabla F_S(x_t)}^2] \\
        & \leq
        \frac{4(F_S(x_0) - F_S^*)}{\alpha T} + \ell^2\lambda^2d^3 + \frac{64\ell L^2\, \alpha T\, d^3\log(e+ (\eps/\delta))}{n^2\eps^2} \\
        & \leq
        \frac{16(\ell (F_S(x_0) - F_S^*) + 2L^2)d\sqrt{d\log(e+(\eps/\delta))}}{n\eps},
    \end{align*}
    with the choice of parameters
    \begin{equation*}
        \alpha T = \frac{n\eps}{4\ell d\sqrt{d\log(e+(\eps/\delta))}}, \quad
        \lambda \leq \frac{4L}{\ell d}\left(\frac{\sqrt{d\log(e+(\eps/\delta))}}{n\eps}\right)^{1/2}.
    \end{equation*}
    This suggests that the total number of iteration is $T=n\eps/\sqrt{d\log(e+(\eps/\delta))}$ and the total number of zeroth-order gradient computations is $nT=n^2\eps/\sqrt{d\log(e+(\eps/\delta))}$. Note that the above selection of parameters ensures scale invariance.
\end{proof}

\begin{proof}[Proof of Theorem \ref{thm:d-dependent-rank}]
    The privacy analysis remains the same as before, and we focus on the utility analysis on $\bE\norm{\nabla F_S(x_\tau)}^2$. By the same reasoning, when setting $C= L d$, Algorithm \ref{algo:d-dependent} reduces to $x_{t+1} = x_t - \alpha (G_\lambda(x_t) + z_t)$ where $G_\lambda(x_t)=(F_S(x_t+\lambda u_t)-F_S(x_t-\lambda u_t))u_t/(2\lambda)$. By Taylor's theorem with remainder, for some $\theta\in(0,1)$, we have that
    \begin{align*}
        F_S(x_{t+1})
        & =
        F_S(x_t) + \nabla F_S(x_t)^\top (x_{t+1} - x_t) + \frac{1}{2} (x_{t+1} - x_t)^\top \nabla^2 F_S(x_t+\theta(x_{t+1} - x_t)) (x_{t+1} - x_t) \\
        & \leq
        F_S(x_t) - \alpha\nabla F_S(x_t)^\top \left(G_\lambda(x_t) + z_t\right) + \frac{\alpha^2}{2} G_\lambda(x_t)^\top H G_\lambda(x_t) + \frac{\alpha^2}{2} z_t^\top H z_t \\
        & \qquad +
        \frac{\alpha^2}{2}\left(G_\lambda(x_t)^\top H z_t + z_t^\top H G_\lambda(x_t)\right).
    \end{align*}
    Here in the inequality, we use Assumption \ref{asp:rank} such that $\nabla^2 F_S(x)\preceq H$ for any $x\in\bR^d$. Similarly to $(iv)$ in Lemma \ref{lm:sphere}, we have that $\bE[z_t^\top H z_t] = \tr(\bE[z_tz_t^\top] H) = \sigma^2\tr(H)$.  Since $z_t$ is sampled from $\cN(0, \sigma^2\rI_d)$ and is independent of $u_t$, $x_t$ and the dataset $S$, taking expectation with respect to $z_t$, we can then obtain that
    \begin{equation}
        \begin{split}
            \bE_{z_t}[F_S(x_{t+1})]
            & \leq
            F_S(x_t) - \alpha\nabla F_S(x_t)^\top G_\lambda(x_t) + \frac{\alpha^2}{2} G_\lambda(x_t)^\top H G_\lambda(x_t) + \frac{\alpha^2}{2} \bE_{z_t}[z_t^\top H z_t] \\
            & =
            F_S(x_t) - \alpha\nabla F_S(x_t)^\top G_\lambda(x_t) + \frac{\alpha^2}{2} G_\lambda(x_t)^\top H G_\lambda(x_t) + \frac{\alpha^2\sigma^2}{2} \tr(H).
        \end{split}
        \label{eq:smooth-d-dependent-rank}
    \end{equation}
    Assumption \ref{asp:rank} implies $F_S(x)$ is also $\ell$-smooth. By a similar argument as \eqref{eq:bound-finite-diff} in the proof of $(ii)$ in Lemma \ref{lm:zero-grad}, we have
    \begin{equation}
        \left(\frac{F_S(x_t+\lambda u_t) - F_S(x_t-\lambda u_t)}{2\lambda}\right)^2 \leq 2\left(u_t^\top \nabla F_S(x_t)\right)^2 + \frac{\ell^2}{2}\lambda^2d^2.
        \label{eq:upper-bound-finite-diff}
    \end{equation}
    As $u_t^\top Hu_t \geq 0$, by $(iv)$ in Lemma \ref{lm:sphere} and Assumption \ref{asp:rank}, we have that
    \begin{align*}
        \bE\left[G_\lambda(x_t)^\top H G_\lambda(x_t)\right]
        & =
        \bE\left[\left(\frac{F_S(x_t+\lambda u_t) - F_S(x_t-\lambda u_t)}{2\lambda}\right)^2 u_t^\top H u_t\right] \\
        & \leq
        2\,\bE\left[\left(u_t^\top \nabla F_S(x_t)\right)^2 u_t^\top H u_t\right] + \frac{\ell^2}{2}\lambda^2d^2 \, \bE\left[u_t^\top H u_t\right] \\
        & =
        \frac{2d}{d+2}\left(2\nabla F_S(x_t)^\top H\nabla F_S(x_t) + \norm{\nabla F_S(x_t)}^2\tr(H)\right) + \frac{\ell^2}{2}\lambda^2d^2 \tr(H) \\
        & \leq
        2\ell(r+2)\norm{\nabla F_S(x_t)}^2 + \frac{\ell^3}{2}\lambda^2d^2 r.
    \end{align*}
    Taking expectation of \eqref{eq:smooth-d-dependent-rank} with respect to $u_t$, by Lemma \ref{lm:zero-grad} for $F_\lambda(x)=\bE_{v}[F_S(x+\lambda v)]$ with $v$ uniformly sampled from $\sqrt{d}\cdot\bB^d$, we have that
    \begin{align}
        \bE[F_S(x_{t+1})]
        & \leq
        F_S(x_t) - \alpha \nabla F_S(x_t)^\top \nabla F_\lambda(x_t) + \ell\alpha^2(r+2)\norm{\nabla F_S(x_t)}^2 + \frac{\ell^3\alpha^2\lambda^2 d^2r}{4} + \frac{\ell\alpha^2\,r\sigma^2}{2} \nonumber \\
        & \leq
        F_S(x_t) - \frac{\alpha}{2}(1 - 2(r+2)\ell\alpha)\norm{\nabla F_S(x_t)}^2 +
        \frac{\alpha}{2}\norm{\nabla F_S(x_t) - \nabla F_\lambda(x_t)}^2 + \frac{\ell^3\alpha^2\lambda^2 d^2r}{4} + \frac{\ell\alpha^2\,r\sigma^2}{2} \nonumber \\
        & \leq
        F_S(x_t) - \frac{\alpha}{2}(1 - 2(r+2)\ell\alpha)\norm{\nabla F_S(x_t)}^2 + \frac{\ell^2\alpha\lambda^2d^2(d+2r\ell\alpha)}{8} + \frac{\ell\alpha^2\,r\sigma^2}{2}.
        \label{eq:onestep2}
    \end{align}
    Choosing $\alpha=1/(4\ell(r+2))$ such that $1-2(r+2)\ell\alpha=1/2$ and $2\ell\alpha r<1\leq d$, we have that
    \begin{align*}
        \bE[\norm{\nabla F_S(x_t)}^2]
        & <
        \frac{4\,\bE[F_S(x_t) - F_S(x_{t+1})]}{\alpha} + \ell^2\lambda^2d^3 + 2\ell\alpha\, r\sigma^2 \\
        & =
        \frac{4\,\bE[F_S(x_t) - F_S(x_{t+1})]}{\alpha} + \ell^2\lambda^2d^3 + \frac{64\ell C^2\, \alpha T\, r\log(e+ (\eps/\delta))}{n^2\eps^2} \\
        & =
        \frac{4\,\bE[F_S(x_t) - F_S(x_{t+1})]}{\alpha} + \ell^2\lambda^2d^3 + \frac{64\ell L^2\, \alpha T\, d^2r\log(e+ (\eps/\delta))}{n^2\eps^2}.
    \end{align*}
    As a result, taking summation from $t=0$ to $T-1$ and dividing both sides by $T$, we have that
    \begin{align*}
        \bE[\norm{\nabla F_S(x_\tau)}^2]
        & =
        \frac{1}{T}\sum_{t=0}^{T-1} \bE[\norm{\nabla F_S(x_t)}^2] \\
        & \leq
        \frac{4(F_S(x_0) - F_S^*)}{\alpha T} + \ell^2\lambda^2d^3 + \frac{64\ell L^2\, \alpha T\, d^2r\log(e+ (\eps/\delta))}{n^2\eps^2} \\
        & \leq
        \frac{16(\ell (F_S(x_0) - F_S^*) + 2L^2)d\sqrt{r\log(e+(\eps/\delta))}}{n\eps},
    \end{align*}
    with the choice of parameters
    \begin{equation*}
        \alpha T = \frac{n\eps}{4\ell d\sqrt{r\log(e+(\eps/\delta))}}, \quad
        \lambda \leq \frac{4L}{\ell d}\left(\frac{\sqrt{r\log(e+(\eps/\delta))}}{n\eps}\right)^{1/2}.
    \end{equation*}
    This suggests that the total number of iteration is $T=n(r+2)\eps/(d\sqrt{r\log(e+(\eps/\delta))})$ and the total number of zeroth-order gradient computations is $nT=n^2(r+2)\eps/(d\sqrt{r\log(e+(\eps/\delta))})$. The above selection ensures scale invariance.
\end{proof}

\section{Detailed Proof and Analysis of DPZero (Algorithm \ref{algo:d-free})}
\label{app:d-free}
    
\paragraph{Privacy guarantee.}
Since $u_t$ is independent of the dataset $S$, the privacy guarantees directly follow from Lemma \ref{lm:composition} and post-processing \citep{dwork2014algorithmic} noticing that the sensitivity is $2C/n$. We want to emphasis that the randomness of $u_t$ is never used for the privacy guarantee, and the analysis holds for any $u_t$ as long as it is independent of the dataset.

\paragraph{Utility guarantee.}
We then focus on the utility guarantee on $\bE\norm{\nabla F_S(x_\tau)}^2$. Since $f(x;\xi)$ is $\ell$-smooth for every $\xi$ by Assumption \ref{asp:rank}, we have that
\begin{equation}
    \begin{split}
        \frac{\abs{f(x_t+\lambda u_t;\xi_i) - f(x_t-\lambda u_t;\xi_i)}}{2\lambda}
        & \leq
        \abs{u_t^\top \nabla f(x_t;\xi_i)} +
        \frac{\abs{f(x_t+\lambda u_t;\xi_i) - f(x_t;\xi_i) - \lambda u_t^\top\nabla f(x_t;\xi_i)}}{2\lambda} \\
        & \quad +
        \frac{\abs{f(x_t-\lambda u_t;\xi_i) - f(x_t;\xi_i) + \lambda u_t^\top\nabla f(x_t;\xi_i)}}{2\lambda} \\
        & \leq
        \abs{u_t^\top \nabla f(x_t;\xi_i)} + \frac{\ell}{2}\lambda d.
    \end{split}
    \label{eq:d-free-C-upper-bound}
\end{equation}
Therefore, by $(ii)$ in Lemma \ref{lm:sphere} and Lipschitzness of $f(x;\xi)$, we have that
\begin{align*}
    \bP\left(\frac{\abs{f(x_t+\lambda u_t;\xi_i) - f(x_t-\lambda u_t;\xi_i)}}{2\lambda} \geq C_0 + \frac{\ell}{2}\lambda d\right)
    & \leq
    \bP(\abs{u_t^\top \nabla f(x_t;\xi_i)} \geq C_0) \\
    & \leq
    2\sqrt{2\pi}\exp\left(-\frac{C_0^2}{8\norm{\nabla f(x_t;\xi_i)}^2}\right) \\
    & \leq
    2\sqrt{2\pi}\exp\left(-\frac{C_0^2}{8L^2}\right).
\end{align*}
We define $Q_{t,i}$ to be the event that the clipping does not happen at iteration $t$ for sample $\xi_i$ and $\bar Q_{t,i}$ to be the event that the clipping does happen. The above equation implies that if the clipping threshold $C\geq C_0+\ell\lambda d/2$, then we have that $\bP(\bar Q_{t,i}) \leq 2\sqrt{2\pi}\exp(-C_0^2/(8L^2))$.
Let $Q_t$ denote the event that the clipping does not happen at iteration $t$ for every sample $1\leq i\leq n$, and let $\bar Q_t$ be the event that there exist some $i$ such that the clipping does happen at iteration $t$.
We also denote $Q$ as the event that the clipping does not happen for every iteration $t=0,1,\cdots,T-1$ and every sample $1\leq i\leq n$ and $\bar Q$ as the event that there exist some $t$ and $i$ such that the clipping does happen. By the union bound, we have that
\begin{align*}
    \bP(\bar Q)
    & =
    \bP\left(\bigcup_{t=0}^{T-1}\bigcup_{i=1}^{n}\bar Q_{t,i}\right) \\
    & \leq
    2\sqrt{2\pi}\, nT\exp\left(-\frac{C_0^2}{8L^2}\right).
\end{align*}
To simplify the notation, we let
\begin{align*}
    G_\lambda(x_t)
    & =
    \frac{1}{n}\sum_{i=1}^n \frac{f(x_t+\lambda u_t;\xi_i) - f(x_t-\lambda u_t;\xi_i)}{2\lambda}u_t \\
    & =
    \frac{F_S(x_t + \lambda u_t) - F_S(x_t - \lambda u_t)}{2\lambda}u_t,
\end{align*}
and its per-sample clipped version as
\begin{equation*}
    \hat G_\lambda(x_t) = \frac{1}{n}\sum_{i=1}^n \clip_C\left(\frac{f(x_t+\lambda u_t;\xi_i) - f(x_t-\lambda u_t;\xi_i)}{2\lambda}\right)u_t.
\end{equation*}
Algorithm \ref{algo:d-free} becomes $x_{t+1} = x_t - \alpha(\hat G_\lambda(x_t) + z_tu_t)$ under the above notation. By Taylor's theorem with remainder, for some $\theta\in(0,1)$, we have that
\begin{align*}
    F_S(x_{t+1})
    & =
    F_S(x_t) + \nabla F_S(x_t)^\top (x_{t+1} - x_t) + \frac{1}{2} (x_{t+1} - x_t)^\top \nabla^2 F_S(x_t+\theta(x_{t+1} - x_t)) (x_{t+1} - x_t) \\
    & \leq
    F_S(x_t) - \alpha\nabla F_S(x_t)^\top \left(\hat G_\lambda(x_t) + z_tu_t\right) + \frac{\alpha^2}{2} \hat G_\lambda(x_t)^\top H \hat G_\lambda(x_t) + \frac{\alpha^2}{2} z_t^2 u_t^\top H u_t \\
    & \qquad +
    \frac{\alpha^2}{2}z_t \left(\hat G_\lambda(x_t)^\top H u_t + u_t^\top H \hat G_\lambda(x_t)\right).
\end{align*}
Here in the inequality, we use Assumption \ref{asp:rank} such that $\nabla^2 F_S(x)\preceq H$ for any $x\in\bR^d$.
The event $Q_t$ depends on the randomness in $u_{<(t+1)}:=\{u_0, u_1, \cdots, u_t\}$ and $z_{<t}:=\{z_0, z_1, \cdots, z_{t-1}\}$. Note that the scalar noise $z_t$ sampled from $\cN(0, \sigma^2)$ is independent of $u_{<(t+1)}$, $z_{<t}$, $x_t$, and the dataset $S$. Conditioned on the event $Q_t$ and taking expectation with respect to $z_{<(t+1)}$ and $u_{<(t+1)}$, we have that
\begin{equation}
    \begin{split}
        \bE_{z_{<(t+1)}, u_{<(t+1)}}[F_S(x_{t+1}) | Q_t]
        & \leq
        \bE_{z_{<t}, u_{<t}} [F_S(x_t) | Q_t] -
        \alpha\,\bE_{z_{<t}, u_{<(t+1)}}\left[\nabla F_S(x_t)^\top \hat G_\lambda(x_t) \middle| Q_t \right] \\
        & \quad +
        \frac{\alpha^2}{2}\,\bE_{z_{<t}, u_{<(t+1)}}\left[ \hat G_\lambda(x_t)^\top H \hat G_\lambda(x_t) \middle| Q_t \right] +
        \frac{\alpha^2\sigma^2}{2}\,\bE_{z_{<t}, u_{<(t+1)}}\left[u_t^\top H u_t \middle| Q_t \right].
    \end{split}
    \label{eq:main-smooth}
\end{equation}
Let $\bE_t:=\bE_{z_{<t}, u_{<(t+1)}}$ for simplicity. Given the condition that $Q_t$ happens, we know that $\hat G_\lambda(x_t)=G_\lambda(x_t)$ and
\begin{equation*}
    \bE_t\left[\hat G_\lambda(x_t)^\top H \hat G_\lambda(x_t) \,\middle|\, Q_t\right] =
    \bE_t\left[
        \left(\frac{F_S(x_t+\lambda u_t) - F_S(x_t-\lambda u_t)}{2\lambda}\right)^2 u_t^\top H u_t
        \,\middle|\, Q_t
    \right].
\end{equation*}
Since $H\succeq0$, we have that $u_t^\top H u_t\geq 0$. By the law of total probability, we obtain
\begin{equation}
    \begin{split}
        & \phantom{=\;\;}
        \bE_t\left[
            \left(\frac{F_S(x_t+\lambda u_t) - F_S(x_t-\lambda u_t)}{2\lambda}\right)^2 u_t^\top H u_t
        \right] \\
        & =
        \bE_t\left[
            \left(\frac{F_S(x_t+\lambda u_t) - F_S(x_t-\lambda u_t)}{2\lambda}\right)^2 u_t^\top H u_t
            \,\middle|\, Q_t
        \right] \bP(Q_t) \\
        & \quad +
        \bE_t\left[
            \left(\frac{F_S(x_t+\lambda u_t) - F_S(x_t-\lambda u_t)}{2\lambda}\right)^2 u_t^\top H u_t
            \,\middle|\, \bar Q_t
        \right] \bP(\bar Q_t) \\
        & \geq
        \bE_t\left[
            \left(\frac{F_S(x_t+\lambda u_t) - F_S(x_t-\lambda u_t)}{2\lambda}\right)^2 u_t^\top H u_t
            \,\middle|\, Q_t
        \right] \bP(Q_t).
    \end{split}
    \label{eq:total-prob}
\end{equation}
Assumption \ref{asp:rank} implies $F_S(x)$ is also $\ell$-smooth. Similarly to the proof of Theorem \ref{thm:d-dependent-rank}, by \eqref{eq:upper-bound-finite-diff} and the fact that $u_t^\top Hu_t \geq 0$, applying $(iv)$ in Lemma \ref{lm:sphere} and Assumption \ref{asp:rank}, we can then obtain that
\begin{align}
    \bE_t\left[\hat G_\lambda(x_t)^\top H \hat G_\lambda(x_t) \,\middle|\, Q_t\right]
    & \leq
    \frac{\bE_t\left[
        (F_S(x_t+\lambda u_t) - F_S(x_t-\lambda u_t))^2 u_t^\top H u_t
    \right]}{4\lambda^2\cdot\bP(Q_t)} \nonumber \\
    & \leq
    \frac{\bE_t\left[
        2\left(u_t^\top\nabla F_S(x_t)\right)^2 u_t^\top H u_t
    \right]}{\bP(Q_t)} +
    \frac{\ell^2\lambda^2d^2}{2\,\bP(Q_t)}\bE_t\left[u_t^\top H u_t\right] \nonumber \\
    & =
    \frac{2d\,\bE_{z_{<t}, u_{<t}}\left[2\nabla F_S(x_t)^\top H\nabla F_S(x_t) + \norm{\nabla F_S(x_t)}^2 \tr(H)\right]}{(d+2)\,\bP(Q_t)} +
    \frac{\ell^2\lambda^2d^2\tr(H)}{2\,\bP(Q_t)} \nonumber \\
    & \leq
    \frac{2\ell(r+2)}{\bP(Q_t)}\bE_{z_{<t}, u_{<t}}\norm{\nabla F_S(x_t)}^2 + \frac{\ell^3\lambda^2d^2r}{2\,\bP(Q_t)}.
    \label{eq:grad-quad}
\end{align}
The same as \eqref{eq:total-prob}, we can also get that
\begin{equation}
    \begin{split}
        \bE_t\left[u_t^\top Hu_t \,\middle|\, Q_t \right]
        & \leq
        \frac{\bE_t\left[u_t^\top Hu_t \right]}{\bP(Q_t)} \\
        & \leq
        \frac{r\ell}{\bP(Q_t)}.
    \end{split}
    \label{eq:noise-quad}
\end{equation}
For the inner-product term, we have that
\begin{equation*}
    \bE_t\left[\nabla F_S(x_t)^\top \hat G_\lambda(x_t) \,\middle|\, Q_t \right] =
    \bE_t\left[\nabla F_S(x_t)^\top G_\lambda(x_t) \,\middle|\, Q_t \right].
\end{equation*}
By the law of total probability, since $u_t$ is independent of $x_t$, we know that
\begin{align*}
    \bE_t\left[\nabla F_S(x_t)^\top G_\lambda(x_t) \,\middle|\, Q_t \right] \bP(Q_t) + \bE_t\left[\nabla F_S(x_t)^\top G_\lambda(x_t) \,\middle|\, \bar Q_t \right] \bP(\bar Q_t)
    & =
    \bE_t\left[\nabla F_S(x_t)^\top G_\lambda(x_t) \right] \\
    & =
    \bE_{z_{<t}, u_{<t}}\left[\nabla F_S(x_t)^\top \nabla F_\lambda(x_t) \right],
\end{align*}
where we use Lemma \ref{lm:zero-grad} for $F_\lambda(x)=\bE_v[F_S(x+\lambda v)]$ with $v$ uniformly sampled from $\sqrt{d}\,\bB^d$. Rearranging terms, we thus obtain that
\begin{align*}
    \bE_t\left[\nabla F_S(x_t)^\top G_\lambda(x_t) \,\middle|\, Q_t \right]
    & =
    \frac{\bE_{z_{<t}, u_{<t}}\left[\nabla F_S(x_t)^\top \nabla F_\lambda(x_t) \right]}{\bP(Q_t)} - \frac{\bE_t\left[\nabla F_S(x_t)^\top G_\lambda(x_t) \,\middle|\, \bar Q_t \right] \bP(\bar Q_t)}{\bP(Q_t)} \\
    & =
    \frac{\bE_{z_{<t}, u_{<t}}\norm{\nabla F_S(x_t)}^2}{2\,\bP(Q_t)} + \frac{\bE_{z_{<t}, u_{<t}}\norm{\nabla F_\lambda(x_t)}^2}{2\,\bP(Q_t)} - \frac{\bE_{z_{<t}, u_{<t}}\norm{\nabla F_S(x_t) - \nabla F_\lambda(x_t)}^2}{2\,\bP(Q_t)} \\
    & \quad -
    \frac{\bE_t\left[\nabla F_S(x_t)^\top G_\lambda(x_t) \,\middle|\, \bar Q_t \right] \bP(\bar Q_t)}{\bP(Q_t)} \\
    & \geq
    \frac{\bE_{z_{<t}, u_{<t}}\norm{\nabla F_S(x_t)}^2}{2\,\bP(Q_t)} - \frac{\ell^2\lambda^2d^3}{8\,\bP(Q_t)} - \frac{\bE_t\left[\nabla F_S(x_t)^\top G_\lambda(x_t) \,\middle|\, \bar Q_t \right] \bP(\bar Q_t)}{\bP(Q_t)},
\end{align*}
where we apply $(ii)$ in Lemma \ref{lm:zero-grad}. Assumption \ref{asp:rank} implies that $F_S(x)$ is also Lipschitz, and thus
\begin{align*}
    \nabla F_S(x_t)^\top G_\lambda(x_t)
    & \leq
    \norm{\nabla F_S(x_t)}\norm{G_\lambda(x_t)} \\
    & \leq
    L^2\norm{u_t}^2 \\
    & =
    L^2d.
\end{align*}
As a result, we obtain that
\begin{equation}
    \bE_t\left[\nabla F_S(x_t)^\top \hat G_\lambda(x_t) \,\middle|\, Q_t \right] \geq \frac{\bE_{z_{<t}, u_{<t}}\norm{\nabla F_S(x_t)}^2}{2\,\bP(Q_t)} - \frac{\ell^2\lambda^2d^3}{8\,\bP(Q_t)} - \frac{L^2d\;\bP(\bar Q_t)}{\bP(Q_t)}.
    \label{eq:inner-product}
\end{equation}
Plugging \eqref{eq:inner-product}, \eqref{eq:grad-quad} and \eqref{eq:noise-quad} back into \eqref{eq:main-smooth}, we obtain that
\begin{align}
    \bE_{z_{<(t+1)}, u_{<(t+1)}}[F_S(x_{t+1})|Q_t]
    & \leq
    \bE_{z_{<t}, u_{<t}}[F_S(x_t)|Q_t] - \frac{\alpha}{2}(1 - 2(r+2)\ell\alpha)\frac{\bE_{z_{<t}, u_{<t}}\norm{\nabla F_S(x_t)}^2}{\bP(Q_t)} + \frac{\ell\alpha^2\, r\sigma^2}{2\,\bP(Q_t)} \nonumber \\ 
    & \quad +
    \frac{\ell^2\alpha(d+2\ell\alpha r)\lambda^2d^2}{8\,\bP(Q_t)} + \frac{\alpha L^2d\;\bP(\bar Q_t)}{\bP(Q_t)}.
    \label{eq:onestep3}
\end{align}
Choosing $\alpha=1/(4\ell(r+2))$ such that $1-2(r+2)\ell\alpha=1/2$ and $2\ell\alpha r<1\leq d$, we have that
\begin{align*}
    \bE_{z_{<t}, u_{<t}}\norm{\nabla F_S(x_t)}^2
    & \leq
    \frac{4\,\bE_{z_{<(t+1)}, u_{<(t+1)}}[F_S(x_t) - F_S(x_{t+1}) | Q_t]\bP(Q_t)}{\alpha} + 2\ell\alpha\, r\sigma^2 + \ell^2d^3\lambda^2 + 4L^2d\;\bP(\bar Q_t) \\
    & \leq
    \frac{4\,\bE_{z_{<(t+1)}, u_{<(t+1)}}[F_S(x_t) - F_S(x_{t+1}) | Q_t]\bP(Q_t)}{\alpha} + 2\ell\alpha\, r\sigma^2 + \ell^2d^3\lambda^2 + 4L^2d\;\bP(\bar Q).
\end{align*}
Recall $Q_t$ is the event that clipping does not happen at iteration $t$ and $Q$ is the event that clipping does not happen for every iteration. By the law of total probability and the assumption that $\abs{F_S(x_t)}\leq B$ for every $t$, we have that
\begin{align*}
    \bE_{z_{<(t+1)}, u_{<(t+1)}}[F_S(x_t) - F_S(x_{t+1}) | Q_t]\bP(Q_t)
    & =
    \bE_{z_{<T}, u_{<T}}[F_S(x_t) - F_S(x_{t+1}) | Q_t]\bP(Q_t) \\
    & =
    \bE_{z_{<T}, u_{<T}}\Big[F_S(x_t) - F_S(x_{t+1}) \Big| Q_t \cap Q\Big]\bP(Q_t \cap Q) \\
    & \quad +
    \bE_{z_{<T}, u_{<T}}\Big[F_S(x_t) - F_S(x_{t+1}) \Big| Q_t \cap \bar Q\Big]\bP(Q_t \cap \bar Q) \\
    & \leq
    \bE_{z_{<T}, u_{<T}}[F_S(x_t) - F_S(x_{t+1}) | Q]\bP(Q) + 2B \; \bP(\bar Q).
\end{align*}
As a result, we have that
\begin{equation}
    \bE_{z_{<t}, u_{<t}}\norm{\nabla F_S(x_t)}^2 \leq
    \frac{4\,\bE_{z_{<T}, u_{<T}}[F_S(x_t) - F_S(x_{t+1}) | Q]\bP(Q)}{\alpha} + 2\ell\alpha\, r\sigma^2 + \ell^2d^3\lambda^2 + \left(4L^2d + \frac{8B}{\alpha}\right)\;\bP(\bar Q).
    \label{eq:onestep3b}
\end{equation}
Taking expectation with respect to all randomness, i.e., $\bE=\bE_{z_{<T}, u_{<T}}$, summing up from $t=0$ to $T-1$, and dividing both sides by $T$, we have that
\begin{align*}
    \bE\norm{\nabla F_S(x_\tau)}^2
    & =
    \frac{1}{T}\sum_{t=0}^{T-1} \bE_{z_{<t}, u_{<t}}\norm{\nabla F_S(x_t)}^2 \\
    & \leq
    \frac{4\,\bE[F_S(x_0) - F_S(x_T) | Q] \bP(Q)}{\alpha T} + \frac{64\ell C^2\, \alpha T\, r\log(e+ (\eps/\delta))}{n^2\eps^2} + \ell^2d^3\lambda^2 \\
    & \quad +
    8\sqrt{2\pi} \,n\,T (L^2 d + 8 \ell B(r+2))\exp\left(-\frac{C_0^2}{8L^2}\right) \\
    & \leq
    \Big(64\ell[F_S(x_0) - F_S^*]+4C^2\Big)\frac{\sqrt{r\log(e+(\eps/\delta))}}{n\eps} + \ell^2d^3\lambda^2 \\
    & \quad +
    \frac{2\sqrt{2\pi}\, n^2\eps(r+2)(L^2 d + 8\ell B(r+2))}{\sqrt{r\log(e+(\eps/\delta))}}\exp\left(-\frac{C_0^2}{8L^2}\right),
\end{align*}
with the choice of parameters to be
\begin{equation*}
    \alpha T=\frac{n\eps}{16\ell\sqrt{r\log(e+(\eps/\delta))}}, \quad \alpha =\frac{1}{4\ell(r+2)}, \quad
    T=\frac{n(r+2)\eps}{4\sqrt{r\log(e+(\eps/\delta))}}.
\end{equation*}
When selecting $\lambda \leq 2(\sqrt{2}-1)C_0/(\ell d)$, we can set $C=\sqrt{2}C_0$ such that $C\geq C_0+\ell\lambda d/2$ is satisfied. If $C_0$ and $\lambda$ further satisfy the conditions that
\begin{equation*}
    C_0^2=8L^2\,\log\left(\frac{2\sqrt{2\pi}\, n^3\eps^2 (r+2)(d + 8\ell B(r+2)/L^2)}{r\log(e+(\eps/\delta))}\right), \quad
    \lambda \leq \frac{L}{\ell d^{3/2}}\left(\frac{\sqrt{r\log(e+(\eps/\delta))}}{n\eps}\right)^{1/2},
\end{equation*}
we can then obtain that
\begin{align*}
    &
    \bE\norm{\nabla F_S(x_\tau)}^2 \\
    & \quad \leq
    \Big(64\,\ell[F_S(x_0) - F_S^*]+4C^2+2L^2\Big)
    \frac{\sqrt{r\log(e+(\eps/\delta))}}{n\eps} \\
    & \quad = \left(
        64\,\ell[F_S(x_0) - F_S^*] + 64\,L^2\,\log\left(\frac{2\sqrt{2\pi}\, n^3\eps^2(r+2)(d+8\ell B(r+2)/L^2)}{r\log(e+(\eps/\delta))}\right) + 2L^2
    \right)
    \frac{\sqrt{r\log(e+(\eps/\delta))}}{n\eps}.
\end{align*}
We conclude that the clipping threshold $C$ and smoothing parameter $\lambda$ should satisfy that
\begin{align*}
    & C=4L\sqrt{\log\left(\frac{2\sqrt{2\pi}\, n^3\eps^2 (r+2)(d + 8\ell B(r+2)/L^2)}{r\log(e+(\eps/\delta))}\right)}, \\
    & \lambda \leq \frac{L}{\ell d}\min\left\{
        4(2-\sqrt{2})\sqrt{\log\left(\frac{2\sqrt{2\pi}\, n^3\eps^2 (r+2)(d + 8\ell B(r+2)/L^2)}{r\log(e+(\eps/\delta))}\right)},
        \frac{1}{\sqrt{d}}\left(\frac{\sqrt{r\log(e+(\eps/\delta))}}{n\eps}\right)^{1/2}
    \right\}.
\end{align*}
The total number of zeroth-order gradient computations is $nT=n^2(r+2)\eps/(4\sqrt{r\log(e+(\eps/\delta))})$.

\section{Extension to the PL Setting}
\label{app:pl}

\begin{assumption}
    The average loss $F_S(x)$ satisfies the PL inequality with parameter $\mu>0$. That is, it holds that $\forall x\in\bR^d$,
    \begin{equation*}
        \norm{\nabla F_S(x)}^2 \geq 2\mu (F_S(x) - F_S^*).
    \end{equation*}
    \label{asp:pl}
\end{assumption}

\begin{corollary}
    Under the same setting of Theorem \ref{thm:d-dependent}, when Assumption \ref{asp:pl} is also met, let $\kappa=\ell/\mu$ be the condition number, the last iterate of Algorithm \ref{algo:d-dependent} satisfies that
    \begin{equation*}
        \bE[F_S(x_T) - F_S^*] \leq \left(\ell(F_S(x_0) - F_S^*) + 64L^2\kappa\, \log\left(\frac{n^2\eps^2}{\kappa\, d^3\log(e+(\eps/\delta))}\right) + 2L^2\right) \frac{d^3\log(e+(\eps/\delta))}{\mu n^2\eps^2},
    \end{equation*}
    with the choice of parameters
    \begin{equation*}
        \alpha=\frac{1}{4\ell d}, \quad
        T=8\,\kappa\, d\,\log\left(\frac{n^2\eps^2}{\kappa\,d^3\log(e+(\eps/\delta))}\right), \quad
        \lambda\leq\frac{2L}{\ell}\,\frac{\sqrt{\log(e+(\eps/\delta))}}{n\eps}, \quad
        C=Ld.
    \end{equation*}
    The total number of zeroth-order gradient computations is $nT=\tilde\cO(nd\kappa)$.
    \label{cor:pl-d}
\end{corollary}

\begin{proof}
    Starting from \eqref{eq:onestep1} in the proof of Theorem \ref{thm:d-dependent}, with the choice that $\alpha=1/(4\ell d)$, we have that
    \begin{align*}
        \bE_{z_t, u_t}[F_S(x_{t+1})]
        & \leq
        F_S(x_t) - \frac{\alpha}{4}\norm{F_S(x_t)}^2 + \frac{\alpha}{4}\, \ell^2\lambda^2 d^3 + \frac{\ell}{2}\alpha^2 \, d\sigma^2 \\
        & \leq
        F_S(x_t) - \frac{\mu\alpha}{2}(F_S(x_t) - F_S^*) + \frac{\alpha}{4}\, \ell^2\lambda^2 d^3 + \frac{\ell}{2}\alpha^2 \, d\sigma^2.
    \end{align*}
    This gives the recursion that
    \begin{equation*}
        \bE[F_S(x_{t+1}) - F_S^*] \leq \left(1 - \frac{\mu\alpha}{2}\right) \bE[F_S(x_t) - F_S^*] + \frac{\alpha}{4}\, \ell^2\lambda^2 d^3 + \frac{\ell}{2}\alpha^2 \, d\sigma^2.
    \end{equation*}
    Resolving the recursion, we obtain that
    \begin{align*}
        \bE[F_S(x_T) - F_S^*]
        & \leq
        \left(1 - \frac{\mu\alpha}{2}\right)^T (F_S(x_0) - F_S^*) + \left(\frac{\alpha}{4}\, \ell^2\lambda^2 d^3 + \frac{\ell}{2}\alpha^2 \, d\sigma^2\right)\left(\left(1 - \frac{\mu\alpha}{2}\right)^{T-1} + \cdots + \left(1 - \frac{\mu\alpha}{2}\right) + 1\right) \\
        & \leq
        \exp\left(-\frac{\mu\alpha T}{2}\right)(F_S(x_0) - F_S^*) + \frac{\ell^2\lambda^2 d^3}{2\mu} + \frac{\ell\alpha \, d\sigma^2}{\mu} \\
        & =
        \exp\left(-\frac{\mu\alpha T}{2}\right)(F_S(x_0) - F_S^*) + \frac{32\ell L^2\, \alpha T\, d^3\log(e+ (\eps/\delta))}{\mu n^2\eps^2} + \frac{\ell^2\lambda^2 d^3}{2\mu} \\
        & =
        \left(\ell(F_S(x_0) - F_S^*) + 64L^2\kappa\,\log\left(\frac{n^2\eps^2}{\kappa\, d^3\log(e+(\eps/\delta))}\right) + 2L^2\right) \frac{d^3\log(e+(\eps/\delta))}{\mu n^2\eps^2},
    \end{align*}
    with the choice of parameters
    \begin{equation*}
        \alpha T = \frac{2}{\mu} \log\left(\frac{n^2\eps^2}{\kappa\, d^3\log(e+(\eps/\delta))}\right), \quad
        \lambda \leq \frac{2L}{\ell}\,\frac{\sqrt{\log(e+(\eps/\delta))}}{n\eps}.
    \end{equation*}
    The total number of iteration is $T=\tilde\cO(\kappa d)$.
\end{proof}

\begin{corollary}
    Under the same setting of Theorem \ref{thm:d-dependent-rank}, when Assumption \ref{asp:pl} is also met, let $\kappa=\ell/\mu$ be the condition number, the last iterate of Algorithm \ref{algo:d-dependent} satisfies that
    \begin{equation*}
        \bE[F_S(x_T) - F_S^*] \leq \left(\ell(F_S(x_0) - F_S^*) + 64L^2\kappa\, \log\left(\frac{n^2\eps^2}{\kappa\, rd^2\log(e+(\eps/\delta))}\right) + 2L^2\right) \frac{rd^2\log(e+(\eps/\delta))}{\mu n^2\eps^2},
    \end{equation*}
    with the choice of parameters
    \begin{equation*}
        \alpha=\frac{1}{4\ell (r+2)}, \quad
        T=8\,\kappa\, (r+2)\,\log\left(\frac{n^2\eps^2}{\kappa\,rd^2\log(e+(\eps/\delta))}\right), \quad
        \lambda\leq\frac{2L}{\ell\sqrt{d}}\,\frac{\sqrt{r\log(e+(\eps/\delta))}}{n\eps}, \quad
        C=Ld.
    \end{equation*}
    The total number of zeroth-order gradient computations is $nT=\tilde\cO(nr\kappa)$.
    \label{cor:pl-d-rank}
\end{corollary}

\begin{proof}
    Starting from \eqref{eq:onestep2} in the proof of Theorem \ref{thm:d-dependent-rank}, with the choice that $\alpha=1/(4\ell (r+2))$, we have that
    \begin{align*}
        \bE_{z_t, u_t}[F_S(x_{t+1})]
        & \leq
        F_S(x_t) - \frac{\alpha}{4}\norm{F_S(x_t)}^2 + \frac{\alpha}{4}\, \ell^2\lambda^2 d^3 + \frac{\ell}{2}\alpha^2 \, r\sigma^2 \\
        & \leq
        F_S(x_t) - \frac{\mu\alpha}{2}(F_S(x_t) - F_S^*) + \frac{\alpha}{4}\, \ell^2\lambda^2 d^3 + \frac{\ell}{2}\alpha^2 \, r\sigma^2.
    \end{align*}
    This gives the recursion that
    \begin{equation*}
        \bE[F_S(x_{t+1}) - F_S^*] \leq \left(1 - \frac{\mu\alpha}{2}\right) \bE[F_S(x_t) - F_S^*] + \frac{\alpha}{4}\, \ell^2\lambda^2 d^3 + \frac{\ell}{2}\alpha^2 \, r\sigma^2.
    \end{equation*}
    Resolving the recursion, we obtain that
    \begin{align*}
        \bE[F_S(x_T) - F_S^*]
        & \leq
        \left(1 - \frac{\mu\alpha}{2}\right)^T (F_S(x_0) - F_S^*) + \left(\frac{\alpha}{4}\, \ell^2\lambda^2 d^3 + \frac{\ell}{2}\alpha^2 \, r\sigma^2\right)\left(\left(1 - \frac{\mu\alpha}{2}\right)^{T-1} + \cdots + \left(1 - \frac{\mu\alpha}{2}\right) + 1\right) \\
        & \leq
        \exp\left(-\frac{\mu\alpha T}{2}\right)(F_S(x_0) - F_S^*) + \frac{\ell^2\lambda^2 d^3}{2\mu} + \frac{\ell\alpha \, r\sigma^2}{\mu} \\
        & =
        \exp\left(-\frac{\mu\alpha T}{2}\right)(F_S(x_0) - F_S^*) + \frac{32\ell L^2\, \alpha T\, rd^2\log(e+ (\eps/\delta))}{\mu n^2\eps^2} + \frac{\ell^2\lambda^2 d^3}{2\mu} \\
        & =
        \left(\ell(F_S(x_0) - F_S^*) + 64L^2\kappa\,\log\left(\frac{n^2\eps^2}{\kappa\, rd^2\log(e+(\eps/\delta))}\right) + 2L^2\right) \frac{rd^2\log(e+(\eps/\delta))}{\mu n^2\eps^2},
    \end{align*}
    with the choice of parameters
    \begin{equation*}
        \alpha T = \frac{2}{\mu} \log\left(\frac{n^2\eps^2}{\kappa\, rd^2\log(e+(\eps/\delta))}\right), \quad
        \lambda \leq \frac{2L}{\ell\sqrt{d}}\,\frac{\sqrt{r\log(e+(\eps/\delta))}}{n\eps}.
    \end{equation*}
    The total number of iteration is $T=\tilde\cO(\kappa r)$.
\end{proof}

\begin{corollary}
    Under the same setting of Theorem \ref{thm:d-free}, when Assumption \ref{asp:pl} is also met, let $\kappa=\ell/\mu$ be the condition number, suppose $\max_{0\leq t\leq T} \abs{F_S(x_t)}\leq B$ and $\abs{F_S^*}\leq B$, the last iterate of Algorithm \ref{algo:d-free} satisfies that
    \begin{equation*}
        \bE[F_S(x_T) - F_S^*] \leq
        \left(\ell(F_S(x_0) - F_S^*) + \log\left(\frac{n^2\eps^2}{\kappa\,r\log(e+(\eps/\delta))}\right)\left(L^2 + 16\tilde L^2\kappa\right) + 2L^2\right) \frac{r\log(e+(\eps/\delta))}{\mu n^2\eps^2},
    \end{equation*}
    where we define
    \begin{equation*}
        \tilde L^2 = 64L^2\log\left(\frac{32\sqrt{2\pi}\,\kappa\,n^3\eps^2(r+2)(d+(8\ell (r+2)+\mu)B/L^2)}{r\log(e+(\eps/\delta))}\right),
    \end{equation*}
    and choose the parameters to be
    \begin{align*}
        & \alpha=\frac{1}{4\ell (r+2)}, \quad
        T=8\,\kappa\, (r+2)\,\log\left(\frac{n^2\eps^2}{\kappa\,r\log(e+(\eps/\delta))}\right), \quad C = \frac{\tilde L}{2}, \\
        & \lambda\leq\frac{1}{2\ell d}\,\min\left\{(2-\sqrt{2})\tilde L,\; \frac{4L}{\sqrt{d}}\frac{\sqrt{r\log(e+(\eps/\delta))}}{n\eps}\right\}.
    \end{align*}
    The total number of zeroth-order gradient computations is $nT=\tilde\cO(nr\kappa)$.
    \label{cor:pl-free}
\end{corollary}

\begin{remark}
    A more precise expression of our theoretical results, including Theorems \ref{thm:d-dependent}, \ref{thm:d-dependent-rank}, and \ref{thm:d-free} and their corresponding Corollaries \ref{cor:pl-d}, \ref{cor:pl-d-rank}, and \ref{cor:pl-free}, is to cover cases where $T$ may be less than 1. Considering Theorem \ref{thm:d-free} as an example, a more accurate statement is
   \begin{equation*}
       T = \max\left\{\frac{n(r+2)\eps}{4\sqrt{r\log(e+(\eps/\delta))}}, 1\right\}, \quad
       \bE[\norm{\nabla F_S(x_\tau)}^2] \leq \min\left\{\tilde\cO\left(\frac{\sqrt{r\log(e+(\eps/\delta))}}{n\eps}\right), L^2\right\}.
   \end{equation*}
   For the sake of clarity and simplicity in presentation, this detail is omitted in the main results.
   \label{rmk:bound}
\end{remark}

\begin{proof}
    Starting from \eqref{eq:onestep3b} in the proof of Theorem \ref{thm:d-free} with the choice $\alpha=1/(4\ell (r+2))$ and using Assumption \ref{asp:pl} such that
    \begin{align*}
        \bE\norm{\nabla F_S(x_t)}^2
        & \geq
        2\mu\, \bE[F_S(x_t) - F_S^*] \\
        & =
        2\mu\, \bE[F_S(x_t) - F_S^* | Q] \, \bP(Q) + 2\mu \, \bE[F_S(x_t) - F_S^* | \bar Q] \,\bP(\bar Q) \\
        & \geq
        2\mu\, \bE[F_S(x_t) - F_S^* | Q] \, \bP(Q),
    \end{align*}
    we have the recursion that
    \begin{equation*}
        \bE[F_S(x_{t+1}) - F_S^* | Q] \bP(Q) \leq \left(1 - \frac{\mu\alpha}{2}\right) \bE[F_S(x_t) - F_S^* | Q] \bP(Q) + \frac{\alpha}{4}\, \ell^2\lambda^2 d^3 + \frac{\ell}{2}\alpha^2\, r\sigma^2 +  (L^2d\alpha + 2B)\bP(\bar Q).
    \end{equation*}
    Resolving the recursion, we obtain that
    \begin{align*}
        \bE[F_S(x_T) - F_S^* | Q] \bP(Q)
        & \leq
        \left(1 - \frac{\mu\alpha}{2}\right)^T (F_S(x_0) - F_S^*) + \frac{2}{\mu\alpha}\left(\frac{\alpha}{4}\, \ell^2\lambda^2 d^3 + \frac{\ell}{2}\alpha^2\, r\sigma^2 +  (L^2d\alpha + 2B)\bP(\bar Q)\right) \\
        & \leq
        \exp\left(-\frac{\mu\alpha T}{2}\right)(F_S(x_0) - F_S^*) + \frac{\ell^2\lambda^2 d^3}{2\mu} + \frac{\ell\alpha \, r\sigma^2}{\mu} + \frac{(2L^2d + 4B/\alpha)\bP(\bar Q)}{\mu}.
    \end{align*}
    Since the event $Q$ happens with high probability, the above results can be refined to
    \begin{align*}
        \bE[F_S(x_T) - F_S^*]
        & =
        \bE[F_S(x_T) - F_S^* | Q]\bP(Q) + \bE[F_S(x_T) - F_S^* | \bar Q]\bP(\bar Q) \\
        & \leq
        \bE[F_S(x_T) - F_S^* | Q]\bP(Q) + 2B \;\bP(\bar Q).
    \end{align*}
    Therefore, we can obtain that
    \begin{align*}
        \bE[F_S(x_T) - F_S^*]
        & \leq
        \exp\left(-\frac{\mu\alpha T}{2}\right)(F_S(x_0) - F_S^*) + \frac{32\ell C^2\, \alpha T\, r\log(e+ (\eps/\delta))}{\mu n^2\eps^2} \\
        & \quad +
        \frac{4\sqrt{2\pi} nT(L^2d + 2B/\alpha + B\mu)}{\mu}\exp\left(-\frac{C_0^2}{8L^2}\right) + \frac{\ell^2\lambda^2 d^3}{2\mu} \\
        & =
        \left(\ell(F_S(x_0) - F_S^*) + L^2\,\log\left(\frac{n^2\eps^2}{\kappa\,r\log(e+(\eps/\delta))}\right)\right) \frac{r\log(e+(\eps/\delta))}{\mu n^2\eps^2} \\
        & \quad +
        \frac{32\ell C^2\, \alpha T\, r\log(e+ (\eps/\delta))}{\mu n^2\eps^2} + \frac{\ell^2\lambda^2 d^3}{2\mu},
    \end{align*}
    with the choice of parameters
    \begin{equation*}
        \alpha T = \frac{2}{\mu} \log\left(\frac{n^2\eps^2}{\kappa\, r\log(e+(\eps/\delta))}\right), \quad
        C_0^2 = 8L^2\,\log\left(\frac{32\sqrt{2\pi}\,\kappa\,n^3\eps^2(r+2)(d + (8\ell (r+2) + \mu)B/L^2)}{r\log(e+(\eps/\delta))}\right).
    \end{equation*}
    When selecting $\lambda$ to be
    \begin{equation*}
        \lambda\leq\min\left\{\frac{2(\sqrt{2} - 1)C_0}{\ell d},\; \frac{2L}{\ell d^{3/2}}\frac{\sqrt{r\log(e+(\eps/\delta))}}{n\eps}\right\},
    \end{equation*}
    we can set $C=\sqrt{2}C_0$ such that $C\geq C_0 + \ell\lambda d/2$ is satisfied, and thus
    \begin{equation*}
        \bE[F_S(x_T) - F_S^*] \leq
        \left(\ell(F_S(x_0) - F_S^*) + \log\left(\frac{n^2\eps^2}{\kappa\,r\log(e+(\eps/\delta))}\right)\left(L^2 + 16\tilde L^2\kappa\right) + 2L^2\right) \frac{r\log(e+(\eps/\delta))}{\mu n^2\eps^2},
    \end{equation*}
    where we define
    \begin{equation*}
        \tilde L^2 = 64L^2\log\left(\frac{32\sqrt{2\pi}\,\kappa\,n^3\eps^2(r+2)(d+(8\ell (r+2)+\mu)B/L^2)}{r\log(e+(\eps/\delta))}\right).
    \end{equation*}
    The total number of iteration is $T=\tilde\cO(\kappa r)$.
\end{proof}

\end{document}